\documentclass[twoside,11pt]{article}

\usepackage{jmlr2e}
\usepackage{verbatim}
\usepackage{hyperref}
\usepackage[cal=boondox]{mathalfa}
\usepackage{amsmath,bm}
\usepackage{bbold}
\usepackage{lipsum}
\usepackage{tikz,pgfplots}
\usetikzlibrary{patterns}
\usepackage{pst-func}
\usepackage{multirow}
\usepackage{algorithm}
\usepackage{algorithmic}
\newcommand{\trans}{{\sf T}}
\newcommand{\tr}{{\rm tr}}

\DeclareMathOperator*{\argmax}{arg\,max}
\DeclareMathOperator*{\argmin}{arg\,min}
\newtheorem{assumption}{Assumption}
\DeclareMathOperator*{\mode}{mode}
\newcommand{\asto}{\overset{\rm a.s.}{\longrightarrow}}

\newcommand{\GREEN}{\color[rgb]{0,0.70,0}}

\newcommand{\xddots}{%
  \raise 4pt \hbox {.}
  \mkern 6mu
  \raise 1pt \hbox {.}
  \mkern 6mu
  \raise -2pt \hbox {.}
}

\DeclareMathSymbol{\Gamma}{\mathord}{operators}{"00}

\ShortHeadings{Large Dimensional Multi Task Learning}{}
\firstpageno{1}
\pgfplotsset{compat=1.14}
\begin{document}
\pgfmathdeclarefunction{gauss}{2}{%
  \pgfmathparse{1/(#2*sqrt(2*pi))*exp(-((x-#1)^2)/(2*#2^2))}%
}
\title{Large Dimensional Analysis and Improvement\\ of Multi Task Learning}
\author{\name Malik Tiomoko \email malik.tiomoko@u-psud.fr \\
       \addr Laboratoire des Signaux et Systèmes\\
       University of Paris Saclay\\
       Orsay, France\\
       \AND
       \name Romain Couillet \email romain.couillet@gipsa-lab.grenoble-inp.fr \\
       \addr Gipsa Lab\\
       Université de Grenoble-Alpes\\
       Saint Martin d'Hères, France\\
       \AND
       \name Hafiz Tiomoko Ali \email hafiz.tiomoko.ali@huawei.com \\
       \addr Huawei Technologies Research and Development (UK) Limited\\
       London, UK
       }

\editor{}

\maketitle

\begin{abstract}
Multi Task Learning (MTL) efficiently leverages useful information contained in multiple related tasks to help improve the generalization performance of all tasks. This article conducts a large dimensional analysis of a simple but, as we shall see, extremely powerful when carefully tuned, Least Square Support Vector Machine (LSSVM) version of MTL, in the regime where the dimension $p$ of the data and their number $n$ grow large at the same rate.

Under mild assumptions on the input data, the theoretical analysis of the MTL-LSSVM algorithm first reveals the ``sufficient statistics'' exploited by the algorithm and their interaction at work. These results demonstrate, as a striking consequence, that the standard approach to MTL-LSSVM is largely suboptimal, can lead to severe effects of \emph{negative transfer} but that these impairments are easily corrected. These corrections are turned into an improved MTL-LSSVM algorithm which can only benefit from additional data, and the theoretical performance of which is also analyzed.

As evidenced and theoretically sustained in numerous recent works, these large dimensional results are robust to broad ranges of data distributions, which our present experiments corroborate. Specifically, the article reports a systematically close behavior between theoretical and empirical performances on popular datasets, which is strongly suggestive of the applicability of the proposed carefully tuned MTL-LSSVM method to real data. This fine-tuning is fully based on the theoretical analysis and does not in particular require any cross validation procedure. Besides, the reported performances on real datasets almost systematically outperform much more elaborate and less intuitive state-of-the-art multi-task and transfer learning methods.
\end{abstract}


\section{Introduction}
The methodology for a long time considered in machine learning has consisted in tackling each given (classification, regression, estimation) problem, hereafter referred to as a \emph{task}, independently. This approach is in general counterproductive as it automatically discards a potentially rich source of data often available to perform more or less similar tasks. Multi Task Learning (MTL) precisely aims to handle this deficiency by connecting datasets and tasks so to improve the generalization performance of one or several specific target tasks. This framework has recently gained renewed interest \citep{yang2020transfer,caruana1997multitask,collobert2008unified}, given the availability of gigantic datasets (such as huge prelabelled image databases) and costly trained learning machines (such as deep neural nets), which must be useful to help solve learning tasks involving much fewer labelled data. Beyond this resurgence, numerous applications inherently benefit from a MTL approach, of which we may cite a few examples: prediction of student test results for a collection of schools \citep{aitkin1986statistical}, patient survival estimates in different clinics \citep{harutyunyan2017multitask,caruana1996using}, values of possibly related financial indicators \citep{allenby1998marketing}, preference modelling of many individuals in a marketing context \citep{greene2000econometric}, etc.

Carefully modelling the relatedness between tasks has long been claimed to be the most critical determinant of the MTL algorithm performance. Several such models have been considered in the literature: task relatedness can be modelled by assuming that the parameters relating the tasks lie on a low dimensional manifold \citep{argyriou2007multi,agarwal2010learning}; these relating parameters may alternatively be assumed to be close in norm \citep{evgeniou2004regularized,xu2013} or be distributed according to similar priors \citep{xue2007multi,yu2005learning}. However, for all these models, a failure in properly matching the task parameters is often likely to induce possibly severe cases of \emph{negative learning}, that is occurrences where additional tasks play \emph{against} rather than in favor of the target task objective. These cases of negative learning are difficult to anticipate as few theoretical works are amenable to prepare the experimenter to these scenarios. In the present work, we adopt a similar strategy as in \citep{evgeniou2004regularized}, but with a strong theoretical background which will automatically eliminate the risks of negative learning.

In detail, the article \citep{evgeniou2004regularized}, the spirit of which is followed here, is inspired by the natural extension of support vector machines (SVMs) \citep{vapnik2005universal} to a multiple, say $k$, task setting, by paralleling $k$ SVMs but constraining their parameters (specifically, the $k$ separating hyperplane normal vectors $\omega_1,\ldots,\omega_k$) to be ``close'' to each other. This is enforced by simply imposing that $\omega_i=\omega_0+v_i$ for some common hyperplane normal vector $\omega_0$ and dedicated hyperplane normal vectors $v_i$. The norm of the vectors $v_i$ is controlled through an additional hyperparameter $\lambda$ to strengthen or relax task relatedness. 
This is the approach followed in the present article, to the noticeable exception that the fully explicit least-square SVM (LSSVM) \citep{xu2013} rather than a margin-based SVM is considered. 
In addition to only marginally altering the overall behavior of the MTL algorithm of \citep{evgeniou2004regularized}, the LSSVM approach entails more explicit, more tractable, as well as more insightful results, let alone numerically cheaper implementations. As a matter of fact, by a now well-established universality argument of large dimensional statistics, it has been shown in closely related works \citep{mai2019high} that quadratic (least-square) cost functions are asymptotically optimal (as the data dimension and number increase) and uniformly outperform alternative costs (such as margin-based methods or logistic approaches), even in a classification setting; this argument further motivates to consider first and foremost the least square version of MTL-SVM.


\medskip

The intricate nature of the MTL framework, even in its simplest MTL-SVM version \citep{evgeniou2004regularized}, has so far left little room to sound and practical useful theoretical analysis -- which we believe to have been a main reason for its decayed importance before the resurgence of the powerful deep learning tools, in capacity to tip the performance-complexity tradeoff. 
Among existing theoretical analyses of MTL, an ``extended VC dimension'' approach to retrieve bounds on the generalization performance is proposed in \citep{baxter2000model,ben2003exploiting}. Using Bayesian and information theoretic arguments, \citep{baxter1997bayesian} answers the question of the minimal information and number of samples per task required to learn $k$ parallel tasks. However, these works only provide loose bounds and orders of magnitude which, if convenient to decide on the impossibility to reach a target objective, do not provide any satisfying accurate performance evaluations, nor do they allow for an optimal hyperparametrization of the MTL framework which, as we shall see, is of dramatic importance.

Following on a recent line of breakthroughs in applied random matrix theory, and specifically walking in the steps of \citep{liao2019large,mai2019large} which study a single-task LSSVM adapted to supervised \citep{liao2019large} and semi-supervised \citep{mai2019large} learning, the article develops a theoretical framework to exhaustively study the behavior and maximize the performance of a $k$-task $m$-class MTL-LSSVM framework, under the regime of numerous ($n$) and large ($p$) data, i.e., $n,p\to\infty$ with $n/p\to c_0\in(0,\infty)$. The data are here modelled as a mixture of $km$ \emph{concentrated random vectors}, i.e., for ${\bf x}$ a data of class $j$ ($j\in\{1,\ldots,m\}$) for Task~$i$ ($i\in\{1,\ldots,k\}$), ${\bf x}\sim\mathcal L_{ij}(\mu_{ij},\Sigma_{ij})$, where $\mathcal L_{ij}(\mu,\Sigma)$ is the law of a Lipschitz-concentrated random vector \citep{ledoux2001concentration} with statistical mean $\mu\in\mathbb R^p$ and covariance $\Sigma\in\mathbb R^{p\times p}$. For instance, ${\bf x}=\varphi_{ij}({\bf z})$ for ${\bf z}\sim \mathcal N(0,I_q)$, $\varphi_{ij}:\mathbb R^q\to\mathbb R^p$ a $1$-Lipschitz function and $\lim q/p\in(0,\infty)$. The main results and practical consequences of the article may be summarized as follows:
\begin{enumerate}
    \item under the regime of large dimensional datasets, the MTL-LSSVM algorithm has an asymptotically predictable behavior and thus a predictable performance; in particular, under the further assumptions of two classes per task ($m=2$) and equal identity covariance of the mixture ($\Sigma_{ij}=I_p$ for all $i,j$), this behavior summarizes as a very insightful \emph{small dimensional} (of size the number of tasks $k$ and not the number $n$ or dimension $p$ of the data) functional (i) of all inner products $\Delta\mu_i^\trans\Delta\mu_{i'}$, $i,i'\in\{1,\ldots,k\}$, where $\Delta\mu_i=\mu_{i1}-\mu_{i2}$, (ii) of the proportions $c_{ij}=\lim n_{ij}/n$ between the number of data $n_{ij}$ of class $j$ in Task~$i$ and the overall number $n$ of data, and (iii) of the hyperparameters $\lambda$ (task relatedness) and $\gamma_1,\ldots,\gamma_k$ (task-wise LSSVM regularization parameters) of the MTL problem;
    \item a fundamental aspect of the (LS)SVM framework is to associate each training data ${\bf x}$ to a label $y\in\{-1,1\}$; we demonstrate that this choice is a dangerous source of \emph{negative transfer}; most importantly, we show that to each ${\bf x}$ must be associated an \emph{``optimal'' score}\footnote{This notion of optimality will be properly defined in the article.} rather than a label $y$, which only depends on the class and task of ${\bf x}$; this optimal score is provided in explicit form by the large dimensional analysis; under this choice of optimal scores $y$, the performance of MTL-LSSVM is necessarily improved over parallel independent single-task LSSVMs, and \emph{discards all risks of negative transfer};
    \item a further aspect of the MTL-(LS)SVM approach is that, in a two-class setting, for an unlabelled data ${\bf x}$ to be associated to class $j\in\{1,2\}$ for Task~$i$, a binary decision of the type $g_i({\bf x})\underset{\mathcal{C}_2}{\overset{\mathcal{C}_1}{\gtrless}} 0$ is performed; we show that this decision rule is in general biased, not only due to imbalances in the number of available data per class and per task, but also by the data statistics and the MTL hyperparameters (unless $\Sigma_{ij}=I_p$ for each $i,j$); similar to an optimal choice of the training data ``labels'', in the all-identity covariance setting ($\Sigma_{ij}=I_p$), we establish an optimal threshold $\zeta_i$ which minimizes the probability of misclassification: $\zeta_i$ can be consistently estimated and thus used in practice;
    \item\label{item:concentrated} the assumption of a mixture of concentrated random vectors for the data samples is far from anecdotal: concentrated random vectors form a broad and rich family of random vectors, which can mimic extremely realistic data, as is the case of the output of generative adversarial networks (GANs) proved to be, by definition, concentrated random vectors \citep{seddik2020random}; the article proves a universality result: the asymptotic performance (as $p,n\to\infty$) of MTL-LSSVM only depends on the statistics $\mu_{ij}$ and $\Sigma_{ij}$ of the mixture model, thereby behaving \emph{as if} the data followed a mere Gaussian mixture model; this strongly suggests that the proposed improved algorithm and its performances are applicable to a wide range of real data;
    \item a series of concrete applications, to hypothesis testing using external tasks, to transfer learning, and to multi-class classification are provided, optimized and confronted to competing methods; these applications have the strong advantage to have predictable performances: this is particularly crucial to appropriately set decision thresholds for type I and II errors in hypothesis testing, as well as to predict \emph{before running the algorithms} their anticipated performances;
    \item a simulation campaign on real datasets is performed which (i) confirms, as strongly suggested by Item~\ref{item:concentrated}, the strong adequacy between the empirical and theoretical results and (ii) demonstrates the large superiority of the proposed algorithm over competing methods.
\end{enumerate}
In a nutshell, by exploiting recent advances in applied random matrix theory, the article provides a modern vision to multi-task and transfer learning. This vision is here turned into an elementary but cost-efficient algorithm, which relies on base principles, but which both largely outperforms competing (sometimes complex) methods and provides strong theoretical guarantees.  As a side note, we must insist that our present objective is to study and improve ``data-generic'' multi-task learning mechanisms under no structural assumption on the data; this is quite unlike recent works exploiting convolutive techniques in deep neural nets to perform transfer or multi-task learning mostly for computer vision-oriented tasks, as in e.g., \citep{zhuang2020comprehensive,krishna2019deep}.

\medskip

In order to best capture the main intuitions drawn from the large dimensional analysis, after a rigorous introduction of the multitask learning framework in Section~\ref{sec:framework}, a first highlight of our main contributions under the qualitatively more telling setting of binary tasks ($m=2$) with data of equal identity covariance ($\Sigma_{ij}=I_p$) is proposed in Section~\ref{sec:theoretical_simple}. The technical details under the most generic data modelling setting as well as the most general technical result are then provided in Section~\ref{sec:theoretical_analysis}. A broad series of applications is provided in Section~\ref{application}. Extensive simulations are then proposed in Section~\ref{experiments}, which corroborate our theoretical findings and show their resilience and compatibility to real data settings.
\medskip

\noindent{\bf Reproducibility.} Matlab codes of the main algorithms and results provided in the article are available at \url{https://github.com/maliktiomoko/RMT-MTLLSSVM.git}.

\medskip

\noindent{\bf Notation.} The following notations and conventions will be used throughout the article: $\mathbb{1}_n\in\mathbb R^n$ is the vector of all ones, 
 $e_{m}^{[n]}\in\mathbb{R}^n$ is the canonical vector of $\mathbb{R}^n$ with $[e_{m}^{[n]}]_i=\delta_{mi}$, and $e_{ij}^{[2k]}\equiv e_{2(i-1)+j}^{[2k]}$. Similarly, $E_{ij}^{[n]}\in\mathbb{R}^{n\times n}$ is the canonical matrix of $\mathbb{R}^{n\times n}$ with $[E_{ij}^{[n]}]_{ab}=\delta_{ia}\delta_{jb}$. 
 The notation $A\otimes B$ for matrices or vectors $A,B$ is the Kronecker product. The notation $A\odot B$ for matrices or vectors $A,B$ is the Hadamard product. $\mathcal{D}_{x}$ stands for a diagonal matrix containing on its diagonal the elements of the vector $x$ and $A_{i.}$ is the $i$-th row of matrix $A$.

\section{The Multi Task Learning Framework}
\label{sec:framework}
\subsection{The deterministic setting}
\label{sec:setting}
Let $X\in \mathbb{R}^{p\times n}$ be a collection of $n$ independent data vectors of dimension $p$. The data are divided into $k$ subsets attached to individual ``tasks'', each task consisting of an $m$-class classification problem ($m$ being the same for each task). Specifically, letting $X=[X_1,\ldots,X_k]$, Task~$i$ is a classification problem from the training samples $X_i=[X_{i}^{(1)},\ldots,X_i^{(m)}]\in\mathbb{R}^{p\times n_i}$ with $X_i^{(j)}=[x^{(j)}_{i1},\ldots,x^{(j)}_{in_{ij}}]\in\mathbb{R}^{p\times n_{ij}}$ the $n_{ij}$ vectors of class $\mathcal C_j$, $j\in\{1,\ldots,m\}$, for Task~$i$. In particular, $n=\sum_{i=1}^k n_i$ and $n_i=\sum_{j=1}^m n_{ij}$ for each $i\in\{1,\ldots,k\}$.

To each datum $x^{(j)}_{il}\in\mathbb{R}^p$ of the training set is attached a corresponding output vector (or score) $y^{(j)}_{il}\in\mathbb{R}^{m}$. Correspondingly to the notation $X$, $X_i$ and $X_i^{(j)}$, let $Y=[Y_1^\trans,\ldots,Y_k^\trans]^\trans\in\mathbb{R}^{n\times m}$ be the matrix of the $m$-dimensional outputs of all data, where $Y_i=[Y_i^{(1)\trans},\ldots,Y_i^{(m)\trans}]^\trans\in\mathbb R^{n_i\times m}$ and $Y_i^{(j)}=[y^{(j)}_{i1},\ldots,y^{(j)}_{in_{ij}}]^\trans\in\mathbb{R}^{n_{ij}\times m}$ the matrix of all outputs for Task~$i$. 

In the standard MTL learning approach \citep{evgeniou2004regularized,xu2013}, one would naturally set $y_{il}^{(j)}=e_j^{[m]}$, i.e., all data of class $\mathcal C_j$ are affected a hot-bit in position $j$. As claimed in the introduction and as we shall see, this hot-bit allocation approach is at the source of deleterious performances, such as negative transfer effects, and we thus voluntarily do not enforce any constraint on the vector $y_{il}^{(j)}$ at this point.

\medskip

Before inserting the data-score pairs $(X,Y)$ into the MTL-LSSVM framework, it is convenient to ``center'' the data $X$ to eliminate additional sources of bias. This centering operation could be performed either on the whole dataset $X$, or task-wise on each $X_i$, or even class-wise on each $X_i^{(j)}$. In \citep{evgeniou2004regularized,xu2013} this centering operation is not performed (which essentially boils down to centering $X$ itself). We choose here to center the data task-wise, and this, for two reasons: (i) centering the whole dataset induces dependencies across tasks so that, even by enforcing the hyperplane controlling factor $\lambda$ to decorrelate the tasks (i.e., $\lambda\to \infty$; see next), residual dependence must remain and negative transfer can still appear, (ii) class-wise centering has the double deleterious effect of cancelling an important discrimination factor of the classes (i.e., their difference in statistical mean) and of necessitating a complex treatment to classify new (unlabelled) input data. Inappropriate centering choices would induce biases and undesired residual terms in our theoretical derivation, which further justifies our present task-wise centering choice (see e.g., Remark~\ref{rem:on_centering}). Specifically, the MTL-LSSVM algorithm studied here is based, not on the data $X_i$ but on their centered version
\begin{align*}
    \mathring{X}_i = X_i \left( I_{n_i} - \frac1{n_i}\mathbb{1}_{n_i}\mathbb{1}_{n_i}^\trans \right),\quad \forall i\in\{1,\ldots,k\},
\end{align*}
and we will systematically consider the data-score pair $(\mathring{X},Y)$, where $\mathring{X}=[\mathring{X}_1,\ldots,\mathring{X}_k]$ rather than $(X,Y)$.


\medskip

Having pre-treated the input data, we are in position to introduce the MTL-LSSVM framework. The MTL-LSSVM algorithm aims to predict, relative to each task $i$, an output score vector ${\bf y}_i\in\mathbb{R}^{m}$ for any new input vector ${\bf x}\in\mathbb{R}^p$.
To this end, MTL-LSSVM determines $k$ ``hyperplane normal-vector'' matrices $W=[W_1,W_2,\ldots,W_k]\in\mathbb{R}^{p\times km}$ which take the form $W_i=W_0+V_i$ for some common $W_0$ and individual task-wise matrices $V=[V_1,\ldots,V_k]$ and biases $b=[b_1^\trans,b_2^\trans,\ldots, b_k^\trans]^{\trans}\in\mathbb{R}^{k\times m}$. These parameters are set to minimize the objective function
%
\begin{align}
\label{eq:optimization}
    \min_{(W_0,V,b)\in\mathbb{R}^{p\times m}\times \mathbb{R}^{p\times km}\times \mathbb{R}^{k\times m}} \mathcal{J}(W_0,V,b)
\end{align}
where
\begin{align*}
    \label{eq:optimization_problem}
    \mathcal{J}(W_0,V,b)&\equiv \frac1{2\lambda} \tr\left(W_0^\trans W_0\right)+\frac 1{2}\sum_{i=1}^k \frac{\tr\left( V_i^\trans V_i\right)}{\gamma_i}+\frac 12\sum_{i=1}^k\tr\left(\xi_i^\trans\xi_i\right) \\
    \xi_i&=Y_i-(\frac{\mathring X_i^{\trans}W_i}{kp}+\mathbb{1}_{n_i}b_i^{\trans}),\quad \forall i\in\{1,\ldots,k\}.
\end{align*}
This is a classical LSSVM formulation in which the quadratic cost $\tr(\xi_i^\trans\xi_i)$ replaces the boundary constraint of margin-based SVM and where the costs $\tr(W_0^\trans W_0)$ and $\tr(V_i^\trans V_i)$ are reminiscent of the hyperplane normal-vector norm minimization of classical SVM.

What is specific to the MTL approach is first the hyperparameter $\lambda$ which enforces or relaxes the relatedness between tasks and the introduction of $k$ extra parameters $\gamma_1,\ldots,\gamma_k$ which enforce a correct classification of the data in their respective classes. Similarly to \citep{evgeniou2004regularized}, we place the hyperparameters $\gamma_i$ as a prefactor of $\tr( V_i^\trans V_i)$, rather than as a prefactor of $\tr(\xi_i^\trans \xi_i)$; this differs from the normalization scheme proposed in \citep{xu2013}. This choice is more flexible in the following sense: for a fixed value of $\lambda$, increasing all ratios $\frac{\lambda}{\gamma_i}$ ``blurs'' the difference between tasks and thus turns the optimization scheme into a single-task SVM (because the optimal $V_i$'s need then be set to zero in the limit); for fixed values of the $\gamma_i$'s instead, small ratios $\frac{\lambda}{\gamma_i}$ decorrelate the tasks (the optimal $W_0$ being close to zero). Note however that, unlike in \citep{evgeniou2004regularized}, we choose to use here one hyperparameter $\gamma_i$ per task instead of a common one. As will be seen next, this choice is more meaningful and of course offers more flexibility.

In passing, remark that the linear common-hyperplane condition $W_i=W_0+V_i$, imposes by definition that all $V_i$'s be of the same size $\mathbb R^{p\times m}$: this severely constrains (i) the data in each task to be of the same dimension $p$ and (ii) the number of classes per task to be the same ($m$). Further linear or even non-linear relaxation schemes for $W_i$ of the type $W_i=V_i+f_i(W_0)$ for some operator $f_i$ could be envisioned to relax this constraint. This however goes beyond the scope of the present article, which seeks to provide insights and optimality into a simplified (yet already non-trivial) form of MTL-LSSVM.

\medskip

As for the choice of the hyperparameters $\lambda$, $\gamma_1,\ldots,\gamma_k$, as well as of the score matrix $Y$ which we recall was left open, it is treated independently and is dictated, not by the present optimization scheme, but by a ultimate objective, such as minimizing the misclassification 
rate for a specific target class. These more applied considerations will be made in Section~\ref{application}.

\begin{remark}[LSSVM classification versus regression]
It may be disputed that the optimization framework \eqref{eq:optimization} takes a regression rather than a classification form. It appears that, under a binary-class LSSVM framework with scores $y_i\in\{\pm 1\}$, the classification constraint (of the form $y_i(W^{T}x_i+b_i)-1=\xi_i$) or the regression constraint (of the form $y_i-W^{T}x_i+b_i=\xi_i$) are associated to the same losses, thereby leading to the same classification solution and performance. Yet, as will become clear in the following, in addition for the solution of \eqref{eq:optimization} to be explicit and theoretically tractable (which is not the case of alternative schemes such as margin-based SVM, logistic regression, Adaboost, etc.), the aforementioned flexibility in the score matrix $Y$ largely outbalances the ``failure'' of treating a classification problem by means of a regression optimization scheme. Besides, under the large dimensional theoretical framework presently studied, recent works in related problems \citep{mai2019high} forcefully suggest that the square loss is optimal to deal with large dimensional data as it uniformly outperforms all alternative cost functions.
\end{remark}

\medskip

Being a quadratic cost optimization under linear constraints, \eqref{eq:optimization} is easily solved using its dual formulation by introducing Lagrangian parameters $\alpha_i\in\mathbb{R}^{n_i\times m}$ for each task $i$ (see details in Section~\ref{MTL_solution}). The solution is explicit and is as follows.
\begin{proposition}
\label{prop:solution_optimization}
The solution to \eqref{eq:optimization} is given by 
\begin{align*}
W_0&= \left(\mathbb{1}_k^\trans\otimes \lambda I_{p}\right)Z\alpha\\
W_i &=\left({e_{i}^{[k]}}^\trans\otimes I_{p}\right)AZ\alpha \\
b &=(P^\trans Q P)^{-1}P^{\trans}QY
\end{align*}
where 
\begin{align*}
     Z&=\begin{pmatrix}
     \mathring X_1\\
     &\xddots\\
     &&\mathring X_k\end{pmatrix}\in\mathbb{R}^{kp\times n} \\
     A&=\left(\mathcal{D}_{\gamma}+\lambda\mathbb{1}_k\mathbb{1}_k^\trans\right)\otimes I_p\in\mathbb{R}^{kp\times kp} \\
    \alpha &= Q(Y-Pb), \quad Q=\left(\frac1{kp}Z^\trans AZ+I_{n}\right)^{-1}\in\mathbb{R}^{n\times n} \\
    P&=\begin{pmatrix}
     \mathbb{1}_{n_1}\\
     &\xddots\\
     &&\mathbb{1}_{n_k}\end{pmatrix}\in\mathbb{R}^{n\times k}.
\end{align*}
\end{proposition}

Despite the apparent intricate expression of $W_i$, it must be stressed that $W_i$ ``essentially'' takes the form of the standard solution to a ridge regression (or regularized least-square) problem as the term $AZQY$ (in which $Q=(\frac1{kp}Z^\trans AZ+I_n)^{-1}$) appearing in the expended form of $W_i$ confirms. From a technical standpoint, the large dimensional statistical behavior of the matrix $Q$, known as the \emph{resolvent} of $\frac1{kp}Z^\trans AZ$ in random matrix theory, plays a central role in the analysis. More specific to the MTL framework, note the interesting isolation of the data subsets $\mathring X_i$ in the data matrix $Z$ (it is not possible, to the best of our knowledge, to ``linearly'' express $W_i$ as a function of $\mathring X$ itself); the elements $\mathring X_i$ are then ``mixed'' by the term $\lambda\mathbb 1_k\mathbb 1_k^\trans$ appearing in matrix $A$, from which it naturally comes that, in the limit $\lambda\to 0$, MTL-LSSVM boils down to $k$ independent LSSVMs with $\mathcal D_\gamma$ imposing weights $\gamma_1,\ldots,\gamma_k$ on each data subset.

\medskip

From Proposition~\ref{prop:solution_optimization}, for any new data point ${\bf x}\in\mathbb{R}^p$, the classification score vector $g_i({\bf x})\in\mathbb R^m$ for Task~$i$, is then defined by
\begin{equation}
\label{eq:score_class}
    g_i({\bf x})= \frac{1}{kp} W_i^\trans\mathring{\bf x}+b_i=\frac{1}{kp} \alpha^\trans Z^\trans A \left(e_{i}^{[k]}\otimes \mathring{\bf x}\right)+b_i
\end{equation}
where $ \mathring{\bf x}={\bf x}-\frac{1}{n_i}X_i\mathbb{1}_{n_i}$ is a centered version of $\bf x$ with respect to the training dataset for Task~$i$. 

This formulation, along with the next remark, confirm again the relevance of a task-wise, rather than class-wise, centering of the data $X$, which allows for a well-defined expression of $\mathring{\bf x}$.

\begin{remark}[Shift invariance of the scores]
\label{rem:on_centering}
If the columns of $Y_i\in\mathbb{R}^{n_i\times m}$ are shifted by some constant vector $P\bar{\mathcal{Y}}$ for some (small dimensional) matrix $\bar{\mathcal{Y}}\in\mathbb{R}^{k\times m}$, i.e., if all data of the same task are affected by the same shift of their scores (or labels), then we find that the Lagrangian parameter $\alpha^{\rm shift}$ after the shift is 
    \begin{align*}
        \alpha^{\rm shift} &= Q\left( I_n - P(P^\trans QP)^{-1}P^\trans Q \right) (Y+P\bar{\mathcal{Y}}) = \alpha.
    \end{align*}
As such, the matrix $W_i=(e_i^{[k]\trans}\otimes I_p)A Z\alpha$ and, consequently, the performance of MTL-LSSVM are insensitive to a simultaneous shift of all the scores of each task.
\end{remark}

\subsection{Statistical modelling and the large dimensional setting}
\label{sec:model_stat}

In order to draw insights into the behavior of MTL-LSSVM and evaluate its performance, the article proposes to first model the dataset $X$ as a mixture of concentrated random vectors and then to assume the dimensions $p,n$ of $X$ to be sufficiently large for deterministic (and predictable) concentration behavior to occur.

\begin{assumption}[Distribution of $X$ and ${\bf x}$]
\label{ass:distribution}
There exist two constants $C,c >0$ (independent of $n,p$) such that, for any $1$-Lipschitz function $f:\mathbb{R}^{p\times n}\to \mathbb{R}$,
\begin{equation*}
\forall t>0,~\mathbb{P}( |f(X)- m_{f(X)} | \geq t)\leq Ce^{-(t/c)^2}
\end{equation*}
where $m_Z$ is a median of the random variable $Z$. We further impose that the columns of $X$ be independent and that the $x_{il}^{(j)}$, for $l\in\{1,\ldots,n_{ij}\}$, be distributed according to the same law $\mathcal L_{ij}$. These conditions guarantee the existence of a mean and covariance for the columns of $X$ and we denote, for all $l\in\{1,\ldots,n_{ij}\}$,
\begin{align*}
    \mu_{ij} &\equiv \mathbb E[x_{il}^{(j)}] \\
    \Sigma_{ij} &\equiv \mathrm{Cov}[x^{(j)}_{il}].
\end{align*}
Furthermore, the dummy variable ${\bf x}\in\mathbb R^p$ used for testing is independent of $X$, and distributed according to one of the laws $\mathcal L_{ij}$. 
\end{assumption}

Assumption~\ref{ass:distribution} notably encompasses the following scenarios: the $x^{(j)}_{il}$'s are (i) independent Gaussian random vectors $\mathcal N(\mu_{ij},\Sigma_{ij})$, (ii) independent random vectors uniformly distributed on the $\mathbb{R}^p$ sphere of radius $\sqrt{p}$ and, most importantly, (iii) any $1$-Lipschitz transformation $\varphi_{ij}(z_{il}^{(j)})$ with $z_{il}^{(j)}$ itself a concentrated random vector. Scenario~(iii) is particularly relevant to model very realistic data by means of advanced non-linear generative models, as recently demonstrated in \citep{seddik2019kernel} in the specific example of generative adversarial networks (GANs). As such, Assumption~\ref{ass:distribution} offers the flexibility to assume either synthetic Gaussian mixture models, or very realistic and advanced generative data models. A core result of the present article consists in showing that, for $n,p$ large, either scenario leads to the same asymptotic performance for MTL-LSSVM (which thus only depends on the statistical means and covariances of the data).

\medskip

Since all data $x^{(j)}_{il}$, $l\in\{1,\ldots,n_{ij}\}$, are identically distributed, we will further impose that their associated scores $y_{il}^{(j)}\in\mathbb R^m$ be identical. That is, $y_{i1}^{(j)}=\ldots=y_{in_{ij}}^{(j)}\equiv \mathcal{Y}_{ij}$ within every class $j$ of each task $i$. The score matrix $Y\in\mathbb{R}^{n\times k}$ may then be reduced under the form
\begin{align*}
    Y &= \left[ \mathcal{Y}_{11}\mathbb{1}^\trans_{n_{11}},\ldots,\mathcal{Y}_{km}\mathbb{1}^\trans_{n_{km}} \right]^\trans \in\mathbb R^{n\times m}
\end{align*}
for $\mathcal Y=[\mathcal{Y}_{11},\ldots, \mathcal{Y}_{km}]^\trans\in\mathbb{R}^{km\times m}$. From Remark~\ref{rem:on_centering}, it is also clear that, the performances of MTL-LSSVM being insensitive to a constant shift in the scores $\mathcal{Y}_{i1},\ldots,\mathcal{Y}_{im}$ in every given task $i$, the centered version $\mathring{\mathcal Y}=[\mathring{\mathcal Y}_{11},\ldots,\mathring{\mathcal Y}_{km}]^\trans$ of $\mathcal Y$, where
\begin{align*}
    \mathring{\mathcal Y}_{ij} &\equiv \mathcal Y_{ij} - \sum_{j=1}^m \frac{n_{ij}}{n_i} \mathcal Y_{ij},
\end{align*}
will naturally appear at the core of the upcoming results.

\medskip

Although practical data will of course be considered to be of finite dimension $p$ and number $n$, it will indeed be convenient, for technical reasons, to work under the following large dimensional random matrix assumption.
\begin{assumption}[Growth Rate]
\label{ass:growth_rate}
As $n\to \infty$, $n/p\to c_0\in(0,\infty)$ and, for $1\leq i\leq k$, $1\leq j\leq m$, $n_{ij}/n\to c_{ij}\in (0,1)$. We further denote $c_i=\sum_{j=m}^kc_{ij}$ and $c=[c_1,\ldots,c_k]^\trans\in\mathbb{R}^k$.
\end{assumption}

\medskip

With these notations and assumptions in place, we are in position to present the main results of the article. Yet, before entering the technical details of the large dimensional analysis of the performance of the MTL-LSSVM framework, the next section first provides a highlight of the main contributions and intuitions drawn by the analysis. To this end, it is convenient to temporarily restrict the setting to binary classes ($m=2$) and to an isotropic mixture model for the data $X$, i.e., $\Sigma_{ij}=I_p$ for each measure $\mathcal L_{ij}$. The most general and slightly more technical setting ($m\geq 2$ and non-isotropic mixture data modelling) is considered in full in Section~\ref{sec:theoretical_analysis}.

\section{Highlights of the main results}
\label{sec:theoretical_simple}

To simplify the exposition of our main results, without impacting their core conclusions, in this section, Assumptions~\ref{ass:distribution}--\ref{ass:growth_rate} are further restricted to the binary-classification setting ($m=2$) and to measures $\mathcal L_{ij}$ of equal covariance $\Sigma_{ij}=I_p$, for all $i,j$.

The advantage of the isotropic ($\Sigma_{ij}=I_p$) condition is that all asymptotic results can be expressed under the form of low-dimensional matrix formulations (of size scaling with $k$ but not with $p,n$). Adjoined to the $m=2$ assumption, the isotropic model further guarantees a simplified form for (i) the (asymptotically) optimal labels $Y$, (ii) the optimal decision thresholds $\zeta_i$, and (iii) the asymptotic performances of MTL-LSSVM, all of which can be estimated consistently as $p,n\to\infty$. Consequently, this simplified setting has the strong benefit to give rise to a first cost-efficient and robust multitask classification algorithm (Algorithm~\ref{alg:binary_algorithm}) which, for practical data, makes the approximation that $\Sigma_{ij}\propto I_p$.

\medskip

 The binary setting does not a priori alter any of the previously introduced notations which stand with $m=2$. Yet, it is particularly convenient in this setting to recast the score vectors $y_{il}^{(j)}\in\mathbb R^m$ into scalar scores $y_{il}^{{\rm bin}(j)}\in\mathbb R$. In a standard classification context, this would correspond to turning a two-dimensional hot-bit vector $e_j^{[2]}$ into a signed scalar $\pm 1$; as we recall that $y_{il}^{(j)}$ is here considered as a \emph{real} score (rather than a binary label) vector, to us this is equivalent to turning a score vector into a scalar score. Matrix $Y\in\mathbb R^{n\times m}$ similarly now becomes a score vector $y^{\rm bin}\in\mathbb R^n$, and in particular we define $\mathring{\mathcal y}^{\rm bin}=[\mathring{\mathcal y}^{\rm bin}_{11},\ldots,\mathring{\mathcal y}^{\rm bin}_{k2}]^\trans\in\mathbb R^{2k}$ with
\begin{align*}
    \mathring{\mathcal y}^{\rm bin}_{ij} &\equiv \mathcal y^{\rm bin}_{ij} - \left( \frac{n_{i1}}{n_i} \mathcal y^{\rm bin}_{i1} + \frac{n_{i2}}{n_i}\mathcal y^{\rm bin}_{i2}\right) \in \mathbb R
\end{align*}
where $\mathcal y^{\rm bin}_{ij}\equiv {{}y^{\rm bin}_{i1}}^{(j)}=\ldots={{}y^{\rm bin}_{in_{ij}}}^{(j)}\in\mathbb R$ is the common score assigned to the identically distributed data of class $j$ for Task~$i$.
Correspondingly, the sought-for $(W_i,b_i)$ collection of $m$ hyperplanes of \eqref{eq:optimization} becomes a single hyperplane $(w^{\rm bin}_i,b^{\rm bin}_i)$ with $w^{\rm bin}_i\in\mathbb R^p$ and $b^{\rm bin}_i\in\mathbb R$. Yet, our present interest is only on the resulting score vector $g_i({\bf x})$ which, replacing $Y$ by $y^{\rm bin}$ in its expression (Equation~\ref{eq:score_class}), becomes the scalar test score
\begin{align*}
    g^{\rm bin}_i({\bf x}) &\equiv \frac1{kp}(y^{\rm bin}-Pb)^\trans Q Z^\trans A\left(e_i^{[k]}\otimes \mathring{\bf x}\right)+[(P^\trans QP)^{-1}P^\trans Qy^{\rm bin}]_i \in\mathbb R.
\end{align*}

\subsection{Theoretical analysis and large dimensional intuitions}

Under the isotropic and binary-class setting, as $n,p\to\infty$ according to Assumption~\ref{ass:growth_rate}, the theoretical performance of MTL-LSSVM explicitly depends on two fundamental and isolated quantities: the data-related matrix $\mathcal M\in\mathbb{R}^{2k\times 2k}$ and the hyperparameter matrix $\mathcal A\in\mathbb{R}^{k\times k}$:
\begin{align*}
    \mathcal M &= \sum_{i,i'=1}^k \Delta\mu_i^\trans\Delta\mu_{i'} \left( E_{ii'}^{[k]}  \otimes \mathbb{c}_i\mathbb{c}_{i'}^\trans \right)
     \\
    \mathcal{A}&=\left(I_k+\mathcal{D}_{\bm\delta^{[k]}}^{-\frac 12}\left(\mathcal{D}_{\gamma}+\lambda\mathbb{1}_{k}\mathbb{1}_{k}^\trans\right)^{-1}\mathcal{D}_{\bm\delta^{[k]}}^{-\frac 12}\right)^{-1}
\end{align*}
where we introduced the shortcut notations
\begin{align*}
     \Delta\mu_i &\equiv \mu_{i1}-\mu_{i2},\quad \mathbb{c}_i \equiv
     \sqrt{ c_{i1}/c_i } \sqrt{ c_{i2}/c_i }\begin{bmatrix} \sqrt{ c_{i2}/c_i } \\ - \sqrt{ c_{i1}/c_i } \end{bmatrix}
\end{align*}
and where ${\bm{\delta}}^{[k]}=[\bm\delta^{[k]}_1,\ldots,\bm\delta^{[k]}_k]^\trans$ 
are the unique positive solutions to the implicit system of $k$ equations
\begin{align}
    \label{eq:tildeDelta}
    \bm\delta^{[k]}_i=\frac{c_i}{c_0}-\mathcal{A}_{ii},\quad i\in\{1,\ldots,k\}.
 \end{align}
In anticipation of future needs, it is convenient to further introduce the $2k$-dimensional variant ${\bm\delta}^{[2k]}=[{\bm\delta}^{[2k]}_{11},\ldots,{\bm\delta}^{[2k]}_{k2}]^\trans\in\mathbb{R}^{2k}$ where
\begin{align}
\label{eq:barDelta}
    {\bm\delta}^{[2k]}_{ij} &= c_0\frac{c_{ij}}{c_i} \bm\delta^{[k]}_i.
\end{align}

The asymptotic performances of MTL-LSSVM will be shown to solely depend on $X$ through the matrices $\mathcal M$ and $\mathcal A$, which thus play the role of (asymptotically) \emph{sufficient statistics}. It is particularly important to stress that, despite the quite generic concentration assumption on $X$ (Assumption~\ref{ass:distribution}), when $\Sigma_{ij}=I_p$, only the $k^2$ inner products $\Delta\mu_i^\trans\Delta\mu_{i'}$ and the $2k$ \emph{class-wise} dimensionality ratios $c_{ij}/c_i$ intervene in the expression of $\mathcal M$ -- so in particular none of the higher order moments of $X$ are accounted for, nor the absolute task-wise dimension ratios $c_i$.
As for $\mathcal A$, it captures instead the information about the impact of the hyperparameters $\lambda,\gamma_1,\ldots,\gamma_k$ as well as the \emph{task-wise} dimensionality ratios $c_1,\ldots,c_k$ and the data number-to-dimension ratio $c_0$.
In the expression of the MTL-LSSVM performance, these two matrices combine into the core matrix $\mathcal{\Gamma}\in\mathbb{R}^{2k\times 2k}$
\begin{align}
\label{eq:Gamma_simple}
    \Gamma &=\left(I_{2k}+\left(\mathcal{A}\otimes \mathbb{1}_{2}\mathbb{1}_{2}^{\trans}\right)\odot \mathcal M\right)^{-1}
\end{align}
where we recall that `$\odot$' is the Hadamard (element-wise) matrix product.

\begin{theorem}[Asymptotics of $g^{\rm bin}_i({\bf x})$]
\label{th:simple_theorem}
Under Assumptions~\ref{ass:distribution}--\ref{ass:growth_rate}, with $m=2$ and $\Sigma_{ij}=I_p$, for a test data ${\bf x}$ with $\mathbb{E}[{\bf x}]=\mu_{ij}$ and $\mathrm{Cov}[{\bf x}]=I_p$, as $p,n\to\infty$,
\begin{equation*}
    g^{\rm bin}_i({\bf x}) - G_{ij} \asto  0,\quad G_{ij} \sim \mathcal{N}(\mathcal{m}_{ij},\sigma_i^2)
\end{equation*}
in distribution, where, letting $\mathcal{m}=[\mathcal{m}_{11},\ldots,\mathcal{m}_{k2}]^\trans$ and the normalized forms ${\bm{\mathcal{ y}}}^{\rm bin}\equiv \mathcal{D}_{{\bm\delta}^{[2k]}}^{\frac12} \mathcal{y}^{\rm bin}$, $\mathring{\bm{\mathcal{y}}}^{\rm bin}= \mathcal{D}_{{\bm\delta}^{[2k]}}^{\frac12} \mathring{\mathcal{y}}^{\rm bin}$, ${\bm{\mathcal{m}}} = \mathcal{D}_{{\bm\delta}^{[2k]}}^{\frac12} \mathcal{m}$, and ${\bm\sigma}^2_i = \bm\delta^{[k]}_i\sigma_i^2$,
\begin{align*}
    {\bm{\mathcal{m}}} &= \bm{\mathcal{y}}^{\rm bin}-\Gamma\mathring{\bm{\mathcal{ y}}}^{\rm bin} \\
    {\bm\sigma}_i^2&=(\mathring{\bm{\mathcal{y}}}^{\rm bin})^\trans\Gamma\mathcal{V}_i\Gamma \mathring{\bm{\mathcal{y}}}^{\rm bin}
\end{align*}
with
\begin{align*}
    \mathcal{V}_i &=\mathcal{D}_{ \mathcal{K}_{i.}^\trans\otimes\mathbb{1}_2}+\left(\mathcal{A}\mathcal{D}_{\mathcal{K}_{i\cdot}^\trans+e_{i}^{[k]}}\mathcal{A}\otimes \mathbb{1}_{2}\mathbb{1}_{2}^{\trans}\right)\odot \mathcal{M} \\
    \mathcal K &=  \frac{c_0}{k}[\mathcal A\odot \mathcal A]\left(\mathcal{D}_c -\frac{c_0}{k}[\mathcal A\odot \mathcal A]\right)^{-1}.
 \end{align*}
\end{theorem}

Theorem~\ref{th:simple_theorem} interestingly indicates that the (asymptotic) statistics of the classification scores $g^{\rm bin}_i({\bf x})$, for $1\leq i\leq k$, reduce to a mere functional of \emph{$2k$-dimensional deterministic vectors and matrices}. In particular, $g^{\rm bin}_i({\bf x})$ depends on the data statistical means $\mu_{i'j'}$, $1\leq i'\leq k$, $1\leq j'\leq 2$, and on the hyperparameters $\lambda$ and $\gamma_1,\ldots,\gamma_k$ mostly through the $2k$-dimensional matrix $\Gamma$ (and more marginally through $\mathcal{V}_i$ and $\mathcal K$ for the variances). 

Another non-trivial point to note is that, being in general non-diagonal, $\Gamma$ acts on the centered scores (labels) $\mathring{\mathcal{y}}_{i'j'}^{\rm bin}$ of all classes $j'$ and tasks $i'$ which, therefore, all influence the performances. It can thus be anticipated that, for the decision on a particular Task~$i$ to be successful, not only the scores $\mathcal y_{i1}^{\rm bin}$ and $\mathcal y_{i2}
^{\rm bin}$, but in fact all scores $\mathcal y_{i'j'}^{\rm bin}$ across all classes and tasks, must be appropriately tuned. 

Remark also that, in this isotropic ($\Sigma_{i'j'}=I_p$) setting, the variance $\sigma_i^2$ of the score $g_i^{\rm bin}({\bf x})$ with $\mathbb E[{\bf x}]=\mu_{ij}$ only depends on $i$, and not on $j$. This is particularly convenient, as shown next, to devise an optimal decision rule for classification into class $1$ or $2$ for Task~$i$.

\medskip

From a more technical standpoint, comparing the exact expression of $g_i({\bf x})$ in \eqref{eq:score_class} and that of $\mathcal{m}_{ij}$ (i.e., the large dimensional approximation of $\mathbb{E}[g_i({\bf x})]$), we may interpret the matrix $\Gamma\in\mathbb{R}^{2k\times 2k}$ as a ``condensed'' form of $Q\in\mathbb{R}^{n\times n}$.
From the expression $(I_{2k}+\mathcal A \otimes \mathbb{1}_2\mathbb{1}_2^\trans)^{-1}\odot \mathcal M$, observe that: (i) if $\lambda\ll 1$, then $\mathcal A$ is diagonal dominant and thus ``filters out'' in the Hadamard product all off-diagonal entries of $\mathcal M$ -- that is, all the cross-terms $\Delta\mu_i^\trans\Delta\mu_j$ for $i\neq j$ --, therefore refusing to exploit the correlation between tasks; (ii) if instead $\lambda\sim 1$, then $\mathcal A$ may be developed (using the Sherman-Morrison matrix inverse formulas) as the sum of a diagonal matrix, which again filters out the $\Delta\mu_i^\trans\Delta\mu_j$ for $i\neq j$, and of a rank-one matrix which instead performs a weighted sum (through the $\gamma_i$ and the $\bm\delta^{[k]}_i$) of the entries of $\mathcal M$; specifically, letting $\gamma^{-1}=(\gamma_1^{-1},\ldots,\gamma_k^{-1})^\trans$, we have
\begin{align*}
    \left(D_\gamma+\lambda \mathbb{1}_k\mathbb{1}_k^\trans\right)^{-1} &= D_\gamma^{-1}-\frac{\lambda\gamma^{-1}(\gamma^{-1})^\trans}{1+\lambda \frac1k\sum_{i=1}^k\gamma_i^{-1}}.
\end{align*}
As such, letting aside the regularization effect of the $\bm\delta^{[k]}_i$'s, the off-diagonal $\Delta\mu_i^\trans\Delta\mu_j$ term intervening in the expression of $\mathcal M$ is weighted by a coefficient $(\gamma_i\gamma_j)^{-1}$: the impact of the $\gamma_{i'}$'s is thus strongly associated to the relevance of the correlation between tasks, and not only to the individual performances of the $k$ isolated LSSVM tasks.

\medskip

Section~\ref{sec:theoretical_analysis} provides a more general version (Theorem~\ref{th:main}) of Theorem~\ref{th:simple_theorem} for $m\geq 2$ classes per task and generic $\Sigma_{ij}$. The technical derivation of these two results, of limited interest at this point of the article, is also deferred to Section~\ref{sec:theoretical_analysis}.

\subsection{Decision threshold and label optimization}
Since $g_i^{\rm bin}({\bf x})$ has a Gaussian limit centered about $\mathcal{m}_{ij}$ and with equal variance for $j=1$ and $j=2$, the (asymptotically) optimal decision for $\bf x$ to be allocated to class~$\mathcal C_1$ or class~$\mathcal C_2$ for Task~$i$, i.e., the decision minimizing the averaged error probability under the prior $\mathbb{P}({\bf x}\in\mathcal C_1)=\mathbb{P}({\bf x}\in\mathcal C_2)$, is obtained by the ``averaged-mean'' test
\begin{align}
\label{eq:am_test}
    g_i^{\rm bin}({\bf x})\underset{\mathcal{C}_2}{\overset{\mathcal{C}_1}{\gtrless}} \zeta_i \equiv\frac 12 \left(\mathcal{m}_{i1}+\mathcal{m}_{i2}\right)
\end{align}
the associated misclassification rate being
\begin{align}
\label{eq:classification}
    \epsilon_{i1} &\equiv \mathbb{P}\left(g_i^{\rm bin}({\bf x})\geq\frac{\mathcal{m}_{i1}+\mathcal{m}_{i2}}2\Big|{\bf x}\in\mathcal{C}_1\right) \nonumber \\
    &=\mathcal Q\left(\frac{ \mathcal{m}_{i1}-\mathcal{m}_{i2}}{2\sigma_i}\right)+o(1)
\end{align}
with $\mathcal{m}_{ij}$, $\sigma_i$ as in Theorem~\ref{th:simple_theorem} and ${\mathcal Q(t)=\frac1{\sqrt{2\pi}}\int_{t}^\infty e^{-\frac{u^2}{2}}du}$.

\medskip

It is of utmost interest at this point to recall that the asymptotics of $g_i^{\rm bin}({\bf x})$ from Theorem~\ref{th:simple_theorem} (as from the more generic Theorem~\ref{th:main}) depend in an elegant and simple manner on the training data scores $\mathcal y^{\rm bin}=\mathcal{D}_{{\bm\delta}^{[2k]}}^{-\frac 12}\bm{\mathcal {y}}^{\rm bin}$. Using again the independence of $\sigma_i^2$ on the genuine class of ${\bf x}$, the vector ${{}\bm{\mathcal {y}}^{\rm bin}}^\star$ minimizing the misclassification rate for Task~$i$ simply reads:
\begin{align*}
    {{}\bm{\mathcal {y}}^{\rm bin}}^\star &=\argmax_{\bm{\mathcal { y}}^{\rm bin}\in\mathbb{R}^{2k}} \frac{({ \mathcal{m}}_{i1}-{ \mathcal{m}}_{i2})^2}{\sigma_i^2}\\
    &=\argmax_{\bm{\mathcal { y}}^{\rm bin}\in\mathbb{R}^{2k}}\frac{\|({{}\bm{\mathcal {y}}^{\rm bin}})^\trans (I_{2k}-\Gamma)\mathcal{D}_{{\bm\delta}^{[2k]}}^{-\frac 12}(e_{i1}^{[2k]}-e_{i2}^{[2k]})\|^2}{({{}\bm{\mathcal{y}}^{\rm bin}})^\trans\Gamma\mathcal{V}_i\Gamma\bm{\mathcal{y}}
   ^{\rm bin}}
\end{align*}
for which the solution is explicitly defined, \emph{up to an arbitrarily multiplicative constant (as it maximizes a ratio) and up to an arbitrary additive constant (as per Remark~\ref{rem:on_centering})}, by:
\begin{equation}
\label{eq:ty_opt_simple}
    {{}\bm{\mathcal {y}}^{\rm bin}}^\star=\Gamma^{-1}\mathcal V_i^{-1}[(\mathcal A\!\otimes\!\mathbb 1_2\mathbb 1_2^\trans)\odot \mathcal M]\mathcal D_{{\bm\delta}^{[2k]}}^{-\frac 12}(e_{i1}^{[2k]}-e_{i2}^{[2k]}).
\end{equation}
and, for this choice of ${{}\bm{\mathcal{y}}^{\rm bin}}^\star$, the corresponding (asymptotically) optimal classification error $\epsilon_{i1}$ defined in \eqref{eq:classification} is then
\begin{align}
\label{eq:epsilon_star}
    \epsilon_{i1}^\star&=\mathcal Q\left(\frac 12\sqrt{(e_{i1}^{[2k]}-e_{i2}^{[2k]})^\trans\mathcal{G}(e_{i1}^{[2k]}-e_{i2}^{[2k]})}\right)
\end{align}
for 
$\mathcal G=\mathcal D_{{\bm\delta}^{[2k]}}^{\frac12}[(\mathcal A\!\otimes\!\mathbb 1_2\mathbb 1_2^\trans)\odot \mathcal M]\mathcal V_i^{-1}[(\mathcal A\!\otimes\!\mathbb 1_2\mathbb 1_2^\trans)\odot \mathcal M]\mathcal D_{{\bm\delta}^{[2k]}}^{\frac12}$.
Of course, by symmetry, $\epsilon_{i2} \equiv P(g_i^{\rm bin}({\bf x})\leq\frac{\mathcal{m}_{i1}+\mathcal{m}_{i2}}2|{\bf x}\in\mathcal{C}_2)$ has the same limiting optimal value $\epsilon_{i2}^\star=\epsilon_{i1}^\star$.

\medskip

The only non-diagonal matrices in \eqref{eq:ty_opt_simple} are $\Gamma$ and $\mathcal V_i$ in which $\mathcal M$ plays the role of a ``variance profile'' matrix. In particular, assume $\Delta\mu_i^\trans\Delta\mu_{i'}=0$ for all $i'\neq i$, i.e., the differences in statistical means of all tasks are orthogonal to those of Task~$i$. Then the two rows and columns of $\mathcal M$ associated to Task~$i$ are all zero but on the $2\times 2$ diagonal block. Therefore, ${{}\mathcal y
^{\rm bin}}^\star$ will have all zero entries but on its Task~$i$ two elements. All other choices for the null entries of ${{}\mathcal y^{\rm bin}}^\star$ (such as the usual $\mathcal y^{\rm bin}=[1,-1,\ldots,1,-1]^\trans$) would be suboptimal and (possibly severely) detrimental to the classification performance of Task~$i$, not by altering the means $\mathcal{m}_{i1},\mathcal{m}_{i2}$ but \emph{by increasing the variance} $\sigma_i^2$.
This extreme example strongly suggests that, in order to maximize the MTL performance on a targeted Task~$i$, one must impose low absolute scores $\mathcal y^{\rm bin}_{i'j}$ to all Tasks~$i'$ strongly different from Task~$i$. 

\medskip

The choice $\mathcal y^{\rm bin}=[1,-1,\ldots,1,-1]^\trans$ can also be very detrimental when $\Delta\mu_i^\trans\Delta\mu_{i'}<0$ for some pair $i,i'$: that is, when the mapping of the two classes within each task is reversed (e.g., if class~$\mathcal C_1$ in Task~$1$ is closer to class~$\mathcal C_2$ than class~$\mathcal C_1$ in Task~$2$). In this setting, it is easily seen that $\mathcal y^{\rm bin}=[1,-1,\ldots,1,-1]^\trans$ works against the classification and performs much worse than a single-task LSSVM.

\medskip

Another interesting conclusion arises from the simplified setting of equal number of samples per task and per class, i.e., $n_{11}=\ldots=n_{k2}$. In this case, ${\bm\delta_{11}^{[k2]}}=\ldots={\bm\delta_{k2}^{[k2]}}$ and, since ${{}\mathcal y^{\rm bin}}^\star$ is defined up to a multiplicative constant, we have
$${{{}\mathcal y^{\rm bin}}^\star=\Gamma^{-1}\mathcal V_i^{-1}\left((\mathcal A\otimes\mathbb 1_2\mathbb 1_2^\trans)\odot \mathcal M\right)(e_{i1}^{[2k]}-e_{i2}^{[2k]})}$$ in which all matrices are organized in $2\times 2$ blocks of equal entries. This immediately implies that ${{}\mathcal y_{i'1}^{\rm bin}}^\star=-{{}\mathcal y_{i'2}^{\rm bin}}^\star$ for all $i'$. So in particular, the detection threshold $\frac12(\mathcal{m}_{i1}+\mathcal{m}_{i2})$ of the averaged-mean test \eqref{eq:am_test} is zero (as conventionally assumed). In all other settings for the $n_{i'j}$'s, it is very unlikely that ${{}\mathcal y_{i1}^{\rm bin}}^\star=-{{}\mathcal y_{i2}^{\rm bin}}^\star$ and the optimal decision threshold \emph{must} also be estimated. As a matter of fact, following up on Remark~\ref{rem:on_centering}, the aforementioned optimal value ${{}\mathcal y^{\rm bin}}^\star$ for $\mathcal y^{\rm bin}$ is not unique and could be shifted by any constant vector. This extra degree of freedom will be of much relevance in the application Section~\ref{application}, as commented in the following remark.

\begin{remark}[Setting the decision threshold to zero]
\label{rem:zero_threshold}
    As per Remark~\ref{rem:on_centering}, the addition of a constant term to $\mathcal{y}^{\rm bin}$ does not affect the ultimate performance of MTL-LSSVM. Yet, it affects the value of the limiting means $\mathcal{m}_{ij}$ of $g_i^{\rm bin}(\bf x)$, so in particular the value of the limiting optimal threshold $\frac12(\mathcal{m}_{i1}+\mathcal{m}_{i2})$. Specifically, one may shift all entries of $\mathcal{y}^{\rm bin}$ in such a way that $\frac12(\mathcal{m}_{i1}+\mathcal{m}_{i2})=0$ and thus recenter the decision threshold to zero. For $\bar{\mathcal{y}}\in\mathbb R$ this constant shift, this boils down to solving in the variable $\bar{\mathcal{y}}$ the equation $$0=\frac12(\mathcal{m}_{i1}+\mathcal{m}_{i2})=\frac 12(\mathcal{y}^{\rm bin}+\bar{\mathcal{y}} e_{i}^{[k]}\otimes\mathbb{1}_2)^\trans\mathcal{D}_{\bm\delta^{[2k]}}^{\frac 12}\left(I_{2k}-\mathcal{Z}_e \Gamma \right)(e_{2(i-1)+1}^{[2k]}+e_{2i}^{[2k]})$$ where $\mathcal{Z}_e=I_{2k}-\sum_{i'=1}^k E_{i'i'}^{[k]}\otimes\mathcal{c}_{i'}$ and $\mathcal{c}_{i'}=\mathbb{1}_2\left[\begin{smallmatrix}\frac{n_{i'1}}{n_{i'}}&\frac{n_{i'2}}{n_{i'}} \end{smallmatrix}\right]$. Similarly, one may instead impose that $\mathcal{m}_{i1}=0$: this will appear to be fundamental to \emph{align classifiers} in the multi-class ``one-versus-all'' extension of the present binary classification scheme (see details in Section~\ref{sec:one-versus-all}). 
\end{remark}

\begin{remark}[Tuning the hyperparameters]
\label{rem:hyperparam}
The previous section provided a high-level interpretation for the impact of the vector parameter $\gamma\in\mathbb R^k$ and the scalar parameter $\lambda\in\mathbb R$, the effect of which is to respectively regularize LSSVM learning and to set the throttle between individual versus collective learning. These hyperparameters intervene deeply inside our theoretical formulas (so far in Theorem~\ref{th:simple_theorem} but later in Theorem~\ref{th:main}) and are not amenable to simple optimization. Yet, as will be confirmed by experiments (see in particular Figure~\ref{fig:Mnist_binary}), the proposed optimization of the input scores $\mathcal y^{\rm bin}$ partly compensates for suboptimal choices in $\gamma,\lambda$. As such, an ``informed guess'', based on our previous discussion of the effects of these parameters, is in general sufficient for highly performing MTL-LSSVM. A further gradient descent operation (or local grid search) on the theoretical performance approximation, initialized at the informed guess values, can further improve the overall learning performance.
\end{remark}

\subsection{Practical implementation of improved MTL-LSSVM}
As already pointed out, a fundamental aspect of Theorem~\ref{th:simple_theorem} lies in the performances of the \emph{large dimensional} ($n,p\gg 1$) classification problem at hand boiling down to $2k$-dimensional statistics. More importantly from a practical perspective, these $2k$-dimensional ``sufficient statistics'' are easily amenable to fast and efficient estimation: it indeed only requires a few training data samples to estimate all quantities involved in the theorem (which, as a corollary, lets one envision the possibility of efficient transfer learning methods based on very scarce data samples).

\begin{remark}[On the estimation of $\mathcal{m}_{ij}$ and $\sigma_i$]
\label{rem:estimate}
All quantities defined in Theorem~\ref{th:simple_theorem} are a priori known, apart from the quantities $\mathcal M\equiv \sum\limits_{i,i'}\Delta\mu_i^\trans\Delta\mu_{i'} \left(E_{ii'}^{[k]} \otimes \mathbb{c}_i\mathbb{c}_{i'}^\trans\right)$ and most specifically the inner products $\Delta\mu_i^\trans\Delta\mu_{i'}$. For these, define, for $j=1,2$, two sets $\mathcal S_{ij},\mathcal S'_{ij}\subset \{1,\ldots,n_{ij}\}$ and the corresponding indicator vectors $\mathbb{j}_{ij},\mathbb{j}'_{ij}\in\mathbb{R}^{n_i}$ with $[\mathbb{j}_{ij}]_a=\delta_{a\in \mathcal S_{ij}}$ and $[\mathbb{j}'_{ij}]_a=\delta_{a\in \mathcal S'_{ij}}$. We further impose that $\mathcal S'_{ij}\cap \mathcal S_{ij}=\emptyset$. Then, for $i\neq i'$, the following estimates hold:
\begin{align*}
    &\Delta\mu_i^\trans\Delta\mu_{i'} - \left(\frac{\mathbb{j}_{i1}}{|\mathcal S_{i1}|}-\frac{\mathbb{j}_{i2}}{|\mathcal S_{i2}|}\right)^\trans \mathring{X}_i^\trans \mathring{X}_{i'} \left(\frac{\mathbb{j}_{i'1}}{|\mathcal S_{i'1}|}-\frac{\mathbb{j}_{i'2}}{|\mathcal S_{i'2}|}\right) \nonumber \\
    &= O\left( (p \min_{l\in\{1,2\}}\{|\mathcal S_{il}|,|\mathcal S_{i'l}|\})^{-\frac12}\right)\\
    &\Delta\mu_i^\trans\Delta\mu_i - \left(\frac{\mathbb{j}_{i1}}{|\mathcal S_{i1}|}-\frac{\mathbb{j}_{i2}}{|\mathcal S_{i2}|}\right)^\trans  \mathring{X}_i^\trans \mathring{X}_i\left(\frac{\mathbb{j}'_{i1}}{|\mathcal S'_{i1}|}-\frac{\mathbb{j}'_{i2}}{|\mathcal S'_{i2}|}\right) \nonumber \\
    &= O\left( (p \min_{l\in\{1,2\}}\{|\mathcal S_{il}|,|\mathcal S'_{il}|\})^{-\frac12} \right).
\end{align*}
Observe in particular that a single sample (two when $i=i'$) per task and per class ($|\mathcal S_{il}|=1$) is sufficient to obtain a consistent estimate for all quantities, so long that $p$ is large. In a transfer learning setting where some tasks may contain few labeled data, it is thus still possible to optimize the MTL algorithm. 
Of course, when more data are available, under our assumption that $p\sim n$, taking all samples in the averaging, the convergence speed is of order $O(1/\sqrt{np})=O(1/n)$, which is a quadratic increase in the speed of the usual central-limit theorem.
\end{remark}

Estimating $\mathcal{m}_{ij}$ and $\sigma_i$ not only allows one to anticipate theoretical performances but also enables the actual estimation of the decision threshold $\frac12(\mathcal{m}_{i1}+\mathcal{m}_{i2})$ of the test \eqref{eq:am_test} and, as shown previously, opens the possibility to largely optimize MTL-LSSVM through an (asymptotically) optimal choice of the training scores $\mathcal y^{\rm bin}$. 

\medskip

The series of theoretical and practical results of this section may be synthetized under the form of Algorithm~\ref{alg:binary_algorithm}.
\begin{algorithm}
 \caption{Proposed binary Multi Task Learning algorithm.}
 \label{alg:binary_algorithm}
 \begin{algorithmic}
     \STATE {{\bfseries Input:} Training samples $X=[X_1,\ldots,X_k]$ with $X_{i'}=[X_{i'}^{(1)},X_{i'}^{(2)}]$ and test data ${\bf x}$.}
     \STATE {{\bfseries Output:} Estimated class $\hat j\in\{1,2\}$ of $\bf x$ for target Task~$i$.}
     \STATE {\bfseries Center and normalize}
     data per task: for all $i'\in\{1,\ldots,k\}$,
     \begin{itemize}
     \item $\mathring{X}_{i'} \leftarrow X_{i'} \left( I_{n_{i'}} - \frac1{n_{i'}}\mathbb{1}_{n_{i'}}\mathbb{1}_{n_{i'}}^\trans \right)$
     \item $\mathring{X}_{i'} \leftarrow \mathring{X}_{i'}/\frac1{n_{i'}p}\tr (\mathring{X}_{i'}\mathring{X}_{i'}^\trans)$
     \end{itemize}
        \STATE {{\bfseries Estimate:} Matrix $\mathcal M$ from Remark~\ref{rem:estimate} and $\bm\delta^{[k]}$ by solving \eqref{eq:tildeDelta}}.
         \STATE {\bfseries Create} scores $\mathcal{y}^{\rm bin}={{}\mathcal{y}^{\rm bin}}^{\star}$ according to \eqref{eq:ty_opt_simple}.
         \STATE {\bfseries Compute} the threshold $\zeta_i$ from \eqref{eq:am_test}, with $\mathcal{m}_{ij}$ defined in Theorem~\ref{th:simple_theorem} for $\mathcal{y}^{\rm bin}={{}\mathcal{y}^{\rm bin}}^{\star}$.
         \STATE {\bfseries (Optional) Estimate} the theoretical classification error $\epsilon_{i1}=\epsilon_{i1}(\lambda,\gamma)$ from \eqref{eq:classification} and minimize over $(\lambda,\gamma)$.\footnotemark
         \STATE {\bfseries Compute} classification score $g_i({\bf x})$ according to \eqref{eq:score_class}.
     \STATE {{\bfseries Output: }} $\hat j$ such that $g_i({\bf x})\underset{\hat j=2}{\overset{\hat j=1}{\gtrless}} \zeta_i$.
 \end{algorithmic}
 \end{algorithm}
 \footnotetext{As per Remark~\ref{rem:hyperparam}, this operation involves reevaluating $\bm\delta^{[k]}$ and thus $\mathcal y^\star$, and thus $\mathcal{m}$ for each $(\lambda,\gamma)$. It can be performed either on a static grid or by gradient descent until a local minimum is reached.}
 
\subsection{Empirical evidence}

This section shortly illustrates the ideas and intuitions developed so far (such as the relevance of an optimal choice of the data labels and decision threshold) through the performances of Algorithm~\ref{alg:binary_algorithm} on a transfer learning benchmark application. Sections~\ref{application}--\ref{experiments} will cover a much larger spectrum of applications and experiments, under the most general data setting discussed in the subsequent sections.

For optimal comparison, we consider here the standard Office+Caltech256 real image classification benchmark \citep{saenko2010adapting,griffin2007caltech}, consisting of four tasks and $m=10$ categories shared by all tasks. The dataset $X$ consists here of the VGG features of size $p=4096$ extracted from these images. We place ourselves under a $k=2$ transfer learning setting where Task~$1$ is the source task and Task~$2$ is the target task (the performance of which we aim to optimize), taken from two of the four tasks of the dataset (Caltech, Webcam, Amason, dslr). For testing, the samples of the target task are randomly selected from the test dataset of Office+Caltech256 and the classification accuracy is averaged over $20$ trials. Table~\ref{tab:compare} reports the accuracy for all possible pairs ($4\times 3=12$ of them) of source and transfer tasks, obtained by Algorithm~\ref{alg:binary_algorithm} (Ours) versus the non-optimized LSSVM of \citep{xu2013} (LSSVM) and versus other state-of-the-art transfer learning algorithms: the max margin domain transform of \citep{hoffman2013efficient} (MMDT) which seeks a linear transform to match the source data to the target data and then applies an SVM on the resulting target domain; the cross-domain landmark selection (CDLS) of \citep{hubert2016learning}, which learns a feature subspace which matches the cross-domain data distribution and eliminates the domain differences; and the invariant latent space (ILS) of \citep{herath2017learning}, which, similar to MMDT, learns an invariant latent space in which the discrepancy between source and target is minimized.  As already pointed out in introduction, since the article aims to propose an improved classification algorithm independent of the feature representation, it is fair to compare it to methods which use the same data features. The algorithms compared in the table all systematically use VGG features. It would be unfair to compare these against "end to end" MTL learning methods including a (explicit or implicit) step of feature learning like recent deep neural networks methods\citep{zhuang2020comprehensive,krishna2019deep}.

Since $m=10$ here, Algorithm~\ref{alg:binary_algorithm} cannot rigorously be used as it stands. We apply instead a \emph{naive} ``one-versus-all'' extension consisting in running in parallel $m=10$ times Algorithm~\ref{alg:binary_algorithm} by considering, for each class $\mathcal C_j$ of Task~$i$, $1\leq j\leq m$, a binary setting where the fictitious ``Class~$\tilde{\mathcal C}_1$'' coincides with $\mathcal C_j$ and the second fictitious ``Class~$\tilde{\mathcal C}_2$'' is the union of all $\mathcal C_{j'}$ for $j'\neq j$. Following up on Remark~\ref{rem:zero_threshold}, each classifier $\ell\in\{1,\ldots,m\}$ can be set in such a way that $\mathbb E[g^{\rm bin}_i({\bf x};\ell)]=0$ when ${\bf x}\in\mathcal C_\ell$. For a new datum $\bf x$, of all $m$ classifiers $g^{\rm bin}_i({\bf x};1),\ldots,g^{\rm bin}_i({\bf x};m)$, the one reaching the greatest score is the selected allocation class for $\bf x$.

Table~\ref{tab:compare} demonstrates that our proposed improved MTL-LSSVM, despite its simplicity and unlike the competing methods used for comparison, has stable performances and is extremely competitive. It either outperforms all other methods or is second-to-best. But, most importantly, the method comes along with performance predictions and guarantees, which none of the competing works are able to provide.\footnote{In the present context of the \emph{naive} ``one-versus-all'', this claim should be taken with care: the performance can indeed be predicted provided the binary class model $\mathcal N(\mu_{i1},I_p)$ versus $\mathcal N(\mu_{i2},I_p)$ correctly matches the actual data distribution; this is likely not the case here as the collected fictitious ``$\tilde{\mathcal C}_2$'' is rather a mixture of Gaussian rather than a unique Gaussian. In Section~\ref{sec:one-versus-all}, a more elaborate, and theoretically better supported, version of the one-versus-all approach will be discussed.}

\begin{table*}[h!t]
\caption{Classification accuracy over Office+Caltech256 database. c(Caltech), w(Webcam), a(Amazon), d(dslr), for different ``Source to target'' task pairs ($S\to T$) based on VGG features. Best score in boldface, second-to-best in italic.}
\label{tab:compare}

\centering
\hspace*{-1cm}\begin{tabular}{p{0.074\linewidth}p{0.054\linewidth}p{0.054\linewidth}p{0.054\linewidth}p{0.054\linewidth}p{0.054\linewidth}p{0.054\linewidth}p{0.054\linewidth}p{0.054\linewidth}p{0.054\linewidth}p{0.054\linewidth}p{0.054\linewidth}p{0.054\linewidth}|p{0.054\linewidth}}
\hline
S/T & c$\,\to\,$w & w$\,\to\,$c & c$\,\to\,$a & a$\,\to\,$c& w$\,\to\,$a & a$\,\to\,$d & d$\,\to\,$a & w$\,\to\,$d&c$\,\to\,$d & d$\,\to\,$c & a$\,\to\,$w & d$\,\to\,$w&Mean score\\
\hline
LSSVM &96.69 & \bf 89.90 & 92.90 &\it 90.00 & 93.80 & 78.70 & 93.50 & 95.00 & 85.00 & \bf 90.20 & 94.70 &\bf ~100 & 91.70 \\

MMDT & 93.90 & 87.05 & 90.83 & 84.40 & \it 94.17 & 86.25 & \bf 94.58 & 97.50 & 86.25 & 87.23 & 92.05 &  \it 97.35 & 90.96 \\

ILS & 77.89 & 73.55 & 86.85 & 76.22 & 86.22 & 71.34 & 74.53 & 82.80 & 68.15 & 63.49 & 78.98 & 92.88 & 77.74\\

CDLS & \it 97.60 &  88.30 & \it 93.54 & 88.30 &  93.54 & \it 92.50 & 93.54 & 93.75 & \bf 93.75 & 88.30 & \it 97.35 & 96.70 & {\it 93.10} \\
\hline
Ours & {\bf 98.68} & {\bf 89.90} & {\bf 94.40} & {\bf 90.60} & {\bf 94.40} & {\bf 93.80} & {\it 94.20} & {\bf ~100} & {\it 92.50} & {\it 89.90} & {\bf 98.70} & {\it 99.30} & {\bf 94.70} \\
\hline
\end{tabular}
\end{table*}

\medskip

These preliminary results are already very conclusive and reveal the strength of our proposed methodology. Yet, the assumptions in place so far are restricted to random concentrated data with identity covariance and to a binary classification setting (which, as already observed, needs be adapted to account for more than two classes per task). The next sections elaborate on the more generic setting of $m\geq 2$ classes per task with more realistic data models. The theoretical results no longer reduce to compact expressions as in the previous sections but are easily understood having already delineated the main take-away messages and ideas.

\section{The General Framework}

The results from the previous section are extended here to the more realistic setting where the data arise from a mixture of $m\geq 2$ concentrated random vectors with generic covariance $\Sigma_{ij}$. New insights, and most importantly, more general and application-driven algorithms will be introduced. In addition, the results are presented here with a sketched development of their main technical arguments, the full proofs being deferred to the appendix.

\label{sec:theoretical_analysis}

\subsection{Main ideas}

Taking for the moment for granted the Gaussian limit for $g_i({\bf x})\in\mathbb{R}^m$ as $p,n\to\infty$ (for $1\leq i\leq k$), the main technical task to obtain our main result (Theorem~\ref{th:main}, which generalizes the already introduced Theorem~\ref{th:simple_theorem}) is to evaluate the large dimensional behavior $\mathcal m_{ij}$ and $C_{ij}$ of the statistical mean $\mathbb{E}[g_i({\bf x})]=\mathcal{m}_{ij}+o(1)$ and covariance matrix ${\rm Cov}[g_i({\bf x})]=C_{ij}+o(1)$ of the classification score $g_i({\bf x})$ in \eqref{eq:score_class} for data vectors ${\bf x}$ in class $\mathcal C_j$ (i.e., such that $\mathbb{E}[{\bf x}]=\mu_{ij}$ and $\mathrm{Cov}[{\bf x}]=\Sigma_{ij}$), respectively given by:
\begin{align}
    \label{eq:statistics_mean}
    &\mathcal{m}_{ij}=\mathbb{E}\left[\frac{1}{kp}(Y-Pb)^\trans Z^\trans A^{\frac 12}\tilde{Q}A^{\frac 12}\left({e_{i}^{[k]}}\otimes\mu_{ij}\right) +b_i\right]+o(1)\\
    \label{eq:statistics_variance}
    &C_{ij}=\mathbb{E}\left[\frac{1}{(kp)^2}(Y-Pb)^\trans Z^\trans A^{\frac 12} \tilde{Q}A^{\frac 12}S_{ij}A^{\frac 12}\tilde{Q}A^{\frac 12}Z(Y-Pb)\right]+o(1)
\end{align}
with $S_{ij} =e_i^{[k]}{e_i^{[k]}}^\trans\otimes \Sigma_{ij}$ and $\tilde{Q}=\left(\frac{A^{\frac 12}ZZ^{\trans}A^{\frac 12}}{kp}+I_{kp}\right)^{-1}$.

Our technical approach to evaluate these terms, in the large dimensional regime of Assumption~\ref{ass:growth_rate} and for data distributed as per Assumption~\ref{ass:distribution}, consists in determining \emph{deterministic equivalents}, a classical object in random matrix theory \citep[Chapter~6]{COUbook}, for the matrices $\tilde Q$, $\tilde QA^{\frac12}Z$, $Z^\trans A^{\frac 12} \tilde{Q}A^{\frac 12}S_{ij}A^{\frac 12}\tilde{Q}A^{\frac 12}Z$ which are at the core of the formulation of $\mathcal{m}_{ij}$ and $C_{ij}$. Specifically, a deterministic equivalent, say $\bar F\in\mathbb{R}^{n\times p}$, of a given random matrix $F\in\mathbb{R}^{n\times p}$ is a \emph{deterministic matrix} such that, for any deterministic linear functional $f:\mathbb{R}^{n\times p}\to\mathbb{R}$ of bounded norm, $f(F-\bar F)\to 0$ almost surely -- in particular, for $u,v$ of unit norm, $u^\trans (F-\bar{F})v\asto 0$ and, for $A\in\mathbb{R}^{p\times n}$ deterministic of bounded operator norm, $\frac{1}{n}\tr A(F-\bar{F})\asto 0$. We will denote for short $F\leftrightarrow \bar{F}$ to indicate that $\bar F$ is a deterministic equivalent for $F$. Deterministic equivalents are thus particularly suitable to handle bilinear forms involving the random matrix $F$, so in particular the statistics \eqref{eq:statistics_mean} and \eqref{eq:statistics_variance} of $g_i({\bf x})$, seen as bilinear forms involving the random matrices $\tilde QA^{\frac12}Z$ and ${Z^\trans A^{\frac 12} \tilde{Q}A^{\frac 12}S_{ij}A^{\frac 12}\tilde{Q}A^{\frac 12}Z}$. 

Lemma~\ref{lem:eq}, deferred to Section~\ref{sec:lemma} of the appendix (as the result in itself does not bring any deep insight worth discussing here), provides the necessary deterministic equivalents for these matrices. It is interesting to point out though that, from a technical standpoint, the block structure followed by the core data matrix $Z$ introduced in Proposition~\ref{prop:solution_optimization} makes the large dimensional random matrix analysis more challenging and the result less straightforward than in similar previous works \citep{mai2019high,liao2019large}. Even in the simplest setting where the $x_{il}^{(j)}$ would be vectors of i.i.d.\@ $\mathcal N(0,1)$ entries, the matrix $Z$ is \emph{not} a matrix of i.i.d.\@ entries (due to precisely located blocks of zeros) and the singular values of $Z$ do not asymptotically follow the popular Mar\u{c}enko-Pastur distribution from \citep{marvcenko1967distribution}, as would be the case in works dealing with single-task learning (single-task LSSVM \citep{liao2019large}, semi-supervised learning \citep{mai2019high}, neural networks \citep{louart2018random}, etc.).

The main information to be extracted from Lemma~\ref{lem:eq} (again, see its complete form in the appendix) is the central role played by the deterministic matrices
\begin{align*}
    M&=\left(e_1^{[k]}\otimes[\mu_{11},\ldots,\mu_{1m}],\ldots,e_k^{[k]}\otimes[\mu_{k1},\ldots,\mu_{km}]\right) \\
 {\mathbb{C}}_{ij}&=A^{\frac 12}\left(e_i^{[k]}{e_i^{[k]}}^\trans\otimes(\Sigma_{ij}+\mu_{ij}\mu_{ij}^\trans) \right)A^{\frac 12}
\end{align*}
which generalize the matrices $\mathcal M$ and $\mathcal A$ discussed at length in Section~\ref{sec:theoretical_simple} when $\Sigma_{ij}=I_p$. While gaining in genericity, unlike $\mathcal M$, the matrices $M$ and $\mathbb C_{ij}$ preserve their large dimensions: this is the main price paid by the generalization to $\Sigma_{ij}\neq I_p$. Yet, the central small dimensional matrix $\Gamma$ defined in \eqref{eq:Gamma_simple} remains small and now becomes
 \begin{align*}
     \Gamma&=\left(I_{mk}+\mathbb{M}^\trans\bar{\tilde{Q}}_0\mathbb{M}\right)^{-1} \\
     \bar{\tilde{Q}}_0&=\left[\sum\limits_{i=1}^k\sum\limits_{j=1}^m\left(\mathcal{D}_{\gamma}+\lambda\mathbb{1}_{k}\mathbb{1}_{k}\right)^{\frac 12}e_{i}^{[k]}{{}e_{i}^{[k]}}^\trans\left(\mathcal{D}_{\gamma}+\lambda\mathbb{1}_{k}\mathbb{1}_{k}\right)^{\frac 12}\otimes {\bm\delta}_{ij}^{[mk]}\Sigma_{ij}+I_{kp}\right]^{-1} \\
     \mathbb{M}&=A^{\frac 12}M\mathcal{D}_{{\bm\delta}^{[mk]}}^{\frac 12}
 \end{align*}
 and the $mk$ scalars ${\bm\delta}_{ij}^{[mk]}$ are the unique positive solutions of the fixed point equations
\begin{align*}
    \bm\delta_{ij}^{[mk]}&=\frac{c_{ij}}{c_0\left(1+\frac 1{kp} \tr (\mathbb{C}_{ij}\bar{\tilde{Q}})\right)} \\
    \bar{\tilde{Q}}&=\left(\sum_{i=1}^k\sum\limits_{j=1}^m{\bm\delta}_{ij}^{[mk]}\mathbb{C}_{ij}+I_{kp}\right)^{-1}.
\end{align*}
Here, $\bar{\tilde Q}$ is a deterministic equivalent of $\tilde{Q}$.
Finally, the matrix $\mathcal K$ appearing in the variance term of Theorem~\ref{th:simple_theorem} now becomes
\begin{align*}
    \mathcal{K}&=c_0\bar{T}\left(\mathcal{D}_{c}-c_0\mathcal{T}\right)^{-1} \\
     \bar{T}_{ij,i'j'}&=\frac{{\bm\delta}_{ij}^{[mk]}{\bm\delta}_{i'j'}^{[mk]}}{kp}\tr\left(\mathbb{C}_{i'j'}\bar{\tilde{Q}}A^{\frac 12}S_{ij}A^{\frac 12}\bar{\tilde{Q}}\right) \\ \mathcal{T}_{ij,i'j'}&=\frac{{\bm\delta}_{ij}^{[mk]}{\bm\delta}_{i'j'}^{[mk]}}{kp}\tr(\mathbb{C}_{ij}\bar{\tilde{Q}}\mathbb{C}_{i'j'}\bar{\tilde{Q}})
\end{align*}
where $\bar{T}_{ij,i'j'}$ is the element at row $m(i-1)+j$ and column $m(i'-1)+j'$ of $\bar{T}$ (and similarly for $\mathcal T$) and $\kappa_{ij,.}\in\mathbb R^{mk}$ represents the $m(i-1)+j$ row of matrix $\kappa\in\mathbb{R}^{mk\times mk}$.

With these technical elements at hand, we are in position to enunciate the main result of the article.

\subsection{Classification score asymptotics}
\label{sec:classification}
\begin{sloppypar}
\begin{theorem}
\label{th:main}
Under Assumptions~\ref{ass:distribution} and \ref{ass:growth_rate}, for a test data ${\bf x}$ with $\mathbb E[{\bf x}]=\mu_{ij}$ and ${\rm Cov}[{\bf x}]=\Sigma_{ij}$, as $p,n\to\infty$,
\begin{equation*}
    g_{i}({\bf x}) - G_{ij} \rightarrow  0,\quad G_{ij} \sim \mathcal{N}(\mathcal{m}_{ij},C_{ij})
\end{equation*}
in law where, letting $\mathcal{m}=[\mathcal{m}_{11},\ldots,\mathcal{m}_{km}]^\trans \in \mathbb{R}^{km\times m}$ and the normalized forms $\bm{\mathcal{Y}}\equiv \mathcal{D}_{{\bm\delta}^{[mk]}}^{\frac12} \mathcal{Y}$, $\mathring{\bm{\mathcal{Y}}}= \mathcal{D}_{{\bm\delta}^{[mk]}}^{\frac12} \mathring{\mathcal{Y}}$, ${\bm{\mathcal{m}}} = \mathcal{D}_{{\bm\delta}^{[mk]}}^{\frac12} \mathcal{m}$,
\begin{align*}
    {\bm{\mathcal{m}}}&=\bm{\mathcal{ Y}}-\Gamma\mathring{\bm{\mathcal{Y}}} \in \mathbb{R}^m \\
     C_{ij}& =\mathring{\bm{\mathcal{Y}}}^\trans\Gamma\mathcal{V}_{ij}\Gamma \mathring{\bm{\mathcal{ Y}}} \in \mathbb R^{m\times m}
\end{align*}
with 
\begin{align*}
    \mathcal{V}_{ij}&=\mathcal{D}_{\kappa_{ij,.}}+\mathbb{M}^{\trans}\bar{\tilde{Q}}_0\mathbb{V}_{ij}\bar{\tilde{Q}}_0\mathbb{M} \\
    \mathbb{V}_{ij} &=A^{\frac 12}S_{ij}A^{\frac 12}+\sum\limits_{i'=1}^k\sum\limits_{j'=1}^m {\bm\delta}_{i'j'}^{[mk]}\kappa_{ij,i'j'}A^{\frac 12}S_{i'j'}A^{\frac 12} \\
    \kappa_{ij,i'j'} &= \frac{\mathcal{K}_{ij,i'j'}}{{\bm\delta}_{ij}^{[mk]}}.
\end{align*}
\end{theorem}
\begin{proof}
See Section~\ref{app:th} of the appendix.
\end{proof}

In the particular case of $\Sigma_{ij}=I_p$ and $m=2$, Theorem \ref{th:main} reduces to Theorem~\ref{th:simple_theorem} (see details in Section~\ref{app:th} of the appendix) by remarking that $\mathbb{M}^\trans\bar{\tilde{Q}}_0\mathbb{M}=\left(\mathcal{A}\otimes\mathbb{1}_k\mathbb{1}_k^\trans\right)\odot \mathcal M$ and $\mathbb{M}^\trans\bar{\tilde{Q}}_0V_{ij}\bar{\tilde{Q}}_0\mathbb{M}=\left(\mathcal{A}\mathcal{D}_{\kappa_i+e_{i}^{[k]}}\mathcal{A}\otimes\mathbb{1}_k\mathbb{1}_k^\trans\right)\odot \mathcal M$ which, as already pointed out, have the advantage to be defined as the product of exclusively small dimensional matrices. Still, although more technical, Theorem~\ref{th:main} follows the same structure as Theorem~\ref{th:simple_theorem}. 

\end{sloppypar}

Before concretely applying the result of Theorem~\ref{th:main} to practical learning problems (multi-task, transfer learning, hypothesis testing), a few comments and immediate corollaries are in order.

\begin{remark}[Optimization of $\mathcal y^{\rm bin}$ for $m=2$ and generic $\Sigma_{ij}$]
\label{rem:optimization_general}
As suggested in Section~\ref{sec:theoretical_simple}, for binary classification ($m=2$), it is particularly convenient to recast the score vectors $y_{il}^{(j)}\in\mathbb R^m$ into scalar scores $y_{il}^{{\rm bin}(j)}\in\mathbb R$ (this being irrespective of the nature of $\Sigma_{ij}$). Inspired by Section~\ref{sec:theoretical_simple}, one can trivially extend Theorem~\ref{th:main} to this binary setting. In this case, $g_i({\bf x})\in\mathbb R^m$ is now turned into a scalar $g_i^{\rm bin}({\bf x})\in\mathbb R$ well approximated by $\mathcal N(\mathcal{m}_{ij},C_{ij})$ where now $\mathcal{m}_{ij}$ and $C_{ij}$ are scalar, obtained by simply replacing $y_{il}^{(j)}\in\mathbb R^m$ by $y_{il}^{{\rm bin}(j)}\in\mathbb R$ in their respective expressions.

With these notations, setting the decision threshold of $g_i^{\rm bin}({\bf x})$ to $\zeta\in\mathbb R$ and assuming equal prior probability for the genuine class of $\bf x$, the classification error rate for a target task $i$ is
\begin{equation*}
    E=\frac{1}{2}\mathcal{Q}\left(\frac{\zeta-\mathcal{m}_{i1}}{\sqrt{C_{i1}}}\right)+\frac{1}{2}\mathcal{Q}\left(\frac{\zeta-\mathcal{m}_{i2}}{\sqrt{C_{i2}}}\right).
\end{equation*}
As in Section~\ref{sec:theoretical_simple}, if $\Sigma_{i1}=\Sigma_{i2}$, then $C_{i1}=C_{i2}\equiv C_i$ and the decision threshold $\zeta$ minimizing the classification error is
\begin{equation*}
    \zeta^\star=\frac{\mathcal{m}_{i1}+\mathcal{m}_{i2}}{2}
\end{equation*}
from which the optimal vector $\mathcal y^{\rm bin}$ for Task~$i$ is computed as
\begin{align}
    {{}\mathcal y^{\rm bin}}^\star &=\argmax_{\mathcal y^{\rm bin}\in\mathbb{R}^{2k}} \frac{(\mathcal{m}_{i1}-\mathcal{m}_{i2})^2}{C_i} \nonumber \\
 \label{eq:ty_opt}
      &=\mathcal D_{{\bm\delta}^{[2k]}}^{-\frac 12}\Gamma^{-1} \mathcal{V}_{ij}^{-1}\mathbb{M}^\trans\bar{\tilde{Q}}_0\mathbb{M}\mathcal D_{{\bm\delta}^{[2k]}}^{-\frac 12}(e_{i1}^{[2k]}-e_{i2}^{[2k]}).
 \end{align}
It is important to recall here that, while ${{}\mathcal y^{\rm bin}}^\star$ expresses here solely as a function of terms involving the index $i$, all other statistics of the tasks $i'\neq i$ are in fact ``embedded'' inside these terms and are thus, of course, accounted for in the optimization.
 
When $C_{i1}\neq C_{i2}$ (which is the case in general), one may minimize $E$ by resorting to numerical optimization techniques. We suggest to use a gradient descent method initialized to the expression obtained in \eqref{eq:ty_opt}. So long that $\Sigma_{i1}$ and $\Sigma_{i2}$ are not drastically different, this approach shows good performances (see our results in Section~\ref{experiments}).

This said, the specific setting of binary classification may in practice be one of hypotheses testing. Under this scenario, one may not demand that the average error $E$ be minimized (i.e., that data from either class is equally well identified) but rather that the probability of misclassification of a given class (say, a Type-I error) be bounded to some $\eta>0$ while minimizing the error rate for the other class (Type-II error). In this context, if, say, one fixes $\mathcal{Q}(\frac{\zeta-\mathcal{m}_{i1}}{\sqrt{C_{i1}}})\equiv\eta$, then the classification error for the second class $\mathcal{Q}(\frac{\zeta-\mathcal{m}_{i2}}{\sqrt{C_{i2}}})$ is minimized by now choosing
\begin{equation*}
     {\mathcal y^{\rm bin}}^\star =\argmax_{\mathcal y^{\rm bin}\in\mathbb{R}^{2k}} \frac{(\sqrt{C_{i1}}\mathcal{Q}^{-1}(\eta)+\mathcal{m}_{i1}-\mathcal{m}_{i2})^2}{C_{i2}}
\end{equation*}
where $\mathcal{Q}^{-1}$ is the inverse function of the $\mathcal{Q}$ function.
This again can be solved using numerical convex optimization initialized at the value of \eqref{eq:ty_opt}. More details on this hypotheses testing setting, along with concrete experiments, are developed in Section~\ref{experiments}.
\end{remark}

\begin{remark}[Estimation of $\mathcal{m}_{ij}$ and $C_{ij}$]
    \label{rem:estimation_general}
    In order to anticipate the performances and best set the decision thresholds for classification, one needs to access all quantities arising in Theorem~\ref{th:main}. Yet, as opposed to Remark~\ref{rem:estimate}, where the low dimensional quantities of interest (mainly the inner products between statistical means) are easily estimated, the low dimensional quantities involved in Theorem~\ref{th:main} are less convenient to estimate, this being due to the presence of the a priori unknown covariance matrices $\Sigma_{ij}$. We propose here two strategies:
    \begin{enumerate}
        \item either make the assumption that $\Sigma_{ij}\simeq \alpha_{ij}I_p$ with $\alpha_{ij}$ estimated by $\frac1{pn_{ij}}\tr \mathring{X}_{i}^{(j)}\mathring{X}_{i}^{(j)\trans}$; then normalize the data as $\mathring{X}_{i}^{(j)}\leftarrow \mathring{X}_{i}^{(j)}/\frac1{pn_{ij}}\tr \mathring{X}_{i}^{(j)}\mathring{X}_{i}^{(j)\trans}$ in the spirit of Algorithm~\ref{alg:binary_algorithm}. This places the experimenter under an isotropic data setting for all classes and tasks, from which the considerations of Section~\ref{sec:theoretical_simple} (possibly generalized to $m>2$) apply.
        \item either estimate $\Sigma_{ij}$ by means of the sample covariance matrix $\frac1{n_{ij}}\mathring{X}_{i}^{(j)}\mathring{X}_{i}^{(j)\trans}$; this procedure is known to be biased, particularly so if $n_{ij}$ is not large compared to $p$; yet, as demonstrated in our experiments in Section~\ref{experiments}, this only marginally (if not at all) alters the performance of our proposed algorithms.\footnote{It must be pointed out that similar random matrix-based studies propose consistent estimates for low dimensional quantities such as those met in Theorem~\ref{th:main}; however, these would assume cumbersome forms which, we believe, go against our present request for simple, intuitive but well parameterized algorithms for multi-task and transfer learning.}
    \end{enumerate}
    The choice of strategy mainly depends on the belief from the experimenter that the genuine covariance matrices are ``well-conditioned'' (i.e., their eigenvalues do not spread much) in which case Option~1 would be favored or ``ill-conditioned'' (typically when the space spanned by the data is much lower than $p$) in which case Option~2 would be more appropriate.
\end{remark}

\begin{remark}[On universality]
As pointed out in the introduction, the input data $X$ follows a very generic \emph{concentrated random vector} assumption (Assumption~\ref{ass:distribution}). This choice provides both a technical, but most importantly, a fundamentally practical, advantage:
\begin{enumerate}
    \item from a technical standpoint, the concentration of measure phenomenon provides efficient and fast mathematical tools \citep{louart2018concentration,ledoux2001concentration} to analyze the random quantities appearing in the classification test scores $g_i({\bf x})$ of MTL-LSSVM (which, in essence, is a mere functional $\mathbb{R}^{p\times n}\to \mathbb{R}$ of the random input data $X$). More specifically, alternative random matrix tools based on Gaussian \citep{pastur2011eigenvalue} or independent entries assumptions \citep{SIL06} of $X$ would both be less general (at least for our machine learning purpose) and more computationally intense;
    \item on the practical side, as underlined in Section~\ref{sec:model_stat}, the concentrated random vector assumption better models realistic datasets by imposing very little structure on the data. Exactly, it only constrains all \emph{Lipschitz functionals} $\mathbb{R}^{p\times n}\to \mathbb R$ of $X$ (i.e., its typical observations) to satisfy a concentration inequality; while this may seem demanding, the family of concentrated random vectors in fact contains all Lipschitz generative models (for instance neural networks) fed by Gaussian inputs (such as GANs \citep{goodfellow2014generative}), as well as all further Lipschitz transformations of these vectors (for instance, features extracted by pretrained neural networks). As such, provided that the assumption of a common statistical mean and covariance per class and per task is reasonable, Theorem~\ref{th:main} ensures for instance that the performance of MTL-LSSVM applied to classes of the popular VGG or ResNet representations of GAN images is predictable. From this remark, it naturally comes that the proposed method is \emph{universal} in the sense of its being robust to a broad range of very realistic random data, and it is not daring to claim that it is equally valid on genuinely real data. This is confirmed by our numerical results of Section~\ref{experiments}.
\end{enumerate}
\end{remark}

With these elements in place, we are in position to apply our findings to a host of applications in statistical learning and to test the resulting algorithms against state-of-the-art alternatives.

\section{Applications}
\label{application}

This section provides various applications and optimizations of the proposed MTL-LSSVM framework based on the findings of the previous sections in the context of multi-class classification.

Having access to the large dimensional behavior of the classification test score in Theorem~\ref{th:main} (i.e., for $m\geq 2$ classes per task) is more fundamental than one may think. It indeed allows for a fine-tuning of the hyperparameters to be set to extend the usually considered binary MTL framework of \citep{evgeniou2004regularized,xu2013} to a multiclass-per-task MTL. 

\subsection{Multi-class classification preliminary}
The literature \citep{bishop2006pattern,rocha2013multiclass} describes broad groups of approaches for dealing with $m>2$ classes. We focus here on the most common methods, namely one-versus-all, one-versus-one, and one hot encoding. Being so far theoretically intractable (before the results of this article), these methods inherently suffer from sometimes severe limitations; these are partly tackled by adapting the theoretical results discussed in Section~\ref{sec:theoretical_analysis}:
\begin{enumerate}
    \item {\bf one-versus-all}: in this method, focusing on Task~$i$, $m$ individual binary classifiers $g_i^{\rm bin}(\ell)$ for $\ell=1,\ldots,m$ are trained, each of them separating Class~$\mathcal C_\ell$ from the other $m-1$ classes $\mathcal C_{\ell'}$, $\ell'\neq \ell$. Each test sample is then allocated to the class with the highest score among the $m$ classifiers. Although quite used in practice, the approach first suffers a severe data unbalancing effect when using binary ($\pm 1$) labels as the set of negative labels in each binary classification is on average $m-1$ times larger than the set of positive labels, and also suffers a centering-scale issue when ultimately comparing the outputs of the $m$ decision functions $g_i^{\rm bin}({\bf x};\ell)$, $\ell=1,\ldots,m$, whose average locations and ranges may greatly differ; these issues lead to undesirable effects, as reported in \citep[section 7.1.3]{bishop2006pattern}). 
    
    In Section~\ref{sec:one-versus-all}, these problems are simultaneously addressed: specifically, having access to the theoretical statistics of the classification scores allows us to appropriately center and scale the scores. Moreover, each binary classifier is optimized by appropriately choosing the class labels (no longer binary) so to minimize the resulting classification error (see Figure~\ref{fig:one-vs_all_explain} for an illustration of the improvement induced by the proposed approach).
    
    \item {\bf one-versus-one}: here, $\frac12m(m-1)$ binary classifiers are trained (one for each pair $j,j'$ of classes, solving a binary classification). For each test sample, each binary classifier decides on -- or ``votes'' for -- the more relevant class. The test sample is then attributed to the class having the majority of votes. Although the number of binary classifiers is greater than in the one-versus-all approach, the training process for each classifier is faster since the training database is much smaller for each binary classifier. Besides, the method is more robust to class imbalances (since only pairwise comparisons are made) but suffers from an undecidability limitation in the case of equal numbers of majority votes for two or more classes. 
    
    In Section~\ref{sec:one-versus-one}, each binary classifier will be optimized according to Algorithm~\ref{alg:binary_algorithm} by choosing appropriate labels and appropriate decision thresholds, thereby largely improving over the basal classifier performance.
    
    \item {\bf one-hot encoding approach}, also known as {\textrm one-per-class coding}: in this approach, each class is encoded using the $m$-dimensional canonical vector of the class (the code vector for class $j$ has a $1$ at position $j$ and $0$'s elsewhere). When testing an unknown sample $\bf x$, the index of the encoding output vector $g_i({\bf x})\in\mathbb R^m$ with maximum value is selected as the class of ${\bf x}$. 
    
    Exploiting the asymptotic performance of this approach from Theorem~\ref{th:main} allows us to derive a different label (or score) encoding for each class which theoretically minimizes the classification error. This is developed in detail in Section~\ref{sec:one-hot}.
    
\end{enumerate}

In the remainder of the section, each of the three classifiers is studied, optimized and their asymptotic performances are analyzed according to our previous results except for one-versus-one classification which involves difficult combinatorial aspects. While this does not provide a definite and general answer as to which of the three classifiers is best, it however provides an accurate assessment of their asymptotic performances; most importantly, these performances may be evaluated \emph{before} running the classifier, thereby helping practitioners to anticipate and optimize the method best suited for the application at hand, without resorting to any cross-validation procedure.

Let us finally insist that, for the two multi-class extensions based on binary classifiers (one-versus-one, one-versus-all), each binary classifier will be optimized independently following Remark~\ref{rem:optimization_general}: i.e., by recasting the score vectors $y_{il}^{(j)}\in\mathbb R^m$ into scalar scores $y_{il}^{{\rm bin}(j)}\in\mathbb R$. As such, from now on, for each binary classifier $\ell$, $g_i({\bf x};\ell)\in\mathbb R^m$ will be systematically turned into a scalar $g_i^{\rm bin}({\bf x},\ell)\in\mathbb R$ well approximated by $\mathcal N(\mathcal{m}_{ij},C_{ij})$ where now $\mathcal{m}_{ij}$ and $C_{ij}$ are scalar, obtained by simply replacing $y_{il}^{(j)}(\ell)\in\mathbb R^m$ by $y_{il}^{{\rm bin}(j)}(\ell)\in\mathbb R$ in their respective expressions.

\subsection{One-versus-all multi-class classification}
\label{sec:one-versus-all}

For every Task~$i$, the one-versus-all approach solves $m$ binary MTL-LSSVM algorithms with target class~$\mathcal C_\ell$, for each $\ell\in\{1,\ldots,m\}$, versus all other classes $\mathcal C_{\ell'}$, $\ell'\neq \ell$. Calling $g_i^{\rm bin}({\bf x};\ell)$ the output of the classifier $\ell$ for a new datum $\bf x$, the class allocation decision is traditionally based on the largest among all scores $g_i^{\rm bin}({\bf x};1),\ldots,g_i^{\rm bin}({\bf x};m)$. This approach generalizes the naive, yet simpler, method proposed in Algorithm~\ref{alg:binary_algorithm} which, despite its good performances (recall Table~\ref{tab:compare}), is fundamentally ``incorrect'' in its assuming that, for each $\ell$, all classes $\mathcal C_{\ell'}$ ($\ell'\neq \ell$) have the same statistics.

However, this presumes that the distribution of the scores $g_i^{\rm bin}({\bf x};1)$ when ${\bf x}\in\mathcal C_1$, $g_i^{\rm bin}({\bf x};2)$ when ${\bf x}\in\mathcal C_2$, etc., have more or less the same mean and variance. This is not the case in general, as depicted in the first column of Figure~\ref{fig:one-vs_all_explain}, where data from class $\mathcal C_1$ are more likely to be allocated to class $\mathcal C_3$ (compare the red curves). 

By providing an accurate estimate of the distribution of the scores $g_i^{\rm bin}({\bf x};\ell)$ for all $\ell$ and all genuine classes of $\bf x$, Theorem~\ref{th:main} allows us to predict the various positions of the Gaussian curves in Figure~\ref{fig:one-vs_all_explain}. In particular, exploiting the theorem along with Remark~\ref{rem:zero_threshold}, it is possible, for binary classifier $\ell$ to shift the corresponding input scores $\mathcal y^{\rm bin}(\ell)$ by a constant term $\bar{\mathcal{y}}(\ell)\in\mathbb R$ in such a way that $\mathbb E_{{\bf x}\in\mathcal C_\ell}[g_i^{\rm bin}({\bf x};\ell)]\simeq \mathcal m_{i\ell}(\ell)=0$ and ${\rm Var}_{{\bf x}\in\mathcal C_\ell}[g_i^{\rm bin}({\bf x};\ell)]\simeq C_{i\ell}(\ell)=1$. This operation prevents the centering and scale problems depicted in the first column of Figure~\ref{fig:one-vs_all_explain}, the result being visible in the second column of Figure~\ref{fig:one-vs_all_explain}.

\medskip

This first improvement step simplifies the algorithm which still boils down to determining the largest $g_i^{\rm bin}({\bf x};\ell)$, $\ell\in\{1,\ldots,m\}$, output but now limiting the risks induced by the inherent centering and scale issues previously discussed.

This being said, our theoretical analysis further allows to adapt the input scores $\mathcal y^{\rm bin}(\ell)$ in such a way to optimize the expected output. Ideally, assuming $\bf x$ genuinely belongs to class $\ell$, one may aim at increasing the distance between the output score $g_i^{\rm bin}({\bf x};\ell)$ and the other output scores $g_i^{\rm bin}({\bf x};\ell')$ for $\ell'\neq \ell$. This however demands to simultaneously adapt all input scores $\mathcal y^{\rm bin}(1),\ldots,\mathcal y^{\rm bin}(m)$. Instead, we resort to maximizing the distance between the output score $g_i^{\rm bin}({\bf x};\ell)$ for ${\bf x}\in\mathcal C_\ell$ and the scores $g_i^{\rm bin}({\bf x};\ell)$ for ${\bf x}\not\in\mathcal C_\ell$. By ``mechanically'' pushing away all wrong decisions, this ensures that, when ${\bf x}\in\mathcal C_\ell$, $g_i^{\rm bin}({\bf x};\ell)$ is greater than $g_i^{\rm bin}({\bf x};\ell')$ for $\ell'\neq \ell$. This is visually seen in the third column of Figure~\ref{fig:one-vs_all_explain}, where the distances between the rightmost Gaussians and the other two is increased when compared to the second column, and we retrieve the desired behavior.

Specifically, our proposed score optimization consists in solving, for each $i\in\{1,\ldots,k\}$ and each $\ell\in\{1,\ldots,m\}$ the optimization problems:
\begin{align}
    {{}\mathcal{y}^{\rm bin}}^{\star}(\ell)&=\argmin_{\mathcal{y}^{\rm bin}(\ell)\in\mathbb{R}^{km}} \max\limits_{j\neq \ell}\mathcal{Q}\left(\frac{\mathcal{m}_{i\ell}(\ell)-\mathcal{m}_{ij}(\ell)}{\sqrt{C_{tj}}}\right) \nonumber \\
    &=\argmin_{\mathcal{y}^{\rm bin}(\ell)\in\mathbb{R}^{km}} \max\limits_{j\neq \ell}\mathcal{Q}\left(\frac{\mathcal{y}^{\rm bin}(\ell)^\trans\left(I_{mk}-\mathcal{D}_{{\bm\delta}^{[mk]}}^{-\frac 12}\Gamma\mathcal{D}_{{\bm\delta}^{[mk]}}^{\frac 12}\right)(e_{m(i-1)+\ell}^{[mk]}-e_{m(i-1)+j}^{[mk]})}{\sqrt{\mathcal{y}^{\rm bin}(\ell)^{\trans}\mathcal{D}_{{\bm\delta}^{[mk]}}^{\frac 12}\Gamma\mathcal{V}_{ij}\Gamma\mathcal{D}_{{\bm\delta}^{[mk]}}^{\frac 12}\mathcal{y}^{\rm bin}(\ell)}}\right)
    \label{eq:optimization_one_versus_all}
\end{align}
with $\mathcal{Q}$ the Gaussian q-function.


Being a non-convex and non-differentiable (due to the max) optimization, Equation~\ref{eq:optimization_one_versus_all} cannot be solved straightforwardly. An approximated solution consists in relaxing the max operator $\max(x_1,\ldots,x_n)$ into the differentiable operator $\frac1{\gamma n}\log(\sum_{j=1}^n \exp(\gamma x_j))$ for some $\gamma>0$, and use a standard gradient descent optimization scheme here initialized at $\mathcal y^{\rm bin}(\ell)\in\mathbb R^{mk}$ filled with $1$'s at every $m(i'-1)+\ell$, for $i'\in\{1,\ldots,m\}$, and with $-1$'s everywhere else.

\medskip

In effect, the optimized vector ${{}\mathcal{y}^{\rm bin}}^{\star}(\ell)$ is evaluated first \emph{before} the constant shift scalar $\bar{\mathcal{y}}$ (ensuring that $\mathbb E_{{\bf x}\in\mathcal C_\ell}[g_i^{\rm bin}({\bf x};\ell)]$ is close to zero) is added to ${{}\mathcal{y}^{\rm bin}}^{\star}(\ell)$. This order of treatment is mandatory as $\mathbb E_{{\bf x}\in\mathcal C_\ell}[g_i^{\rm bin}({\bf x};\ell)]$ depends explicitly on the value of the input score vector $\mathcal y^{\rm bin}$. This global procedure is described in Algorithm~\ref{alg:multi class} below.

\begin{algorithm}
 \caption{Proposed one-versus-all multi-task learning algorithm.}
 \label{alg:multi class}
 \begin{algorithmic}
     \STATE {{\bfseries Input:} Training samples $X=[X_1,\ldots,X_k]$ with $X_i=[X_i^{(1)},\ldots,X_i^{(m)}]$, $X_i^{(j)}\in\mathbb{R}^{p\times n_{ij}}$ and test data ${\bf x}$.}
     \STATE {{\bfseries Output:} Estimated class $\hat{\ell}\in\{1,\ldots,m\}$ of $\bf x$ for Task~$i$.}
     \FOR {$\ell=1$ {\bfseries to} $m$}
        \STATE {\bfseries Center and normalize} 
        data per task: for all $i'\in\{1,\ldots,k\}$,
        \begin{itemize}
            \item $\mathring{X}_{i'} \leftarrow X_{i'} \left( I_{n_{i'}} - \frac1{n_{i'}}\mathbb{1}_{n_{i'}}\mathbb{1}_{n_{i'}}^\trans \right)$
            \item $\mathring{X}_{i'} \leftarrow \mathring{X}_{i'}/\frac1{n_{i'}p}\tr (\mathring{X}_{i'}\mathring{X}_{i'}^\trans)$.
        \end{itemize}
        \STATE {{\bfseries Estimate:} $\mathbb{M}^\trans\bar{\tilde{Q}}_0\mathbb{M}$ and $\mathcal{V}_{i\ell}$ according to Remark~\ref{rem:estimation_general}. }
         \STATE {\bfseries Create} scores ${{}\mathcal{y}^{\rm bin}}^{\star}(\ell)$ by numerically solving \eqref{eq:optimization_one_versus_all} (see discussion following \eqref{eq:optimization_one_versus_all}). 
         \STATE {\bfseries Shift} scores ${{}\mathcal{y}^{\rm bin}}^{\star}(\ell)$ according to Remark~\ref{rem:zero_threshold}.
         \STATE {\bfseries Estimate} $C_{i\ell}(\ell)$ from Theorem~\ref{th:main} and Remark~\ref{rem:estimation_general}.
         \STATE {\bfseries Compute} classification scores $g_i^{\rm bin}({\bf x};\ell)$ according to \eqref{eq:score_class}.
     \ENDFOR
     \STATE {{\bfseries Output: }} $\hat \ell=\argmax_{\ell\in\{1,\ldots,m\}} \left\{\frac{g_i^{\rm bin}({\bf x};\ell)}{\sqrt{C_{i\ell}(\ell)}}\right\}$.
 \end{algorithmic}
 \end{algorithm}

As an immediate corollary of Theorem~\ref{th:main}, for large dimensional data, the classification accuracy of Algorithm~\ref{alg:multi class} can be precisely estimated, as follows.
\begin{proposition}
\label{prop:classification_accuracy_one_vs_all}
Under the notations of Theorem~\ref{th:main}, the probability of correct classification $P_i^{(j)}({\bf x})$ for Task~$i$ of a test data ${\bf x}\in \mathcal{C}_j$ is given by
\begin{equation*}
    P_i^{(j)}({\bf x})=\underbrace{\idotsint_{0}^{\infty}}_{m-1}\frac{1}{\sqrt{(2\pi)^{m-1}|C(j)|}}\exp\left(-\frac 12 (x-\mu(j))^{\trans}C(j)^{-1}(x-\mu(j))\right)dx+o(1)
\end{equation*}
where $\mu(j)=\mathcal{Y}_{-j}(I_{mk}-\mathcal{D}_{\bm\delta}^{-\frac 12}\Gamma\mathcal{D}_{\bm\delta}^{\frac 12})e_{m(i-1)+j}^{[mk]}\in\mathbb{R}^{m-1}$ and $C(j)=\mathcal{Y}_{-j}\mathcal{D}_{\bm\delta}^{\frac12}\Gamma\mathcal{V}_{ij}\Gamma\mathcal{D}_{\bm\delta}^{\frac12}\mathcal{Y}_{-j}^{\trans}\in\mathbb{R}^{(m-1)\times (m-1)}$, with
 ${\mathcal{Y}_{-j}=\{\mathcal{y}^{\rm bin}(j)^\trans-\mathcal{y}^{\rm bin}(j')^{\trans}}\}_{j'\neq j}\in\mathbb{R}^{(m-1)\times km}$.
 \end{proposition}

Figure~\ref{fig:one-vs_all_explain}, succinctly introduced above, illustrates the successive improvements of the proposed algorithms. Specifically, it shows the gains of the centering-scaling operation on the input and output scores (2nd column) and of the optimization of the input scores (3rd column) when compared with the standard approach (1st column). Here synthetic data arising from a Gaussian mixture model are considered in a two-task ($k=2$) and three-class ($m=3$) setting in which $x_{1l}^{(j)}\sim\mathcal{N}(\mu_{1j},I_p)$ and $x_{2l}^{(j)}\sim\mathcal{N}(\mu_{2j},I_p)$, where $\mu_{2j}=\beta\mu_{1j}+\sqrt{1-\beta^2}\mu_{1j}^{\perp}$, with $\mu_{1j}=2 e_{j}^{[p]}$ and $\mu_{1j}^{\perp}=e_{p-j}^{[p]}$. Here $p=200$, $\beta=0.2$, $[n_{11},n_{12},n_{13},n_{21},n_{22},n_{23}]=[393, 309,394,20,180,480]$ and the optimization framework used for input score (label) $\mathcal y^{\rm bin}$ optimization is a standard interior point method \citep{boyd2004convex}.\footnote{Here we use the \texttt{fmincon} function implemented in Matlab.}
\begin{figure}
    \centering
    \begin{tabular}{lll}
 \hspace{-1cm}
  \begin{tikzpicture}[scale=0.55]
      \begin{axis}[ymin=0,title={\huge\texttt{Classical}},domain=-1.5:1,samples=50,smooth]
          \addplot[red,line width=2pt] {gauss(-0.7558,0.2274)};
          \addplot[blue,line width=2pt] {gauss(-0.9275,0.2249)};
          \addplot[green,line width=2pt] {gauss(-0.9540,0.2247)};
          \draw [dashed,line width=2pt] (0,0) -- (0,2);
          \legend{$g_i(x)|x\in\mathcal{C}_1$,$g_i(x)|x\in\mathcal{C}_2$,$g_i(x)|x\in\mathcal{C}_3$}
      \end{axis}
      
  \end{tikzpicture}
 &
 \begin{tikzpicture}[scale=0.55]
     \begin{axis}[ymin=0,title={\huge\texttt{Scaled scores}},domain=-4:4,samples=50,smooth]
         \addplot[red,line width=2pt] {gauss(0,1)};
          \addplot[blue,line width=2pt] {gauss(-0.7552,0.9898)};
          \addplot[green,line width=2pt] {gauss(-0.8719,0.9886)};
          \draw [dashed,line width=2pt] (0,0) -- (0,2);
          \draw [dashed] (-0.7552,0) -- (-0.7552,2);
           \draw[<->] (-0.7552,0.42)--(0,0.42);
     \end{axis}
 \end{tikzpicture}
 &
 \begin{tikzpicture}[scale=0.55]
     \begin{axis}[ymin=0, title={\huge\texttt{Optimized labels}},domain=-4:4,samples=50,smooth]
         \addplot[red,line width=2pt] {gauss(0,1)};
         \addplot[blue,line width=2pt] {gauss(-1.37,0.9575)};
         \addplot[green,line width=2pt] {gauss(-1.3648,0.9551)};
         \draw [dashed] (-1.3648,0) -- (-1.3648,2);
         \draw [dashed,line width=2pt] (0,0) -- (0,2);
          \draw[<->] (-1.3648,0.43)--(0,0.43);
     \end{axis}
 \end{tikzpicture}
  \\
  \hspace{-1cm}
  \begin{tikzpicture}[scale=0.55]
     \begin{axis}[ymin=0,domain=-1.5:1,samples=50,smooth]
          \addplot[blue,line width=2pt] {gauss(0.1389,0.3835)};
          \addplot[red,line width=2pt] {gauss(-0.3546,0.3831)};
          \addplot[green,line width=2pt] {gauss(-0.7040,0.3824)};
          \draw [dashed,line width=2pt] (0,0) -- (0,2);
     \end{axis}
 \end{tikzpicture}
 &
  \begin{tikzpicture}[scale=0.55]
      \begin{axis}[ymin=0,domain=-4:4,samples=50,smooth]
          \addplot[blue,line width=2pt] {gauss(0,1)};
          \addplot[red,line width=2pt] {gauss(-1.2879,1.0008)};
          \addplot[green,line width=2pt] {gauss(-2.1998,0.9980)};
          \draw [dashed,line width=2pt] (0,0) -- (0,2);
      \end{axis}
  \end{tikzpicture}
  &
  \begin{tikzpicture}[scale=0.55]
      \begin{axis}[ymin=0,domain=-4:4,samples=50,smooth]
          \addplot[blue,line width=2pt] {gauss(0,1)};
          \addplot[red,line width=2pt] {gauss(-1.3789,1.0004)};
          \addplot[green,line width=2pt] {gauss(-2.4824,0.9975)};
          \draw [dashed,line width=2pt] (0,0) -- (0,2);
      \end{axis}
      \end{tikzpicture}
    \\
    \hspace{-1cm}
    \begin{tikzpicture}[scale=0.55]
     \begin{axis}[ymin=0,domain=-1.5:1,samples=50,smooth]
          \addplot[green,line width=2pt] {gauss(0.6580,0.3995)};
          \addplot[red,line width=2pt] {gauss(0.1103,0.4009)};
          \addplot[blue,line width=2pt] {gauss(-0.2114,0.4002)};
          \draw [dashed,line width=2pt] (0,0) -- (0,2);
     \end{axis}
 \end{tikzpicture}
 &
 \begin{tikzpicture}[scale=0.55]
     \begin{axis}[ymin=0,domain=-4:4,samples=50,smooth]
          \addplot[green,line width=2pt] {gauss(0,1)};
          \addplot[red,line width=2pt] {gauss(-1.3708,1.0035)};
          \addplot[blue,line width=2pt] {gauss(-2.1761,1.0018)};
          \draw [dashed] (0,0) -- (0,2);
     \end{axis}
 \end{tikzpicture}
 &
 \begin{tikzpicture}[scale=0.55]
     \begin{axis}[ymin=0,domain=-4:4,samples=50,smooth]
          \addplot[green,line width=2pt] {gauss(0,1)};
          \addplot[red,line width=2pt] {gauss(-1.3611,1.0027)};
          \addplot[blue,line width=2pt] {gauss(-2.5043,1.0023)};
          \draw [dashed,line width=2pt] (0,0) -- (0,2);
     \end{axis}
 \end{tikzpicture}
 \end{tabular}
    \caption{Test score distribution in a $2$-task and $3$ classes-per-task setting, using a one-versus-all multi-class classification. Every graph in row~$\ell$ depicts the limiting distributions of $g_i({\bf x};\ell)$ for ${\bf x}$ in different classes. Column~$1$ (Classical) is the standard implementation of the one-versus-all approach. Column~$2$ (Scaled scores) is the output for centered and scaled $g_i({\bf x};\ell)$ for ${\bf x}\in\mathcal C_\ell$. Column~$3$ (Optimized labels) is the same as Column~$2$ but with optimized input scores (labels) ${{}\mathcal y^{\rm bin}}^{\star}(\ell)$. Under ``classical'' approach, data from $\mathcal C_1$ (red curves) will often be misclassified as $\mathcal C_2$. With ``optimized labels'', the discrimination of scores for $\bf x$ in either class $\mathcal C_2$ or $\mathcal C_3$ is improved (blue curve in 2nd row further away from blue curve in 1st row; and similarly for green curve in 3rd versus 1st row).}
    \label{fig:one-vs_all_explain}
\end{figure}
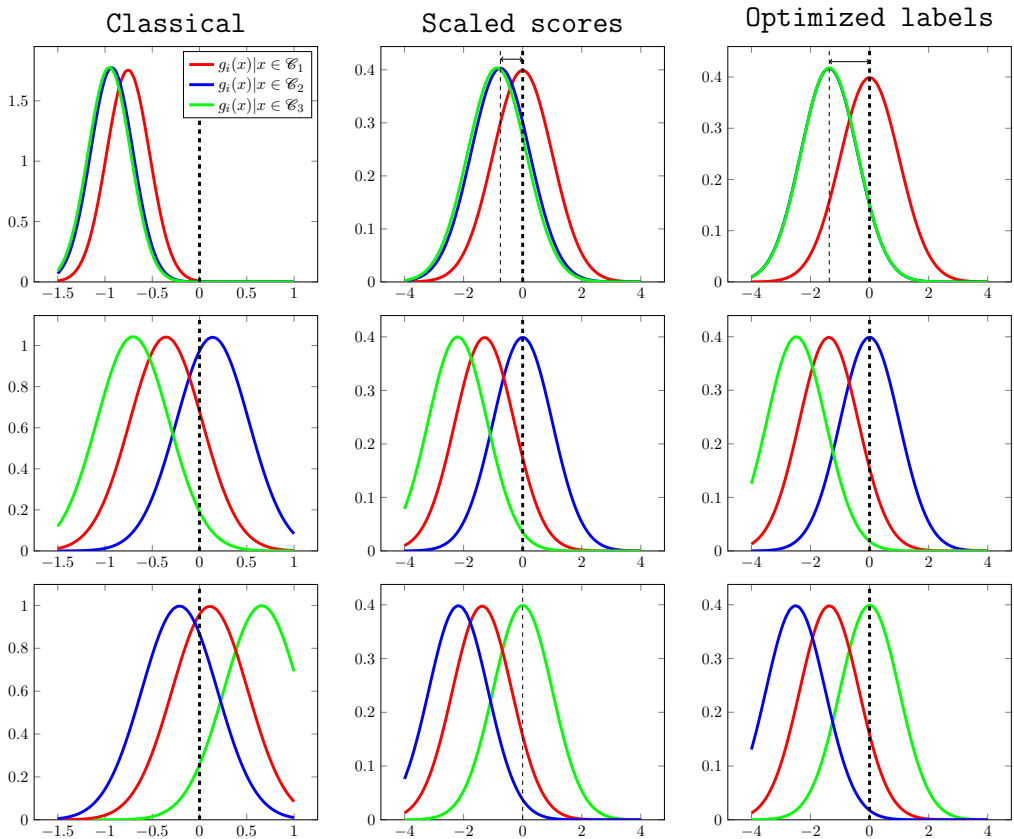
\subsection{One-versus-one multi-class classification}
\label{sec:one-versus-one}

For a given Task $i$, the one-versus-one multi-class method trains $\frac 12 m(m-1)$ binary classifiers for each pair $j\neq j'\in\{1,\ldots,m\}$. As intensively discussed in the previous section, as well as in Section~\ref{sec:theoretical_simple} and Remark~\ref{rem:optimization_general}, each resulting binary classifier $g_i^{\rm bin}({\bf x};j,j')$ can be optimized by choosing optimal input labels $\mathcal y^{\rm bin}(j,j')$. This leads to Algorithm \ref{alg:multi class_one_vs_one} described below.

\begin{algorithm}
 \caption{Proposed one-versus-one multi-task learning algorithm.}
 \label{alg:multi class_one_vs_one}
 \begin{algorithmic}
     \STATE {{\bfseries Input:} Training samples $X=[X_1,\ldots,X_k]$ with $X_i=[X_i^{(1)},\ldots,X_i^{(m)}]$, $X_i^{(j)}\in\mathbb{R}^{p\times n_{ij}}$ and test data ${\bf x}$.}
     \STATE {{\bfseries Output:} Estimated class $\hat \ell\in\{1,\ldots,m\}$ of $\bf x$ for Task~$i$.}
     \STATE {{\bfseries Center and normalize} 
    data per task: for all $i'\in\{1,\ldots,k\}$,
        \begin{itemize}
            \item $\mathring{X}_{i'} \leftarrow X_{i'} \left( I_{n_{i'}} - \frac1{n_{i'}}\mathbb{1}_{n_{i'}}\mathbb{1}_{n_{i'}}^\trans \right)$
            \item $\mathring{X}_{i'} \leftarrow \mathring{X}_{i'}/\frac1{n_{i'}p}\tr (\mathring{X}_{i'}\mathring{X}_{i'}^\trans)$.
        \end{itemize}
        }
     \FOR {$j=1$ {\bfseries to} $m$}
        \FOR {$j'\in\{1,\ldots,m\}\setminus \{j'\}$}
            \STATE {{\bfseries Estimate:} $\mathbb{M}^\trans\bar{\tilde{Q}}_0\mathbb{M}$ and $\mathcal{V}_{ij}$ according to Remark~\ref{rem:estimation_general}.
            }
            \STATE {{\bfseries Create} optimal scores $\mathcal y^{{\rm bin}\star}(j',j)$ according to Remark~\ref{rem:optimization_general}.}
             \STATE {\bfseries Compute} classification scores according to \eqref{eq:score_class} and deduce the predicted class $c(j,j')=j$ or $c(j,j')=j'$ based on the decision rule in \eqref{eq:am_test}.
         \ENDFOR
     \ENDFOR
     \STATE {{\bfseries Output: }} $\hat j=\mode\limits_{j',j\in\{1,\ldots,m\}} \{c(j,j')\}$.\footnotemark
 \end{algorithmic}
 \end{algorithm}
  \footnotetext{The mode of a set of indices is defined as the most frequent value. When multiple indices occur equally frequently, the smallest of those indices is considered by convention.}
In order to derive the asymptotic correct classification of class $\ell$ based on Algorithm~\ref{alg:multi class_one_vs_one}, it is necessary to enumerate all scenarios which lead to the prediction of the class $\ell$. Although this could be done in theory, the combinatorics, already for three classes, are cumbersome and not worth developing here. For the specific one-versus-one setting, we therefore do not provide a theoretical performance analysis.

\subsection{One-hot encoding approach}
\label{sec:one-hot}
For a given Task~$i$ in a one-hot encoding approach, using the canonical vector encoding for each class (i.e., $\mathcal Y_{ij}=e_j^{[m]}$ encodes all training input data $x_{il}^{(j)}$ of class $\mathcal C_j$), the class allocated to an unknown test sample $\bf x$ is the index of the output vector $g_i({\bf x})\in\mathbb R^m$ with maximum value. 

We disrupt here from this approach by explicitly not imposing a one-hot encoding for $\mathcal Y_{ij}$. Instead we consider a generic encoding $\mathcal{Y}\in\mathbb{R}^{km\times m}$, which will be optimized in such a way to maximize the classification accuracy.
\begin{proposition}
\label{prop:classification_accuracy_one_hot}
Under a ``one-hot encoding'' scheme with generic $\mathcal Y$, the probability of correct classification $P_i^{(j)}({\bf x})$ for Task~$i$ of a test data ${\bf x}\in \mathcal{C}_j$ is given by
\begin{equation*}
    P_i^{(j)}({\bf x})=\underbrace{\idotsint_{0}^{\infty}}_{m-1}\frac{1}{\sqrt{(2\pi)^{m-1}|C(j)|}}\exp\left(-\frac 12 (x-\mu(j))^{\trans}C(j)^{-1}(x-\mu(j))\right)dx,
\end{equation*}
where $\mu(j)=\mathcal{E}_j\mathcal{Y}^{T}\left(I_{mk}-\mathcal{D}_{\bm\delta}^{\frac 12}\Gamma\mathcal{D}_{\bm\delta}^{\frac 12}\right)e_{(m(i-1)+j)}^{[km]}\in\mathbb{R}^{m-1}$ and
$C(j)=\mathcal{E}_j\mathcal{Y}^\trans\mathcal{D}_{\bm\delta}^{\frac 12}\Gamma\mathcal{V}_{ij}\Gamma\mathcal{D}_{\bm\delta}^{\frac 12}\mathcal{Y}\mathcal{E}_j^\trans\in\mathbb{R}^{(m-1)\times (m-1)}$
with
$\mathcal{E}_j=\{(e_j^{(m)}-e_{j'}^{(m)})^\trans\}_{j\neq j'}\in\mathbb{R}^{(m-1)\times m}$.
\end{proposition}
A natural objective is to set $\mathcal Y$ so to maximize the average correct classification accuracy $\frac 1m \sum_{j=1}^m P_i^{(j)}({\bf x})$ (under assumed uniform prior on $\bf x$). This form again is not convex in $\mathcal Y$ but may be approximated by gradient descent starting from the one-hot encoding solution, as described in Algorithm~\ref{alg:multi class_one_hot}.
\begin{algorithm}
 \caption{Proposed ``one-hot encoding'' multi-task learning algorithm.}
 \label{alg:multi class_one_hot}
 \begin{algorithmic}
     \STATE {{\bfseries Input:} Training samples $X=[X_1,\ldots,X_k]$ with $X_i=[X_i^{(1)},\ldots,X_i^{(m)}]$, $X_i^{j}\in\mathbb{R}^{p\times n_{ij}}$ and test data ${\bf x}$.}
     \STATE {{\bfseries Output:} Estimated class $\hat \ell\in\{1,\ldots,m\}$ of $\bf x$ for target Task~$i$.}
     \STATE {{\bfseries Center and normalize} 
    data per task: for all $i'\in\{1,\ldots,k\}$,
        \begin{itemize}
            \item $\mathring{X}_{i'} \leftarrow X_{i'} \left( I_{n_{i'}} - \frac1{n_{i'}}\mathbb{1}_{n_{i'}}\mathbb{1}_{n_{i'}}^\trans \right)$
            \item $\mathring{X}_{i'} \leftarrow \mathring{X}_{i'}/\frac1{n_{i'}p}\tr (\mathring{X}_{i'}\mathring{X}_{i'}^\trans)$.
        \end{itemize}
        }
    \STATE {{\bfseries Estimate} Matrix $\mathbb{ M}^\trans\tilde{Q}_0\mathbb{M}$ and $\mathcal{V}_{ij}$ according to Remark~\ref{rem:estimation_general}.
    \STATE {\bfseries Compute} the theoretical score $\mu(j)$ and covariance $C(j)$ and derive the asymptotic classification accuracy $P_i({\bf x})=\frac 1m \sum\limits_{j=1}^m P_i^{(j)}({\bf x})$.    
    \STATE {\bfseries Create} optimal scores $\mathcal Y^\star=\argmax_{\mathcal Y} P_i(\bf x)$.
     
      }
     \STATE {\bfseries Compute} classification scores $g_i({\bf x})$ according to \eqref{eq:score_class}.
     \STATE {{\bfseries Output: }} $\hat \ell=\argmax_{\ell\in\{1,\ldots,m\}} g_i({\bf x};\ell)$.
 \end{algorithmic}
 \end{algorithm}

\section{Experiments}
\label{experiments}

This section has a double objective. The first part (Section~\ref{sec:exp_binary}) devises numerical experiments on binary classification settings to corroborate the theoretical analyses and conclusions drawn in this previous section. Here, the target is threefold: (i) empirically illustrate the effects of the bias in the threshold decision and in the label optimization scheme discussed in Section~\ref{sec:theoretical_simple}, (ii) discuss the impact of numerous tasks ($k>2$) in the binary class setting, thereby emphasizing the effects of negative transfer and its correction through input score (label) optimization, and (iii) exemplify the relevance of our theoretical findings to a specific application to hypothesis testing in a multi-task setting.

In a second part (Section~\ref{sec:exp_multiclass}), experiments on both synthetic and real data for multi-class classification are realized, which demonstrate, even for real data: (i) the extreme accuracy of the theoretical performance predictions of Propositions~\ref{prop:classification_accuracy_one_vs_all}--\ref{prop:classification_accuracy_one_hot} against empirical data and (ii) the large performance gains induced by the various improvements introduced at length in Section~\ref{application}.

\subsection{Experiments on binary classification}
\label{sec:exp_binary}

\subsubsection{Effect of input score (label) choice}

In the present experiment, the effects of the bias in the decision threshold (in general not centered on zero) and of the input score (label) optimization are demonstrated on both synthetic data and real data. 

Specifically, MTL-LSSVM is first applied to the following two-task ($k=2$) and two-class ($m=2$) setting: for Task~$1$, $x^{(j)}_{1l}\sim\mathcal{N}((-1)^j\mu_1,I_p)$ (evenly distributed in both classes) and for Task~$2$, $x_{2l}^{(j)}\sim\mathcal{N}((-1)^j\mu_2,I_p)$ (evenly distributed in both classes), where $\mu_2=\beta\mu_1+\sqrt{1-\beta^2}\mu_{1}^{\perp}$ and $\mu_{1}^{\perp}$ is any vector orthogonal to $\mu_{1}$ and $\beta\in[0,1]$. This setting allows us to tune, through $\beta$, the similarity between tasks. For four different values of $\beta$, Figure~\ref{fig:binary_label} depicts the distribution of the binary output scores $g_i^{\rm bin}({\bf x})$ both for the classical MTL-LSSVM (top displays) and for our proposed random matrix improved scheme, with optimized input labels (bottom display). 

As a first remark, note that both theoretical prediction and empirical outputs closely fit for all values of $\beta$, thereby corroborating our theoretical findings. 
In practical terms, the figure supports (i) the importance to estimate the threshold decision which is non-trivial (not always close to zero) and (ii) the relevance of an appropriate choice of the input labels to improve the discrimination performance between both classes, especially when the two tasks are not quite related. In effect, the entries of ${{}\mathcal y^{\rm bin}}^\star$ naturally drop to zero for all unrelated tasks and classes, thereby discarding the undesired use of the latter; the classical binary input labels instead inappropriately exploit these data and induce a negative learning effect, sometimes to such an extent to completely switch the final decision (as here when $\beta=-1$).

 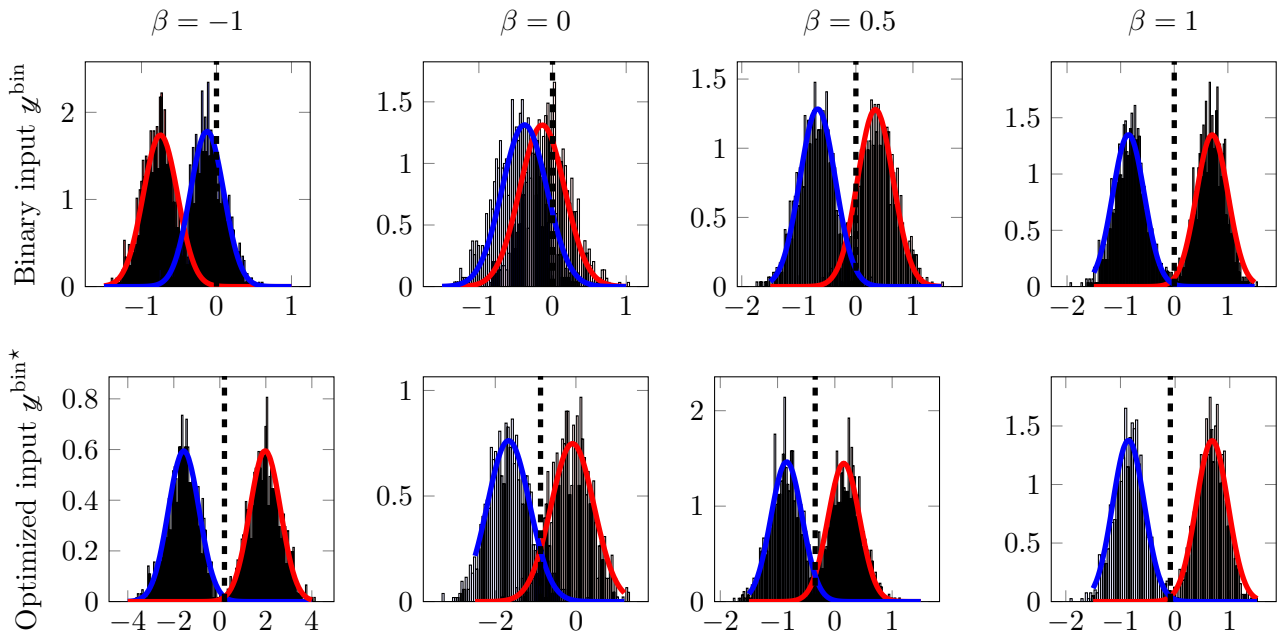
\begin{figure}
 \begin{tabular}{llll}
 \hspace{-1.5cm}
  \begin{tikzpicture}
      \begin{axis}[title={\texttt{$\beta=-1$}},domain=-1.5:1,samples=50,smooth,ymin=0,ylabel={ Binary input $\mathcal y^{\rm bin}$},width=.3\linewidth,height=.3\linewidth]
          \addplot [hist={density,bins=70},fill=red, fill opacity=0.6] table [y index=0] {Fig-text/score1_no_beta11.txt};
          \addplot [hist={density,bins=70},fill=blue, fill opacity=0.6] table [y index=0] {Fig-text/score2_no_beta11.txt};
          \addplot[red,line width=2pt] {gauss(-0.7480,0.2294)};
          \addplot[blue,line width=2pt] {gauss(-0.1260,0.2240)};
          \draw [dashed,line width=2pt] (0,0) -- (0,3);
      \end{axis}
  \end{tikzpicture}
 &
 \begin{tikzpicture}
     \begin{axis}[title={\texttt{$\beta=0$}},domain=-1.5:1,samples=50,ymin=0,smooth,width=.3\linewidth,height=.3\linewidth]
         \addplot [hist={density,bins=70},fill=red, fill opacity=0.1] table [y index=0] {Fig-text/score1_no_beta0.txt};
         \addplot [hist={density,bins=70},fill=blue, fill opacity=0.1] table [y index=0] {Fig-text/score2_no_beta0.txt};
         \addplot[red,line width=2pt] {gauss(-0.1398,0.3042)};
         \addplot[blue,line width=2pt] {gauss(-0.3801,0.3038)};
         \draw [dashed,line width=2pt] (0,0) -- (0,2);
     \end{axis}
 \end{tikzpicture}
 &
 \begin{tikzpicture}
     \begin{axis}[ymin=0, title={\texttt{$\beta=0.5$}},domain=-1.5:1.5,ymin=0,samples=50,smooth,width=.3\linewidth,height=.3\linewidth]
         \addplot [hist={density,bins=70},fill=red, fill opacity=0.1] table [y index=0] {Fig-text/score1_no_beta5.txt};
         \addplot [hist={density,bins=70},fill=blue, fill opacity=0.1] table [y index=0] {Fig-text/score2_no_beta5.txt};
         \addplot[red,line width=2pt] {gauss(0.3416,0.3117)};
         \addplot[blue,line width=2pt] {gauss(-0.6708,0.3107)};
         \draw [dashed,line width=2pt] (0,0) -- (0,2);
     \end{axis}
 \end{tikzpicture}
 &
 \begin{tikzpicture}
     \begin{axis}[ymin=0, title={\texttt{$\beta=1$}},domain=-1.5:1.5,samples=50,ymin=0,smooth,width=.3\linewidth,height=.3\linewidth]
          \addplot [hist={density,bins=70},fill=red, fill opacity=0.1] table [y index=0] {Fig-text/score1_no_beta1.txt};
          \addplot [hist={density,bins=70},fill=blue, fill opacity=0.1] table [y index=0] {Fig-text/score2_no_beta1.txt};
          \addplot[red,line width=2pt] {gauss(0.7077,0.2961)};
          \addplot[blue,line width=2pt] {gauss(-0.8538,0.2957)};
          \draw [dashed,line width=2pt] (0,0) -- (0,2);
     \end{axis}
 \end{tikzpicture}
  \\
  \hspace{-1.5cm}
  \begin{tikzpicture}
     \begin{axis}[domain=-4:4,samples=50,ymin=0,smooth,ylabel={ Optimized input ${{}\mathcal y^{{\rm bin}}}^\star$},width=.3\linewidth,height=.3\linewidth]
          \addplot [hist={density,bins=70},fill=blue, fill opacity=0.1] table [y index=0] {Fig-text/score1_o_beta11.txt};
          \addplot [hist={density,bins=70},fill=red, fill opacity=0.1] table [y index=0] {Fig-text/score2_o_beta11.txt};
          \addplot[blue,line width=2pt] {gauss(-1.5683,0.6713)};
          \addplot[red,line width=2pt] {gauss(1.9722,0.6703)};
          \draw [dashed,line width=2pt] (0.2019,0) -- (0.2019,2);
     \end{axis}
 \end{tikzpicture}
 &
  \begin{tikzpicture}
      \begin{axis}[domain=-2.5:1.2,samples=50,ymin=0,smooth,width=.3\linewidth,height=.3\linewidth]
          \addplot [hist={density,bins=70},fill=red, fill opacity=0.1] table [y index=0] {Fig-text/score1_o_beta0.txt};
          \addplot [hist={density,bins=70},fill=blue, fill opacity=0.1] table [y index=0] {Fig-text/score2_o_beta0.txt};
          \addplot[red,line width=2pt] {gauss(-0.0847,0.5331)};
          \addplot[blue,line width=2pt] {gauss(-1.6699,0.5234)};
          \draw [dashed,line width=2pt] (-0.8773,0) -- (-0.8773,2);
      \end{axis}
  \end{tikzpicture}
  &
  \begin{tikzpicture}
      \begin{axis}[domain=-1.5:1.5,ymin=0,samples=50,smooth,width=.3\linewidth,height=.3\linewidth]
          \addplot [hist={density,bins=70},fill=red, fill opacity=0.1] table [y index=0] {Fig-text/score1_o_beta5.txt};
          \addplot [hist={density,bins=70},fill=blue, fill opacity=0.1] table [y index=0] {Fig-text/score2_o_beta5.txt};
          \addplot[red,line width=2pt] {gauss(0.1562,0.2755)};
          \addplot[blue,line width=2pt] {gauss(-0.8428,0.2726)};
          \draw [dashed,line width=2pt] (-0.3433,0) -- (-0.3433,3);
      \end{axis}
  \end{tikzpicture}
  &
  \begin{tikzpicture}
      \begin{axis}[domain=-1.5:1.5,samples=50,ymin=0,smooth,width=.3\linewidth,height=.3\linewidth]
          \addplot [hist={density,bins=50},fill=red, fill opacity=0.1] table [y index=0] {Fig-text/score1_o_beta1.txt};
          \addplot [hist={density,bins=50},fill=blue, fill opacity=0.1] table [y index=0] {Fig-text/score2_o_beta1.txt};
          \addplot[red,line width=2pt] {gauss(0.6780,0.2902)};
          \addplot[blue,line width=2pt] {gauss(-0.8526,0.2898)};
          \draw [dashed,line width=2pt] (-0.0873,0) -- (-0.0873,2);
      \end{axis}
  \end{tikzpicture}
 \end{tabular}
 \caption{Score distribution for new datum $\bf x$ of Class~$\mathcal C_1$ (red) and Class $\mathcal C_2$ (blue) for Task~$2$ in a $2$-task ($k=2$) and $2$ class-per-task ($m=2$) setting of isotropic Gaussian mixtures for: {\bf (top)} classical MTL-LSSVM with no optimization and a threshold assumed at $\zeta=0$; {\bf (bottom)} proposed optimized MTL-LSSVM with estimated threshold $\zeta$; decision threholds $\zeta$ represented in dashed vertical lines; differently related tasks ($\beta=0$ for orthogonal means, $\beta>0$ for positively correlated means and $\beta<0$ for negatively correlated means), $p=100$, $[c_{11},c_{12},c_{21},c_{22}]=[0.3,0.4,0.1,0.2]$, $\gamma=\mathbb{1}_2$, $\lambda=10$. Histograms drawn from $1\,000$ test samples of each class. The figure clearly depicts the deviation from $0$ of the decision threshold in unbalanced classes and the deleterious effect of ``negative transfer'' when $\beta$ is small; these problems are well handled by the proposed optimized scheme.}
 \label{fig:binary_label}
 \end{figure}
 
 For experiments on real data, the MNIST datasets \citep{deng2012mnist} is considered. Specifically, the setting is that of a binary classification for two tasks, mimicking a transfer learning setting: there, the ``target'' Task~$2$ aims to discriminate Class $\mathcal C_1$ and Class $\mathcal C_2$ respectively composed of images of digit $1$ and digit $4$. The ``source'' Task~$1$ is here used as a support for classification in the target task, and consists of the classification of other pairs of digits: either $(5,9)$, $(9,5)$, $(6,2)$ or $(8,3)$ (we recall that the order of the set of digits $(X,Y)$ is important for the non-optimized MTL-LSSVM since source and target tasks labels are ``paired''; thus $(5,9)$ or $(9,5)$ digits for the source task will bring different results). We compare here again the non-optimized MTL-LSSVM with labels ${\mathcal y}^{\rm bin}=[-1,1,-1,1]^\trans$ to our proposed optimized scheme (as detailed in Remark~\ref{rem:optimization_general}).  For both methods, the optimal theoretical threshold decision $\zeta$ is used (rather than $\zeta=0$ for the non-optimized setup) in order to emphasize the influence of input score (label) optimization.  
 
 Figure~\ref{fig:Mnist_binary} depicts the performance for both methods as a function of the hyperparameter $\lambda$. We recall that, as $\lambda\to 0$, the multi-task scheme becomes equivalent to independent single-task classifiers, while as $\lambda\to\infty$, both source and target tasks are considered together as one task. Figure~\ref{fig:Mnist_binary} raises the stability of optimal input labelling with respect to $\lambda$: this is explained by the fact that ${{}\mathcal y^{\rm bin}}^\star$ is a \emph{function of $\lambda$} and thus adapts to each value of $\lambda$, even if suboptimal. Besides, for appropriate values of $\lambda$, the proposed improved labelling can largely outperform the non-optimized setting, even here on real data.
 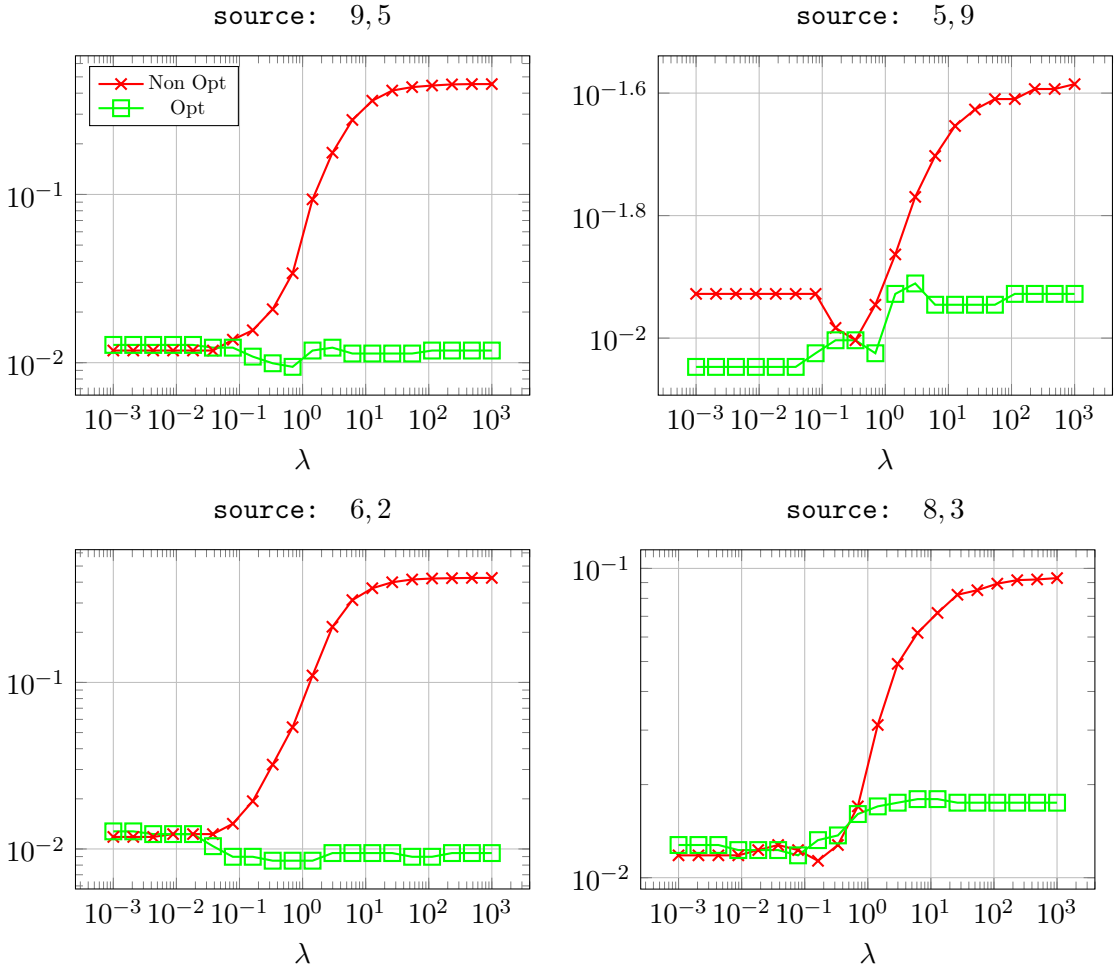
\begin{figure}
 \hspace{-0.5cm}
     \begin{tabular}{ll}
     \begin{tikzpicture}
     \begin{loglogaxis}[title={ \texttt{source: $9,5$}},grid=major,xlabel={ $\lambda$},legend pos=north west,legend style={nodes={scale=0.7, transform shape}},width=.5\linewidth,height=.4\linewidth]
        \addplot[thin,mark=x,mark size=3pt,red,thick]coordinates{
                        (1.000000e-03,1.180916e-02)(2.069138e-03,1.180916e-02)(4.281332e-03,1.180916e-02)(8.858668e-03,1.180916e-02)(1.832981e-02,1.180916e-02)(3.792690e-02,1.180916e-02)(7.847600e-02,1.369863e-02)(1.623777e-01,1.558810e-02)(3.359818e-01,2.078413e-02)(6.951928e-01,3.401039e-02)(1.438450e+00,9.352858e-02)(2.976351e+00,1.771375e-01)(6.158482e+00,2.772792e-01)(1.274275e+01,3.608880e-01)(2.636651e+01,4.147378e-01)(5.455595e+01,4.345772e-01)(1.128838e+02,4.440246e-01)(2.335721e+02,4.511101e-01)(4.832930e+02,4.529995e-01)(1000,4.529995e-01)

};
         \addplot[title=\texttt{source: $5,9$},thin,mark=square,mark size=3pt,green,thick]coordinates{                              
                         (1.000000e-03,1.275390e-02)(2.069138e-03,1.275390e-02)(4.281332e-03,1.275390e-02)(8.858668e-03,1.275390e-02)(1.832981e-02,1.275390e-02)(3.792690e-02,1.228153e-02)(7.847600e-02,1.228153e-02)(1.623777e-01,1.086443e-02)(3.359818e-01,9.919698e-03)(6.951928e-01,9.447331e-03)(1.438450e+00,1.180916e-02)(2.976351e+00,1.228153e-02)(6.158482e+00,1.133680e-02)(1.274275e+01,1.133680e-02)(2.636651e+01,1.133680e-02)(5.455595e+01,1.133680e-02)(1.128838e+02,1.180916e-02)(2.335721e+02,1.180916e-02)(4.832930e+02,1.180916e-02)(1000,1.180916e-02)

}
 ;
 \legend{{Non Opt},{Opt}};
        \end{loglogaxis}
     \end{tikzpicture}
     &
          \begin{tikzpicture}
     \begin{loglogaxis}[title={ \texttt{source: $5,9$}},grid=major,xlabel={ $\lambda$},mark=x,legend pos=south west,legend style={nodes={scale=0.7, transform shape}},width=.5\linewidth,height=.4\linewidth]
         \addplot[thin,mark size=3pt,mark=x,red,thick]coordinates{        (1.000000e-03,1.180916e-02)(2.069138e-03,1.180916e-02)(4.281332e-03,1.180916e-02)(8.858668e-03,1.180916e-02)(1.832981e-02,1.180916e-02)(3.792690e-02,1.180916e-02)(7.847600e-02,1.180916e-02)(1.623777e-01,1.039206e-02)(3.359818e-01,9.919698e-03)(6.951928e-01,1.133680e-02)(1.438450e+00,1.369863e-02)(2.976351e+00,1.700520e-02)(6.158482e+00,1.983940e-02)(1.274275e+01,2.220123e-02)(2.636651e+01,2.361833e-02)(5.455595e+01,2.456306e-02)(1.128838e+02,2.456306e-02)(2.335721e+02,2.550779e-02)(4.832930e+02,2.550779e-02)(1000,2.598016e-02)

};
          \addplot[thin,mark size=3pt,mark=square,green,thick]coordinates{                                                      (1.000000e-03,8.974965e-03)(2.069138e-03,8.974965e-03)(4.281332e-03,8.974965e-03)(8.858668e-03,8.974965e-03)(1.832981e-02,8.974965e-03)(3.792690e-02,8.974965e-03)(7.847600e-02,9.447331e-03)(1.623777e-01,9.919698e-03)(3.359818e-01,9.919698e-03)(6.951928e-01,9.447331e-03)(1.438450e+00,1.180916e-02)(2.976351e+00,1.228153e-02)(6.158482e+00,1.133680e-02)(1.274275e+01,1.133680e-02)(2.636651e+01,1.133680e-02)(5.455595e+01,1.133680e-02)(1.128838e+02,1.180916e-02)(2.335721e+02,1.180916e-02)(4.832930e+02,1.180916e-02)(1000,1.180916e-02)

};
        \end{loglogaxis}
     \end{tikzpicture}
     \\
          \begin{tikzpicture}
     \begin{loglogaxis}[title={ \texttt{source: $6,2$}},grid=major,xlabel={ $\lambda$},mark=x,legend pos=south west,legend style={nodes={scale=0.7, transform shape}},width=.5\linewidth,height=.4\linewidth]
         \addplot[thin,mark size=3pt,mark=x,red,thick]coordinates{            (1.000000e-03,1.180916e-02)(2.069138e-03,1.180916e-02)(4.281332e-03,1.180916e-02)(8.858668e-03,1.228153e-02)(1.832981e-02,1.228153e-02)(3.792690e-02,1.228153e-02)(7.847600e-02,1.417100e-02)(1.623777e-01,1.936703e-02)(3.359818e-01,3.212093e-02)(6.951928e-01,5.384979e-02)(1.438450e+00,1.100614e-01)(2.976351e+00,2.158715e-01)(6.158482e+00,3.122343e-01)(1.274275e+01,3.679735e-01)(2.636651e+01,3.996221e-01)(5.455595e+01,4.152102e-01)(1.128838e+02,4.208786e-01)(2.335721e+02,4.227681e-01)(4.832930e+02,4.241852e-01)(1000,4.241852e-01)

};
          \addplot[thin,mark size=3pt,mark=square,green,thick]coordinates{                                                          (1.000000e-03,1.275390e-02)(2.069138e-03,1.275390e-02)(4.281332e-03,1.228153e-02)(8.858668e-03,1.228153e-02)(1.832981e-02,1.228153e-02)(3.792690e-02,1.039206e-02)(7.847600e-02,8.974965e-03)(1.623777e-01,8.974965e-03)(3.359818e-01,8.502598e-03)(6.951928e-01,8.502598e-03)(1.438450e+00,8.502598e-03)(2.976351e+00,9.447331e-03)(6.158482e+00,9.447331e-03)(1.274275e+01,9.447331e-03)(2.636651e+01,9.447331e-03)(5.455595e+01,8.974965e-03)(1.128838e+02,8.974965e-03)(2.335721e+02,9.447331e-03)(4.832930e+02,9.447331e-03)(1000,9.447331e-03)

}
  ;
        \end{loglogaxis}
     \end{tikzpicture}
     &
     \begin{tikzpicture}
        \begin{loglogaxis}[title={{ \texttt{ source: $8,3$}}},grid=major,mark=x,xlabel={$\lambda$},legend pos=south west,legend style={nodes={scale=0.7, transform shape}},width=.5\linewidth,height=.4\linewidth]
         \addplot[thin,mark=x,mark size=3pt,red,thick]coordinates{            (1.000000e-03,1.180916e-02)(2.069138e-03,1.180916e-02)(4.281332e-03,1.180916e-02)(8.858668e-03,1.180916e-02)(1.832981e-02,1.228153e-02)(3.792690e-02,1.275390e-02)(7.847600e-02,1.228153e-02)(1.623777e-01,1.133680e-02)(3.359818e-01,1.275390e-02)(6.951928e-01,1.700520e-02)(1.438450e+00,3.117619e-02)(2.976351e+00,4.912612e-02)(6.158482e+00,6.188002e-02)(1.274275e+01,7.179972e-02)(2.636651e+01,8.219178e-02)(5.455595e+01,8.502598e-02)(1.128838e+02,8.927728e-02)(2.335721e+02,9.163911e-02)(4.832930e+02,9.211148e-02)(1000,9.305621e-02)

};
          \addplot[thin,mark=square,mark size=3pt,green,thick]coordinates{                                                (1.000000e-03,1.275390e-02)(2.069138e-03,1.275390e-02)(4.281332e-03,1.275390e-02)(8.858668e-03,1.228153e-02)(1.832981e-02,1.228153e-02)(3.792690e-02,1.228153e-02)(7.847600e-02,1.180916e-02)(1.623777e-01,1.322626e-02)(3.359818e-01,1.369863e-02)(6.951928e-01,1.606046e-02)(1.438450e+00,1.700520e-02)(2.976351e+00,1.747756e-02)(6.158482e+00,1.794993e-02)(1.274275e+01,1.794993e-02)(2.636651e+01,1.747756e-02)(5.455595e+01,1.747756e-02)(1.128838e+02,1.747756e-02)(2.335721e+02,1.747756e-02)(4.832930e+02,1.747756e-02)(1000,1.747756e-02)

};
        \end{loglogaxis}
     \end{tikzpicture}
          \end{tabular}
     \caption{Classification error of digit pair $(1,4)$ with different source training pairs for classical LSSVM and optimized LSSVM. $n_{11}=n_{12}=100$, $n_{21}=n_{22}=10$ and $\gamma=\mathbb{1}_2$. A PCA preprocessing is performed on each image to extract their $p=100$ principal components; the accuracy is performed over $n_{\rm test}=1\,135$ test samples. The proposed method shows a low sensitivity to $\lambda$. 
     }
     \label{fig:Mnist_binary}
 \end{figure}

 Table~\ref{tab:sufficient_satitstics} complements the figure by effectively displaying the optimal vectors ${{}\mathcal y^{\rm bin}}^\star$ at the optimal value for $\lambda$. The table demonstrates the appropriate adjustment of the labels to the data correlation $\frac{\Delta\mu_1^\trans\Delta\mu_2}{\|\Delta\mu_2\|^2}$. Specifically, for a negative correlation between the classes of both tasks, the method naturally ``switches'' the labels (the input data scores) by opposing the signs of ${{}\mathcal y^{\rm bin}}^\star$ in entries $1,3$ (Class~$\mathcal C_1$ in each task) and $2,4$ (Class~$\mathcal C_2$ in each task). 
 For rather orthogonal tasks (here typically $(8,3)$), the entries of ${{}\mathcal y^{\rm bin}}^\star$ corresponding to the source task (entries $1$ and $2$) are almost zero, thereby discarding the source data and avoiding negative transfer.
 It is also interesting to note that, for moderately correlated tasks (here for the source digits $(5,9)$), despite the fact that the source task offers ten times more data ($n_{1j}=100$, $n_{2j}=10$) and is thus deemed trustworthy for classification, the corresponding entries $1$ and $2$ in ${{}\mathcal y^{\rm bin}}^\star$ are much smaller than the entries $3$, $4$ corresponding to the target task: the algorithm thus judges the few target data more relevant to target classification than the many related source tasks.
 \begin{table}
      \centering
      \begin{tabular}{c|c|c|c|c}
           [Source]&(9,5)&(5,9)&(6,2)&(8,3) \\
           \hline
           $\frac{\Delta\mu_1^\trans\Delta\mu_2}{\|\Delta\mu_2\|^2}$&-0.2450 &0.2450 &-0.1670& -0.0818\\
           \hline
           ${\mathcal y^{\rm bin}}^\star=\begin{bmatrix}{\mathcal{y}_{11}^{\rm bin}}^\star\\ {\mathcal{y}_{12}^{\rm bin}}^\star\\ {\mathcal{y}_{21}^{\rm bin}}^\star\\ {\mathcal{y}_{22}^{\rm bin}}^\star \end{bmatrix}$&$\begin{bmatrix}-0.2808\\ 0.2808\\ 0.6489\\ -0.6489 \end{bmatrix}$ &$\begin{bmatrix}0.2808\\ -0.2808\\ 0.6489\\ -0.6489 \end{bmatrix}$ &$\begin{bmatrix}-0.2879\\ 0.2879\\ 0.6459\\ -0.6459 \end{bmatrix}$ & $\begin{bmatrix}-0.0400\\ 0.0400\\ 0.7060\\ -0.7060 \end{bmatrix}$
      \end{tabular}
      \caption{Optimal input label ${\mathcal y^{\rm bin}}^\star$ as a function of the source data pair in the ($\lambda$-optimal) configuration of Figure~\ref{fig:Mnist_binary}.}
      \label{tab:sufficient_satitstics}
  \end{table}
  
\subsubsection{Analysis of increasing number of tasks}
This next experiment illustrates the effect of adding more tasks for the transfer learning setting on synthetic and MNIST datasets. For synthetic data, Gaussian classes with mean $\mu_{ij}=\beta\mu_{i1}+\sqrt{1-\beta^2}\mu_{i 1}^\perp$ and various values of $\beta$ are successively added. For the MNIST dataset, different classifications of digits are added progressively to help classify the specific pair of digits $(1,4)$ . Figure~\ref{adding_more_tasks} depicts the classification error after each new task addition, both for a classical binary input label choice and for the proposed optimized input labels. The figure forcefully illustrates that our proposed framework avoids negative transfer, as the classification error of MTL never increases as the number of tasks grows. This is quite unlike the non-optimized scheme which severely suffers from negative transfer.
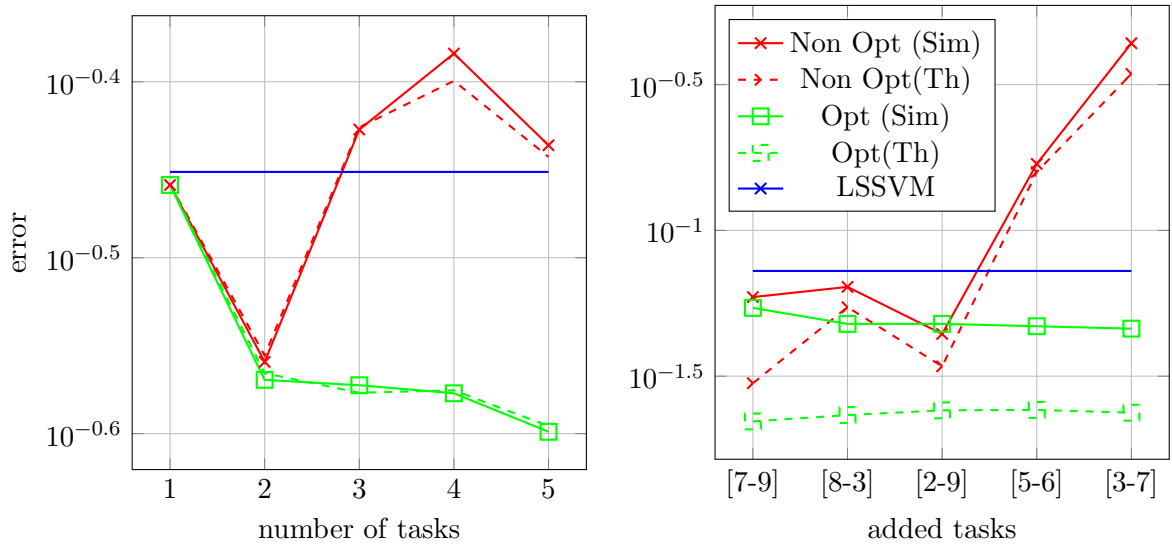
\begin{figure}
    \centering
    \begin{tabular}{ll}
    \begin{tikzpicture}
    \begin{semilogyaxis}[grid=major,mark=x,,legend pos=north west,legend style={nodes={scale=0.7, transform shape}},xlabel={number of tasks},ylabel={error},width=.5\linewidth,height=.5\linewidth]
        \addplot[thin,mark=x,mark size=3pt,red,thick]coordinates{                (1,3.478000e-01)(2,2.759000e-01)(3,3.740000e-01)(4,4.131000e-01)(5,3.664000e-01)
};
\addplot[thin,mark size=3pt,red,thick,dashed]coordinates{                (1,3.481275e-01)(2,2.782663e-01)(3,3.750896e-01)(4,3.986355e-01)(5,3.608890e-01)
};
        \addplot[thin,mark=square,mark size=3pt,green,thick]coordinates{             (1,3.478000e-01)(2,2.695000e-01)(3,2.676000e-01)(4,2.648000e-01)(5,2.518000e-01)
};
\addplot[thin,mark size=3pt,green,thick,dashed]coordinates{                    (1,3.481275e-01)(2,2.718528e-01)(3,2.649897e-01)(4,2.658420e-01)(5,2.537048e-01)
};
\addplot[thin,mark size=3pt,blue,thick]coordinates{                    (1,0.3538)(2,0.3538)(3,0.3538)(4,0.3538)(5,0.3538)
};
    \end{semilogyaxis}
    \end{tikzpicture}
    &
    \begin{tikzpicture}
    \begin{semilogyaxis}[grid=major,mark=x,legend pos=north west
    ,xtick=data,xticklabels={[7-9],[8-3],[2-9],[5-6],[3-7]},xlabel={added tasks},width=.5\linewidth,height=.5\linewidth]
        \addplot[thin,mark=x,mark size=3pt,red,thick]coordinates{                (1,5.905327e-02)(2,6.400764e-02)(3,4.415515e-02)(4,1.691029e-01)(5,4.384516e-01)
};
        \addplot[thin,mark=x,mark size=3pt,red,thick,dashed]coordinates{            (1,2.992020e-02)(2,5.452800e-02)(3,3.418007e-02)(4,1.585640e-01)(5,3.441571e-01)

};
\addplot[thin,mark size=3pt,mark=square,green,thick]coordinates{               (1,5.423616e-02)(2,4.779422e-02)(3,4.779422e-02)(4,4.691316e-02)(5,4.603210e-02)
};
\addplot[thin,mark size=3pt,mark=square,green,thick,dashed]coordinates{                   (1,2.212067e-02)(2,2.326790e-02)(3,2.414620e-02)(4,2.422433e-02)(5,2.369790e-02)

};
\addplot[thin,mark size=3pt,blue,thick]coordinates{               (1,0.0726)(2,0.0726)(3,0.0726)(4,0.0726)(5,0.0726)
};
\legend{Non Opt (Sim),Non Opt(Th),Opt (Sim),Opt(Th),LSSVM}
    \end{semilogyaxis}
    \end{tikzpicture}
    \end{tabular}
    \caption{Classification accuracy for increasing number of tasks. {\bf (Left)} Synthetic data with task correlations $\beta=1$, $.9$, $.5$, $.2$, $.8$ in this order, $p=100$ and $c=[.07,.11,.10,.10,.06,.08,.09,.12,.10,.11,.03,.03]^\trans$; accuracy evaluated out of $10\,000$ test samples. {\bf (Right)} MNIST dataset with digits $(1,4)$ as target task, each added task being shown in x-axis; $100$ training samples are used for each class of the source tasks and $10$ training samples for each class of the target class; HOG features with $p=144$ for each image digit; accuracy evaluated out of $n_{\rm test}=1\,135$ test samples. For both setting, $\gamma=\mathbb{1}_{k}$ and $\lambda=10$.
    The optimized scheme avoids negative transfer by systematically benefiting from additional tasks.
    }
    \label{adding_more_tasks}
\end{figure}
\subsubsection{Hypothesis testing}

The next experiments, both synthetic and on real data, apply the results of MTL-LSSVM to a hypothesis test on a \emph{target} Task~$t$ based on training samples both from a source Task~$s$ and the target Task~$t$. For data $\bf x$ in the target task, the test
\begin{align*}
g_t^{\rm bin}({\bf x})\underset{\mathcal{H}_0}{\overset{\mathcal{H}_1}{\gtrless}} \zeta
\end{align*}
is performed, where $\mathcal H_0$ is the null hypothesis (say, Class~$2$) and $\mathcal H_1$ the alternative (say, Class~$1$) and $\zeta=\zeta(\eta)$ is a decision threshold here selected in such a way to enforce the false alarm rate constraint $P(g_t^{\rm}({\bf x})\geq \zeta(\eta)~|~{\bf x}\in \mathcal H_0)\leq \eta$, for a given $\eta\in (0,1)$. The objective is then to maximize over the input scores ${\mathcal y}^{\rm bin}$ the correct detection rate $P(g_t^{\rm bin}({\bf x})\geq \zeta(\eta)~|~{\bf x}\in \mathcal H_1)$: this induces a different value for the optimal scores ${\mathcal y^{\rm bin}}^\star$ than proposed in \eqref{eq:ty_opt}, which can be constructed following Remark~\ref{rem:optimization_general}.

The experimental synthetic data is here a two-task ($k=2$) setting in which $x_{1j}\sim\mathcal{N}(\pm\mu_{11},I_p)$ (i.e., $\mu_{12}=-\mu_{11}$) and $x_{2j}\sim\mathcal{N}(\pm\mu_{21},I_p)$, where $\mu_{21}=\beta\mu_{11}+\sqrt{1-\beta^2}\mu_{11}^{\perp}$, $\mu_{11}$ is a unit-norm vector and  $\mu_{11}^{\perp}$ any unit-norm vector orthogonal to $\mu_{11}$. We take here $\beta=0.5$, so that both tasks are ``slightly'' correlated. As for the real-world experiment, they are based on the MIT-BIH Arrhythmia dataset \citep{moody2001impact}. The dataset consists of $109\,446$ samples from $5$ medical heart condition categories: ``Normal (N)": $0$, ``Atrial premature (S)": $1$, ``Ventricular (V)": $2$, ``Ventricular-Norma (F)": $3$, and ``Unclassifiable (Q)": $4$. For illustration, we consider here a binary classification with source Classes~$\{1,2\}$ and target Classes~$\{3, 4\}$. A false alarm is raised when misclassifying (target) Class~$3$ into Class~$4$ and the performance objective consists in maximizing the correct classification of target Class~$4$.

\medskip

Figure~\ref{fig:my_label} depicts the algorithm performance through a receiver-operating curve (ROC) for false alarm rates $\eta$ on both synthetic and real-world data. Both theoretical (Th) asymptotics (used to set the decision threshold $\zeta$) and actual performances (Sim) are displayed, for the optimal (Opt) choice of $\mathcal y^{\rm bin}$ (Opt) and for $\mathcal y^{\rm bin}=[-1,1,-1,1]^\trans$ (Non-Opt).

Both synthetic and real data graphs of Figure~\ref{fig:my_label} confirm, here under the hypothesis testing problem, the large superiority of our proposed optimized MTL-LSSVM  over the standard non-optimized alternative. Besides, the theoretical classification error prediction is an accurate fit to the actual empirical performance, even for not so large values of $p$ and the $n_{ij}$'s, and even for small error values.\footnote{Since our main result (Theorem~\ref{th:main}) is a central limit theorem, it is not expected to be particularly accurate in the ``tails'' of the distribution of the output scores $g_i^{\rm bin}({\bf x})$; as such, the observed high accuracy for small error values is remarkable.} This remark is here all the more fundamental that, in practice, $\eta$ can be set a priori, using Theorem~\ref{th:main} with no need for heavy, unreliable, and data-consuming cross-validation procedures.

\begin{figure}
  \centering
  \begin{tabular}{cc}
  \hspace*{-.5cm}\begin{tikzpicture}
		\begin{loglogaxis}[grid=major,xlabel={$P({\bf x}\to\mathcal{C}_2|{\bf x}\in\mathcal{C}_1)$},ylabel={$P({\bf x}\to\mathcal{C}_1|{\bf x}\in\mathcal{C}_2)$},width=.5\linewidth,height=.5\linewidth,y label style={at={(axis description cs:-.25,.5)},anchor=south},legend pos=south west]
			            \addplot[thin,mark=x,mark size=3pt,red,thick]coordinates{                                 (    0.000001,     0.972000)(    0.000003,     0.959000)(    0.000008,     0.934000)(    0.000022,     0.901000)(    0.000060,     0.849000)(    0.000167,     0.776000)(    0.000464,     0.694000)(    0.001292,     0.581000)
};
 			            \addplot[thin,mark=o,mark size=3pt,black,thick]coordinates{                (    0.000001,     0.973325)(    0.000003,     0.957368)(    0.000008,     0.933236)(    0.000022,     0.897765)(    0.000060,     0.847319)(    0.000167,     0.778268)(    0.000464,     0.687955)(    0.001292,     0.576168)

 };			           
 			            \addplot[thin,mark=square,mark size=3pt,green,thick]coordinates{               (    0.000001,     0.232819)(    0.000003,     0.173422)(    0.000008,     0.122756)(    0.000022,     0.081850)(    0.000060,     0.050857)(    0.000167,     0.029026)(    0.000464,     0.014945)(    0.001292,     0.006764)

 };
  		            \addplot[thin,mark=*,mark size=2pt,blue,thick]coordinates{                                   (    0.000001,     0.289000)(    0.000003,     0.218000)(    0.000008,     0.161000)(    0.000022,     0.108000)(    0.000060,     0.072000)(    0.000167,     0.037000)(    0.000464,     0.013000)(    0.001292,     0.005000)

};
 
		\end{loglogaxis}
\end{tikzpicture}
&
  \centering
  \hspace*{-.5cm}\begin{tikzpicture}
		\begin{loglogaxis}[grid=major,xlabel={$P({\bf x}\to\mathcal{C}_2|{\bf x}\in\mathcal{C}_1)$},width=.5\linewidth,height=.5\linewidth,legend pos=south west
		]

			            \addplot[thin,mark=x,mark size=4pt,red,thick]coordinates{                                  (    0.003597,     0.794890)(    0.005396,     0.718923)(    0.007194,     0.626381)(    0.008993,     0.513122)(    0.016187,     0.420580)(    0.032374,     0.323204)(    0.062950,     0.225829)(    0.095324,     0.130525)(    0.151079,     0.071823)

};
\addplot[thin,mark=o,mark size=3pt,black,thick]coordinates{ 
(    0.001799,     0.818370)(    0.005396,     0.743785)(    0.008993,     0.644337)(    0.008993,     0.540055)(    0.014388,     0.429558)(    0.030576,     0.336326)(    0.066547,     0.229282)(    0.098921,     0.138122)(    0.161871,     0.072514)
};
\addplot[thin,mark=*,mark size=2pt,blue,thick]coordinates{                                      (    0.000215,     0.808603)(    0.000464,     0.741061)(    0.001000,     0.658134)(    0.002154,     0.560628)(    0.004642,     0.451778)(    0.010000,     0.337785)(    0.021544,     0.227701)(    0.046416,     0.132208)(    0.100000,     0.060959)

};
\addplot[thin,mark=square,mark size=3pt,green,thick]coordinates{                                          (    0.000215,     0.793796)(    0.000464,     0.724964)(    0.001000,     0.641597)(    0.002154,     0.544771)(    0.004642,     0.437836)(    0.010000,     0.326836)(    0.021544,     0.220340)(    0.046416,     0.128285)(    0.100000,     0.059563)

};
\legend{Non Opt (Sim), Non Opt (Th), Opt (Th), Opt (Sim)};
		\end{loglogaxis}
\end{tikzpicture}
    \end{tabular}
    \caption{ROC curve for proposed optimized versus standard MTL-LSSVM. {\bf (Left)} Synthetic data with $p=128$, $n_{11}=384$, $n_{12}=256$, $n_{21}=64$, $n_{22}=40$, $\mu_{11}=-\mu_{12}=[1,0,\ldots,0]^\trans$, $\mu_{21}=-\mu_{22}=[.87,.5,0,\ldots,0]^\trans$. {\bf (Right)} MIT-BIH arrhythmia database, with $p=550$, $n_{ij}=500$, $\lambda=1$, $\gamma=\mathbb{1}_2$. The accuracy of the theoretical anticipation is remarkable and allows for a precise setting of the decision threshold ensuring a desired false alarm rate.} 
  \label{fig:my_label}
\end{figure}
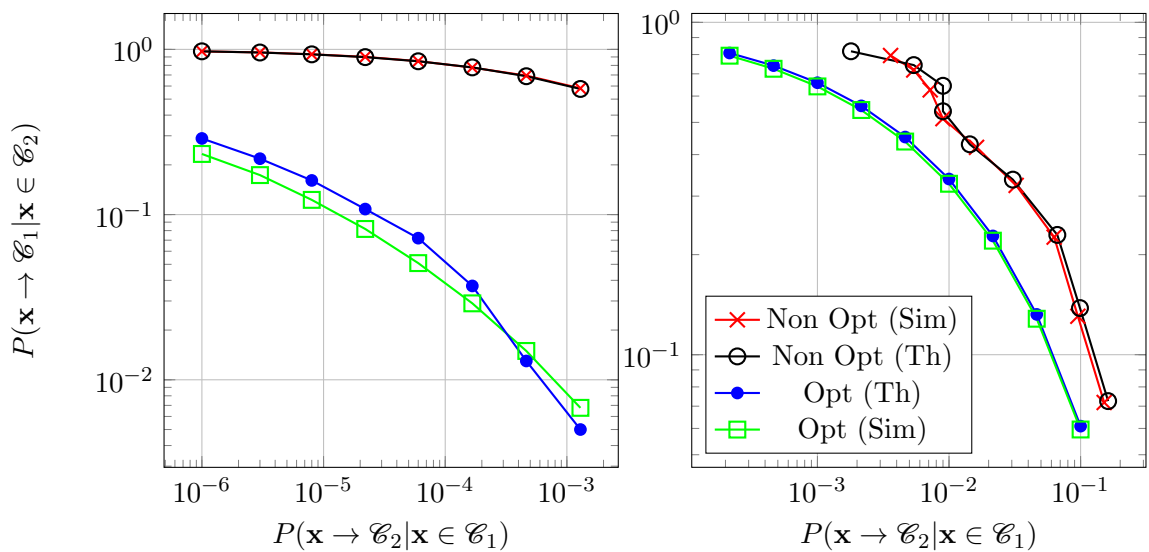

\subsection{Experiments on multi-class classification}
\label{sec:exp_multiclass}

We here consider the complete setting of a $k\geq 2$, $m>2$ multi-class learning scenario, first on synthetic and then on real image datasets. 

\subsubsection{Experiments on synthetic dataset}

In the synthetic data experiment, the scenario is a two-task ($k=2$) setting in which $x_{1l}^{(j)}\sim\mathcal{N}(\mu_{1j},I_p)$ and $x_{2l}^{(j)}\sim\mathcal{N}(\mu_{2j},I_p)$, where $\mu_{2j}=\beta\mu_{1j}+\sqrt{1-\beta^2}\mu_{1j}^{\perp}$, with $\mu_{1j}=2 e_{j}^{[p]}$ and  $\mu_{1j}^{\perp}=e_{p-j}^{[p]}$, and $\beta$ varies from $0.1$ to $0.8$.

Table~\ref{tab:multi_extension} provides the empirical classification accuracy achieved by one-versus-all (Algorithm~\ref{alg:multi class}), one-versus-one (Algorithm~\ref{alg:multi class_one_vs_one}) and one-hot (Algorithm~\ref{alg:multi class_one_hot}) learning versus their standard (non-optimized) algorithm equivalent on $10\,000$ test samples. The table also reports the theoretical classification accuracies predicted by the empirical estimation of the quantities involved in Propositions~\ref{prop:classification_accuracy_one_vs_all}--\ref{prop:classification_accuracy_one_hot} (therefore without any cross validation) for the one-versus-all and one-hot methods.
\begin{table}[h!t]
\caption{Classification accuracy for synthetic data $x_{1l}^{(j)}\sim\mathcal{N}(\mu_{1j},I_p)$ and $x_{2l}^{(j)}\sim\mathcal{N}(\mu_{2j},I_p)$, $\mu_{2j}=\beta\mu_{1j}+\sqrt{1-\beta^2}\mu_{1j}^{\perp}$, for different values of the data-correlation $\beta>0$ and various multi-class learning algorithms. Theoretical performance predictions are provided in parentheses. Here $m=5$, $p=100$, $c_{1j}=.16$, $c_{2j}=.04$, for $j\in\{1,\ldots,5\}$, $\lambda=1$ and $\gamma=\mathbb{1}_k$.
The performance gains of the proposed optimal scheme is particularly clear in tasks with low correlation.
}
\label{tab:multi_extension}

\centering
\hspace*{-.2cm}\begin{tabular}{ccccc}
\hline
$\beta$ &Method & one-vs-all & one-vs-one & one-hot\\
\hline
\multirow{2}{*}{$\beta=0.1$}&Classical& 61.43 (59.87) & 65.31 & 65.61 (64.35)\\
&Optimized &67.63 (67.57)& 74.98 & 67.63 (67.55)\\
\hline
\multirow{2}{*}{$\beta=0.5$}&Classical& 65.47 (66.00) & 71.30 & 67.41 (67.90)\\
&Optimized &68.00 (68.52)& 76.31 & 68.03 (68.48)\\
\hline
\multirow{2}{*}{$\beta=0.8$}&Classical& 71.16 (70.63) & 78.20 & 70.97 (70.58)\\
&Optimized &71.19 (70.76)& 78.55 & 71.14 (70.67)\\
\end{tabular}
\end{table}

The output performance scores naturally show an improvement using the proposed MTL-LSSVM framework and confirm again the extremely accurate prediction of performance by the theoretical formulas. Most importantly, the table reveals that the gap between the non-optimized and optimized schemes is all the more important that the correlation between task (through the parameter $\beta$) is small; this indicates that the optimized MTL-LSSVM learning framework better exploits the (even little) correlation arising between tasks or, alternatively, that the non-optimized scheme suffers from negative learning when ``over-emphasizing'' the weight of data from the other task (through the binary input labels $\mathcal Y$).

As for the comparison of the three classification methods (one-versus-all, one-versus-one and one-hot), it shows here an overall superiority of the one-versus-one approach. This result should nonetheless be interpreted with extreme care as no optimization over the hyperparameters $\gamma,\lambda$ is conducted in any scenario.

\subsubsection{Image classification}
Similarly as in Section~\ref{sec:theoretical_simple}, we now turn to the popular Office+Caltech256 multi-task image classification benchmark \citep{saenko2010adapting,griffin2007caltech} often exploited for transfer learning. The overall database consists of $10$ categories shared by both Office and Caltech256 datasets. As in Table~\ref{tab:compare}, we consider in sequence the transfer learning of one out of four possible source tasks, each of which consisting in classifying data from one sub-database (images issued from the Caltech set (c), Webcam images (w), Amazon pictures (a) or dslr images (d)), towards another task; this boils down to $4\times (4-1)=12$ source-target comparison pairs.)

The results in Table~\ref{tab:compare} using VGG features for the image representations are extremely close to $100\%$, already for the ``naive'' approach consisting in a simplified one-versus-all extension of Algorithm~\ref{alg:binary_algorithm}. Little would be gained (at least not in computational efforts) by running the more involved Algorithm~\ref{alg:multi class} on the same database. For this reason, for the present experiment, we compare the more challenging (since less discriminative) $p=800$ SURF-BoW features of the Office+Caltech256 images instead of their VGG features. 

Half of the samples of the target task is randomly selected as test data and the accuracy is evaluated over $20$ independent trials. For complexity reasons, as in Section~\ref{sec:theoretical_simple}, for each experiment, the naive version of the one-versus-all algorithm is run $10$ times, considering a fictitious two-class $\tilde{\mathcal C}_1$-versus-$\tilde{\mathcal C}_2$ setting where, for the classifier focusing on class $\mathcal C_\ell$, class $\tilde{\mathcal C}_1=\mathcal C_\ell$ while class $\tilde{\mathcal C}_2$ is the union of all other classes $\mathcal C_{\ell'}$, $\ell'\neq \ell$.

\medskip

Table~\ref{tab:compare_2} reports the accuracy obtained by the algorithm (Proposed) versus the non optimized MTL-LSSVM from \citep{xu2013} (LSSVM) and state-of-the-art transfer learning algorithms already introduced in Section~\ref{sec:theoretical_simple}.
Table~\ref{tab:compare_2} again demonstrates that our proposed improved MTL-LSSVM, despite its simplicity and unlike the competing methods used for comparison, has stable performances and is highly competitive.

\begin{table*}[h!t]
\caption{Classification accuracy for transfer learning on the Office+Caltech256 database, against state-of-the-art alternatives. Here with c(Caltech), w(Webcam), a(Amazon), d(dslr) based on SURF-BoW features. Our proposed approach is systematically best or second to best and best on average.}
\label{tab:compare_2}

\centering
\hspace*{-1cm}\begin{tabular}{p{0.074\linewidth}p{0.054\linewidth}p{0.054\linewidth}p{0.054\linewidth}p{0.054\linewidth}p{0.054\linewidth}p{0.054\linewidth}p{0.054\linewidth}p{0.054\linewidth}p{0.054\linewidth}p{0.054\linewidth}p{0.054\linewidth}p{0.054\linewidth}|p{0.054\linewidth}}
\hline
S/T & c$\,\to\,$w & w$\,\to\,$c & c$\,\to\,$a & a$\,\to\,$c& w$\,\to\,$a & a$\,\to\,$d & d$\,\to\,$a & w$\,\to\,$d&c$\,\to\,$d & d$\,\to\,$c & a$\,\to\,$w & d$\,\to\,$w& Mean score\\
\hline
LSSVM &79.47 & 47.70 & 68.10 &49.65 & 68.13 & 57.50 & 70.00 & 73.75&67.50& 46.45&74.83&84.11 & 65.60 \\

MMDT & 69.47 & 42.55 & 68.95 & 39.70 & 65.24 & 59.50 & 62.16 & \bf 86.06 & 56.94 & 27.92 & 68.54& \bf 87.88 & 61.24 \\

ILS & 24.5 & 20.92 & 25.21 & 21.10 & 22.92 & 26.25 & 27.08 & 43.75 & 30.00 & 26.95 & 15.23 & 57.62 & 28.46 \\

CDLS & \it 82.28 & \bf 54.21 & \it 73.75 & \bf 54.49 & \bf 71.52 & \it 68.56 & \it 70.54 & 69.44 & \it 69.44 & \bf 53.86 & \bf 81.59 & 82.78 & {\it 69.37} \\
\hline
Ours & {\bf 86.09} & {\it 49.65} & {\bf 75.00} & {\it 50.35} & {\it 68.83} & {\bf 73.75} & {\bf 71.25} & {\it 72.50} & {\bf 77.50} & {\it 48.05} & {\it 80.13} & {\it 85.43} & {\bf 69.88} \\
\hline
\end{tabular}
\end{table*}

\section{Concluding remarks}

Through the example of multi-task learning, as well as its particularization to transfer learning, the article demonstrates the ability of random matrix theory to predict the performance of advanced machine learning schemes (here based on an extension of LSSVM) and most importantly to propose improved learning mechanisms, which are competitive with, if not largely outperforming, elaborate state-of-the-art alternatives. 

Interestingly, as already reported in recent works \citep{mai2018random,mai2019large}, the proposed random-matrix-optimized framework is largely counter-intuitive and comes along with novel insights on the overall learning mechanisms of large dimensional data classification. Here specifically, the proposed input score (label) optimization is at odds with the conventional binary input label insights of most machine learning schemes, but is key to optimize the exploitation of other tasks and to discard altogether the long standing problem of negative transfer.

The random-matrix framework also draws a significant advantage in its being \emph{universal} to data distributions. As shown here, our main results (Theorem~\ref{th:main}) are valid for data modelled as mixtures of concentrated random vectors which go quite beyond the usually assumed Gaussian mixtures, as they encompass extremely realistic synthetic data models (such as GAN images). This universality phenomenon, possible surprising at first, in fact holds for a wide range of large dimensional ``dense'' (as opposed to sparse) data representation vectors, encompassing not only images but also likely other forms of data representations, such as word embeddings in natural language processing, vectors of moments of graphons in statistical graph analysis, etc.

To conclude, we importantly emphasize a fundamental underlying take-away message of the present work: recalling that LSSVM is nothing but an explicit and computationally-cheap linear regression method, the fact that it competes or even outperforms elaborate MTL methods testifies of the possibility, when dealing with large dimensional data, to design highly performing elementary and cost-efficient random-matrix-based learning schemes. This remark is in line with the recent parallel analysis of information theoretic bounds on the performances of machine learning problems, such as in \citep{lelarge2019asymptotic} for semi-supervised learning (SSL); similar to the present work, in \citep{mai2018random}, the authors propose a random-matrix-based optimization of standard graph SSL learning which they demonstrate to tightly reach the information theoretic upper bound of \citep{lelarge2019asymptotic}. This simultaneously (i) opens the path to a tentative exploration of information-theoretic bounds on transfer learning and multi-task learning for large dimensional data, the results of which could then be confronted to the present proposed scheme, and (ii) strongly suggests the practical relevance of ``reinvesting'' research efforts in simple, cost-efficient, theoretically tractable, controllable, and usually more stable machine learning schemes, rather than in complex and theoretically intractable techniques.

\acks{We thank Cosme Louart for fruitful discussions about technical aspects related to concentrated random vectors. This work is supported by the UGA IDEX GSTATS Chair and the MIAI LargeDATA Chair at University Grenoble Alpes.}


\newpage

\appendix
\section{}
\subsection{Solution of MTL-LSSVM}
\label{MTL_solution}

The Lagrangian of the constrained optimization problem  using the relatedness assumption ($W_i=W_0+V_i$)  reads:
\begin{align*}
    \mathcal{L}(\omega_0,v_i,\xi_i,\alpha_i,b_i)&=\frac1{2\lambda} \tr\left(W_0^\trans W_0\right)+\frac 1{2}\sum_{i=1}^k \frac{\tr\left( V_i^\trans V_i\right)}{\gamma_i}+\frac 12\sum_{i=1}^k\tr\left(\xi_i^\trans\xi_i\right)\\
    &+\sum\limits_{i=1}^k\tr\left(\alpha_i^\trans\left(Y_i-\frac{\mathring{X}_i^\trans W_0}{kp}-\frac{\mathring{X}_i^\trans V_i}{kp}-\mathbb{1}_{n_i}b_i^\trans-\xi_i\right)\right)
\end{align*}
with $\alpha_i\in\mathbb{R}^{n_i\times m}$ the Lagrangian parameter attached to task $i$.

Differentiating with respect to the unknowns $W_0$, $V_i$, $\xi_i$, $\alpha_i$, and $b_i$ leads to the following system of equations:
\begin{align}
    \frac 1{\lambda}W_0-\sum\limits_{i=1}^k X_i\alpha_i&=0 \label{eq_w0}\\
    \frac 1{\gamma_i}V_i -X_i\alpha_i&=0 \label{eq_vi}\\
    \xi_i-\alpha_i&=0 \label{eq_xi}\\
    Y_i-\frac{\mathring{X}_i^\trans W_0}{kp}-\frac{\mathring{X}_i^\trans V_i}{kp}-\mathbb{1}_{n_i}b_i^\trans-\xi_i&=0 \label{eq_yi}\\
    \alpha_i^\trans\mathbb{1}_{n_i}&=0.
\end{align}

Plugging the expression of $W_0$ (Equation~\eqref{eq_w0}), $V_i$ (Equation~\eqref{eq_vi}) and $\xi_i$ (Equation~\eqref{eq_xi}) into Equation~\eqref{eq_yi} leads to:

\begin{align*}
Y_i&=\left(\lambda+\gamma_i\right)\frac{\mathring{X}_i^\trans \mathring{X}_i}{kp}\alpha_i+\lambda\sum\limits_{j\neq i}\frac{\mathring{X}_i^\trans X_j}{kp}\alpha_j+\mathbb{1}_{n_i}b_i^\trans+\alpha_i\\
\mathbb{1}_{n_i}^\trans\alpha_i&=0.
\end{align*}

With $Y=[Y_1^\trans,\ldots,Y_k^\trans]^\trans\in\mathbb{R}^{n}$, $\alpha=[\alpha_1^\trans,\ldots,\alpha_k^\trans]^\trans\in\mathbb{R}^{n}$, ${Z=\sum_{i=1}^k e_{i}^{[k]}{e_{i}^{[k]}}^\trans\otimes \mathring{X}_i}\in\mathbb{R}^{kp\times n}$  and ${P\in\mathbb{R}^{n\times k}}$ such that the $j$-th column is 
${P_{.j}=[\textbf{0}_{n_1+\ldots+n_{j-1}}^\trans,\mathbb{1}_{n_j}^\trans,\textbf{0}_{n_{j+1}+\ldots+n_{k}}^\trans]^\trans}$, this system of equations can be written under the following compact matrix form:
\begin{align*}
    Pb+Q^{-1}\alpha &=Y\\
    P^\trans \alpha&=\mathbf{0}_{k}
\end{align*}
with $Q=\left(\frac{Z^\trans AZ}{kp}+I_{n}\right)^{-1}\in\mathbb{R}^{n\times n}$, and $A=\left(\mathcal{D}_{\gamma}+\lambda\mathbb{1}_k\mathbb{1}_k^\trans\right)\otimes I_p\in\mathbb{R}^{kp\times kp}$.

Solving for $\alpha$ and $b$ then gives:
\begin{align*}
    \alpha &= Q(Y-Pb)\\
     b&=(P^\trans Q P)^{-1}P^{\trans}QY.
\end{align*}
Moreover, using $W_i=W_0+V_i$ and Equations~\eqref{eq_w0} and \eqref{eq_vi}, the expression of $W_i$ becomes:
\begin{equation*}
W_i=\left({e_{i}^{[k]}}^\trans\otimes I_{p}\right)AZ\alpha.
\end{equation*}

\subsection{Calculus of deterministic equivalents}
\label{sec:lemma}
\begin{lemma}[Deterministic equivalents]
\label{lem:eq}
\begin{sloppypar}
Define, for class $\mathcal C_j$ in Task~$i$, the data deterministic matrices
\begin{align*}
    M&=\left(e_1^{[k]}\otimes[\mu_{11},\dots,\mu_{1m}],\ldots,e_k^{[k]}\otimes[\mu_{k1},\dots, \mu_{km}]\right) \\
 \mathbb{C}_{ij}&=A^{\frac 12}\left(e_i^{[k]}{e_i^{[k]}}^\trans\otimes(\Sigma_{ij}+\mu_{ij}\mu_{ij}^\trans) \right)A^{\frac 12}.
\end{align*}
Then we have the deterministic equivalents of first order
\begin{align*}
     \tilde{Q} &\leftrightarrow \bar{\tilde{Q}} \equiv  \left(\sum_{i=1}^{k}\sum\limits_{j=1}^{m}{\bm\delta}_{ij}^{[mk]}\mathbb{C}_{ij}+I_{kp}\right)^{-1}  \\
     A^{\frac 12}\tilde{Q}A^{\frac 12}Z &\leftrightarrow A^{\frac 12}\Bar{\Tilde{Q}}A^{\frac 12}M_{\bm\delta}J^{\trans}
\end{align*}
and of second order
\begin{align*}
     \tilde{Q}A^{\frac 12}S_{ij}A^{\frac 12}\tilde{Q}&\leftrightarrow B_{ij} \\
      Z^\trans A^{\frac 12}\tilde{Q}A^{\frac 12}S_{ij}A^{\frac 12}\tilde{Q}A^{\frac 12}Z&\leftrightarrow JM_{\bm\delta}^\trans A^{\frac 12}(B_{ij}A^{\frac 12}M_{\bm\delta}J^\trans-\bar{\tilde{Q}}A^{\frac 12}M_{\bm\delta}W_{ij})+F_{ij}
\end{align*}
in which we defined 
\begin{align*}
W_{ij}&=[w_{11},\ldots,w_{km}]^\trans, \quad w_{sl}=\left[\textbf{0}_{n_{11}+\ldots+n_{(s-1)l}}^\trans,\frac{2{\bm\delta_{sl}^{[mk]}}\tr\left( B_{ij}\mathbb{C}_{sl}\right)}{n_{sl}}\mathbb{1}_{n_{sl}}^\trans,\textbf{0}_{n_{(s+1)l}+\ldots+n_{km}}^\trans\right]^\trans\\
F_{ij}&=\sum_{i',j'} \frac{c_0^2{{}{\bm\delta}_{i'j'}^{[mk]}}^2}{c_{i'j'}^2}\tr(\mathbb{C}_{i'j'}B_{ij})e_{i'j'}^{[mk]}{e_{i'j'}^{[mk]}}^\trans\\
B_{ij} &=\bar{\tilde{Q}}A^{\frac 12}S_{ij}A^{\frac 12}\bar{\tilde{Q}}+\sum\limits_{i'=1}^k\sum\limits_{j'=1}^2 d_{i'j'} T_{ij,i'j'}[\bar{\tilde{Q}}\mathbb{C}_{i'j'}\bar{\tilde{Q}}] \\
D &=\sum_{i,j}d_{ij}e_{ij}^{[mk]}{e_{ij}^{[mk]}}^\trans,~d_{ij}=\frac{c_{0}}{c_{ij}}{{}{\bm\delta}_{ij}^{[mk]}}^2 \\
J &=[j_{11},\ldots,j_{km}],\\
j_{lm}&=\left(0_{n_{11}+\ldots+n_{(i-1)m}}^\trans,\mathbb{1}_{n_{ij}}^\trans,0_{n_{(i+1)1}+\ldots+n_{km}}^\trans\right)^{\trans},\\
M_{\bm\delta} &= M \sum_{ij} \frac{c_0}{c_{ij}}{\bm\delta}_{ij}^{[mk]}e_{ij}^{[mk]}e_{ij}^{[mk]\trans} \\
S_{ij} &=e_i^{[k]}{e_i^{[k]}}^\trans\otimes \Sigma_{ij} \\
T&=\bar{T}(I_{k}-D\mathcal{T})^{-1},~\mathcal{T}_{ij,i'j'}=\frac1{kp}\tr(\mathbb{C}_{ij}\bar{\tilde{Q}}\mathbb{C}_{i'j'}\bar{\tilde{Q}}), \bar{T}_{ij,i'j'}=\frac{1}{kp}tr\left(\mathbb{C}_{i'j'}\bar{\tilde{Q}}A^{\frac 12}S_{ij}A^{\frac 12}\bar{\tilde{Q}}\right)
\end{align*}
and the $(\bm\delta_{11}^{[mk]},\ldots,\bm\delta_{km}^{[mk]})$ are the unique positive solutions of
\begin{equation*}
    \bm\delta_{ij}^{[mk]}=\frac{c_{ij}}{c_0\left(1+\frac 1{kp} \tr (\mathbb{C}_{ij}\bar{\tilde{Q}})\right)},~\forall i,j.
\end{equation*}
\end{sloppypar}
\end{lemma}
\subsubsection{Proof of Lemma~\ref{lem:eq}}
\paragraph{First order deterministic equivalent.}
A deterministic equivalent for $\tilde Q$ is retrieved similarly as provided in \citep{louart2018concentration}.
Our objective is then to find, based on this result, a deterministic equivalent for the random matrix $ A^{\frac 12}\tilde{Q}A^{\frac 12}Z$. To this end, we evaluate the scalar quantity $\mathbb{E}[u^\trans A^{\frac 12}\tilde{Q}A^{\frac 12}Z v]$ for any deterministic vector $u\in\mathbb{R}^{kp}$ and $v\in\mathbb{R}^{n}$ such that $\|u\|=1$ and $\|v\|=1$, which we can write
\begin{align}
\label{eq:mean}
\mathbb{E}\left[u^\trans A^{\frac 12}\tilde{Q}A^{\frac 12}Z v\right]=\sum_{i=1}^n v_i\mathbb{E}\left[u^\trans A^{\frac 12}\tilde{Q}A^{\frac 12}z_i \right].
\end{align}

Furthermore, let us define for convenience the matrix $Z_{-i}$, which is the matrix $Z$ with a vector of $\textbf{0}$ on its $i$-th column such that ${ZZ^\trans=Z_{-i}Z_{-i}^\trans+z_{i}z_{i}^\trans}$.
Using the Sherman-Morrison matrix inversion lemma (i.e., ${(A+uv^\trans)^{-1}=A^{-1}-\frac{A^{-1}uv^\trans A^{-1}}{1+v^\trans A^{-1}u}}$), we find:
\begin{align}
    \tilde Q&=\left(\frac{A^{\frac 12}ZZ^{\trans}A^{\frac 12}}{kp}+I_{kp}\right)^{-1}=\tilde{Q}_{-i}-\frac{1}{kp}\frac{\tilde{Q}_{-i}A^{\frac 12}z_iz_i^\trans A^{\frac 12}\tilde{Q}_{-i}}{1+\frac{1}{kp}z_i^\trans A^{\frac 12}\tilde{Q}_{-i}A^{\frac 12}z_i}\label{eq:shermann1}
\end{align}
with $\tilde{Q}_{-i}=(\frac{A^{\frac 12}Z_{-i}Z_{-i}^{\trans}A^{\frac 12}}{kp}+I_{kp})^{-1}$.
Furthermore,
\begin{align}
\label{eq:shermann2}
    \tilde{Q}A^{\frac 12}z_i&=\frac{\tilde{Q}_{-i}A^{\frac 12}z_i}{1+\frac{1}{kp}z_i^\trans A^{\frac 12}\tilde{Q}_{-i}A^{\frac 12}z_i}.
\end{align}

 Plugging Equation~\eqref{eq:shermann2} into Equation~\eqref{eq:mean} leads to 
 \begin{align}
\mathbb{E}\left[u^\trans A^{\frac 12}\tilde{Q}A^{\frac 12}Z v\right]=\sum_{i=1}^n v_i\mathbb{E}\left[u^\trans \frac{A^{\frac 12}\tilde{Q}_{-i}A^{\frac 12}z_i}{1+\frac{1}{kp}z_i^\trans A^{\frac 12}\tilde{Q}_{-i}A^{\frac 12}z_i} \right].
\end{align}
Moreover, following the same line of reasoning as in \citep[Proposition~A.3]{seddik2020random}, based on Assumption~\ref{ass:distribution} and tools from concentration of measure theory (see also \citep{ledoux2001concentration,louart2018random}), one can show that:
\begin{align}
\label{eq:approximation_delta}
    \sum_{i=1}^n v_i\mathbb{E}\left[u^\trans \frac{A^{\frac 12}\tilde{Q}_{-i}A^{\frac 12}z_i}{1+\frac{1}{kp}z_i^\trans A^{\frac 12}\tilde{Q}_{-i}A^{\frac 12}z_i} \right]=\sum_{i=1}^n v_i\mathbb{E}\left[u^\trans \frac{A^{\frac 12}\tilde{Q}_{-i}A^{\frac 12}z_i}{1+\bm\delta_{ij}} \right]+\mathcal{O}\left(\sqrt{\frac{\log p}{ p}}\right)
\end{align}
with $\bm\delta_{ij}\equiv\mathbb{E}\left[\frac{1}{kp}z_i^\trans A^{\frac 12}\tilde{Q}_{-i}A^{\frac 12}z_i\right]$.
Note that $\bm\delta_{ij}$ can be estimated as the solution of the fixed point equation 
\begin{align*}
    \bm\delta_{ij}=\frac{1}{kp}\mathbb{E}\left[\tr\left(A^{\frac 12}z_iz_i^\trans A^{\frac 12}\tilde{Q}_{-i}\right)\right]=\frac{1}{kp}\tr\left(\mathbb{C}_{ij}\bar{\tilde{Q}}\right)+\mathcal{O}\left(\frac{1}{\sqrt{p}}\right)
\end{align*}
since $z_i$'s are independent from $\tilde{Q}_{-i}$.

We then conclude that: 
 \begin{align*}
\mathbb{E}\left[u^\trans A^{\frac 12}\tilde{Q}A^{\frac 12}Z v\right]
=\sum_{i=1}^n v_iu^\trans \frac{\mathbb{E}\left[A^{\frac 12}\tilde{Q}_{-i}A^{\frac 12}z_i\right]}{1+\bm\delta_{ij}}+\mathcal{O}\left(\sqrt{\frac{\log p}{p}}\right)=u^\trans A^{\frac 12}\bar{\tilde{Q}}A^{\frac 12}M_{\bm\delta}v+\mathcal{O}\left(\sqrt{\frac{\log p}{p}}\right)
\end{align*}
where in the last equality, we used the fact that $\tilde{Q}_{-i}$ is independent from $z_i$.
This concludes the proof.

\paragraph{Second order deterministic equivalent}
We aim in the following section to prove that $Z^\trans A^{\frac 12}\tilde{Q}A^{\frac 12}S_{ij}A^{\frac 12}\tilde{Q}A^{\frac 12}Z\leftrightarrow JM_{\bm\delta}^\trans A^{\frac 12}(B_{ij}A^{\frac 12}M_{\bm\delta}J^\trans-\bar{\tilde{Q}}A^{\frac 12}M_{\bm\delta}W_{ij})+F_{ij}$.

Let us define for convenience $\mathcal{C}(i)$ the class of the $i$-th sample. Similarly as done for the first order deterministic equivalents, the focus will be on $\mathbb{E}[u^\trans Z^\trans A^{\frac 12}\tilde{Q}A^{\frac 12}S_{ij}A^{\frac 12}\tilde{Q}A^{\frac 12}Z v]$.

In order to obtain an estimate of this bilinear form, or equivalently here a deterministic equivalent for $Z^\trans A^{\frac 12}\tilde{Q}A^{\frac 12}S_{ij}A^{\frac 12}\tilde{Q}A^{\frac 12}Z$, one must isolate the contribution the off-diagonal versus diagonal elements of the latter matrix. Starting with the off-diagonal elements, using successively Equation~\eqref{eq:shermann1} and Equation~\eqref{eq:approximation_delta} on $i$ and $j$ , we have
\begin{align*}
    &\sum_{\substack{i,j=1\\i\neq j}}^n u_i v_i\mathbb{E}\left[z_{i}^\trans A^{\frac 12}\tilde{Q}A^{\frac 12}S_{ij}A^{\frac 12}\tilde{Q}A^{\frac 12}z_j\right]\nonumber \\
    &=\sum_{\substack{i,j=1\\i\neq j}}^n u_i v_i\mathbb{E}\left[\frac{z_{i}^\trans A^{\frac 12}\tilde{Q}_{-i}A^{\frac 12}S_{ij}A^{\frac 12}\tilde{Q}_{-j}A^{\frac 12}z_j}{(1+\bm\delta_{\mathcal{C}(i)})(1+\bm\delta_{\mathcal{C}(j)})}\right]+\mathcal{O}\left(\sqrt{\frac{\log p}{p}}\right) \\
    &=\sum_{\substack{i,j=1\\i\neq j}}^n u_i v_i\mathbb{E}\left[\frac{z_{i}^\trans A^{\frac 12}\tilde{Q}_{\substack{-i\\-j}}A^{\frac 12}S_{ij}A^{\frac 12}\tilde{Q}_{\substack{-j\\-i}}A^{\frac 12}z_j}{(1+\bm\delta_{\mathcal{C}(i)})(1+\bm\delta_{\mathcal{C}(j)})}-\frac{z_{i}^\trans A^{\frac 12}\tilde{Q}_{\substack{-i\\-j}}A^{\frac 12}S_{ij}A^{\frac 12}\tilde{Q}_{\substack{-j\\-i}}A^{\frac 12}z_iz_i^\trans A^{\frac 12}\tilde{Q}_{-j}A^{\frac 12}z_j}{kp(1+\bm\delta_{\mathcal{C}(i)})(1+\bm\delta_{\mathcal{C}(j)})}\right.\\
    &-\left.\frac{z_{i}^\trans A^{\frac 12}\tilde{Q}_{\substack{-i\\-j}}A^{\frac 12}z_jz_j^\trans A^{\frac 12}\tilde{Q}_{-i}A^{\frac 12}S_{ij}A^{\frac 12}\tilde{Q}_{-j}A^{\frac 12}z_j}{kp(1+\bm\delta_{\mathcal{C}(i)})(1+\bm\delta_{\mathcal{C}(j)})}\right]+\mathcal{O}\left(\sqrt{\frac{\log p}{p}}\right)\\
     &=\sum_{\substack{i,j=1\\i\neq j}}^n u_i v_i\mathbb{E}\left[\frac{z_{i}^\trans A^{\frac 12}\tilde{Q}_{\substack{-i\\-j}}A^{\frac 12}S_{ij}A^{\frac 12}\tilde{Q}_{\substack{-j\\-i}}A^{\frac 12}z_j}{(1+\bm\delta_{\mathcal{C}(i)})(1+\bm\delta_{\mathcal{C}(j)})}-\frac{z_{i}^\trans A^{\frac 12}\tilde{Q}_{\substack{-i\\-j}}A^{\frac 12}S_{ij}A^{\frac 12}\tilde{Q}_{\substack{-j\\-i}}A^{\frac 12}z_iz_i^\trans A^{\frac 12}\tilde{Q}_{\substack{-j\\-i}}A^{\frac 12}z_j}{kp(1+\bm\delta_{\mathcal{C}(i)})(1+\bm\delta_{\mathcal{C}(j)})(1+\bm\delta_{\mathcal{C}(i)})}\right.\\
     &-\left.\frac{z_{i}^\trans A^{\frac 12}\tilde{Q}_{\substack{-i\\-j}}A^{\frac 12}z_jz_j^\trans A^{\frac 12}\tilde{Q}_{\substack{-i\\-j}}A^{\frac 12}S_{ij}A^{\frac 12}\tilde{Q}_{\substack{-j\\-i}}A^{\frac 12}z_j}{kp(1+\bm\delta_{\mathcal{C}(i)})(1+\bm\delta_{\mathcal{C}(j)})(1+\bm\delta_{\mathcal{C}(j)})}\right.\\
     &+\left.\frac{1}{(kp)^2}\frac{z_{i}^\trans A^{\frac 12}\tilde{Q}_{\substack{-i\\-j}}A^{\frac 12}z_jz_j^\trans A^{\frac 12}\tilde{Q}_{-i}A^{\frac 12}S_{ij}A^{\frac 12}\tilde{Q}_{\substack{-j\\-i}}A^{\frac 12}z_iz_i^\trans A^{\frac 12}\tilde{Q}_{\substack{-j\\-i}}A^{\frac 12}z_j}{(1+\bm\delta_{\mathcal{C}(i)})(1+\bm\delta_{\mathcal{C}(j)})(1+\bm\delta_{\mathcal{C}(i)})}\right]+\mathcal{O}\left(\sqrt{\frac{\log p}{p}}\right)\\
     &=\sum_{\substack{i,j=1\\i\neq j}}^n u_i v_i\mathbb{E}\left[\frac{z_{i}^\trans A^{\frac 12}\tilde{Q}_{\substack{-i\\-j}}A^{\frac 12}S_{ij}A^{\frac 12}\tilde{Q}_{\substack{-j\\-i}}A^{\frac 12}z_j}{(1+\bm\delta_{\mathcal{C}(i)})(1+\bm\delta_{\mathcal{C}(j)})}-\frac{z_{i}^\trans A^{\frac 12}\tilde{Q}_{\substack{-i\\-j}}A^{\frac 12}S_{ij}A^{\frac 12}\tilde{Q}_{\substack{-j\\-i}}A^{\frac 12}z_iz_i^\trans A^{\frac 12}\tilde{Q}_{\substack{-j\\-i}}A^{\frac 12}z_j}{kp(1+\bm\delta_{\mathcal{C}(i)})(1+\bm\delta_{\mathcal{C}(j)})(1+\bm\delta_{\mathcal{C}(i)})}\right.\\
     &-\left.\frac{z_{i}^\trans A^{\frac 12}\tilde{Q}_{\substack{-i\\-j}}A^{\frac 12}z_jz_j^\trans A^{\frac 12}\tilde{Q}_{\substack{-i\\-j}}A^{\frac 12}S_{ij}A^{\frac 12}\tilde{Q}_{\substack{-j\\-i}}A^{\frac 12}z_j}{kp(1+\bm\delta_{\mathcal{C}(i)})(1+\bm\delta_{\mathcal{C}(j)})(1+\bm\delta_{\mathcal{C}(j)})}\right]+\mathcal{O}\left(\sqrt{\frac{\log p}{p}}\right).
\end{align*}
where the term $$\frac{1}{(kp)^2}\frac{z_{i}^\trans A^{\frac 12}\tilde{Q}_{\substack{-i\\-j}}A^{\frac 12}z_jz_j^\trans A^{\frac 12}\tilde{Q}_{-i}A^{\frac 12}S_{ij}A^{\frac 12}\tilde{Q}_{\substack{-j\\-i}}A^{\frac 12}z_iz_i^\trans A^{\frac 12}\tilde{Q}_{\substack{-j\\-i}}A^{\frac 12}z_j}{(1+\bm\delta_{\mathcal{C}(i)})(1+\bm\delta_{\mathcal{C}(j)})(1+\bm\delta_{\mathcal{C}(i)})}$$ 
is proved to be order $\mathcal{O}(\frac{1}{\sqrt{p}})$ using \citep[Lemma~A.2]{seddik2020random}.

As such, the ``sub-deterministic equivalent'' for the matrix $Z^\trans A^{\frac 12}\tilde{Q}A^{\frac 12}S_{ij}A^{\frac 12}\tilde{Q}A^{\frac 12}Z$ with diagonal elements discarded is
$JM_{\bm\delta}^\trans A^{\frac 12}B_{ij}A^{\frac 12}M_{\bm\delta}J^\trans-JM_{\bm\delta}^\trans A^{\frac 12}\bar{\tilde{Q}}A^{\frac 12}M_{\bm\delta}W_{ij}$,
with
\begin{align*}
    &A^{\frac 12}\tilde{Q}A^{\frac 12}S_{ij}A^{\frac 12}\tilde{Q}A^{\frac 12}\leftrightarrow B_{ij} \\
    &W_{ij}=[w_{11},\ldots,w_{km}]^\trans,\quad w_{sl}=\left[\textbf{0}_{n_{11}+\ldots+n_{(s-1)l}}^\trans,\frac{2\tr\left( B_{ij}\mathbb{C}_{sl}\right)}{kp(1+\bm\delta_{sl})}\mathbb{1}_{n_{sl}}^\trans,\textbf{0}_{n_{(s+1)l}+\ldots+n_{km}}^\trans\right]^\trans
\end{align*}
(note that this matrix estimator of the off-diagonal elements is not zero on the diagonal; however its diagonal elements vanish as $n,p\to\infty$ and may thus be maintained without affecting the final result).

We next need to handle the contribution of the diagonal elements. These are obtained similarly as the off-diagonal elements and lead to the deterministic diagonal matrix equivalent
\begin{align*}
    F_{ij}=\sum_{i',j'} \frac{\tr(\mathbb{C}_{i'j'}B_{ij})}{(1+\bm\delta_{i'j'})^2}e_{i'j'}^{[mk]}{e_{i'j'}^{[mk]}}^\trans.
\end{align*}

Put together, the complete deterministic equivalent is then:
$$JM_{\bm\delta}^\trans A^{\frac 12}B_{ij}A^{\frac 12}M_{\bm\delta}J^\trans-JM_{\bm\delta}^\trans A^{\frac 12}\bar{\tilde{Q}}A^{\frac 12}M_{\bm\delta}W_{ij}+\sum_{i',j'} \frac{\tr(\mathbb{C}_{i'j'}B_{ij})}{(1+\bm\delta_{i'j'})^2}e_{i'j'}^{[mk]}{e_{i'j'}^{[mk]}}^\trans.$$
This proves that  $Z^\trans A^{\frac 12}\tilde{Q}A^{\frac 12}S_{ij}A^{\frac 12}\tilde{Q}A^{\frac 12}Z\leftrightarrow JM_{\bm\delta}^\trans A^{\frac 12}(B_{ij}A^{\frac 12}M_{\bm\delta}J^\trans-\bar{\tilde{Q}}A^{\frac 12}M_{\bm\delta}W_{ij})+F_{ij}$. 

\paragraph{Calculus of $B_{ij}$.}
To conclude the proof of Lemma~\ref{lem:eq}, it then remains to find a deterministic equivalent for $\tilde{Q}A^{\frac 12}S_{ij}A^{\frac 12}\tilde{Q}$ which we denote by $B_{ij}$.
Similar derivations and results are provided in detail in \citep{louart2018random}. For conciseness, we sketch the most important elements of the proof. The interested reader can refer to \citep[Section~5.2.3]{louart2018random}.
Let us evaluate $\mathbb{E}[u^\trans\tilde{Q}A^{\frac 12}S_{ij}A^{\frac 12}(\tilde{Q}-\bar{\tilde{Q}})v]$ for any deterministic vector $u\in\mathbb{R}^{n}$ and $v\in\mathbb{R}^{n}$ such that $\|u\|=1$ and $\|v\|=1$ by using successively Equations~\eqref{eq:approximation_delta}~and~\eqref{eq:shermann1}:
\begin{align*}
\mathbb{E}\left[u^\trans\tilde{Q}A^{\frac 12}S_{ij}A^{\frac 12}(\tilde{Q}-\bar{\tilde{Q}})v\right]&=\mathbb{E}\left[u^\trans\tilde{Q}A^{\frac 12}S_{ij}A^{\frac 12}\tilde{Q}(-\frac{A^{\frac 12}ZZ^\trans A^{\frac 12}}{kp}+\mathbb{C}_{\bm\delta})\bar{\tilde{Q}}v\right]\\
&=-\frac{1}{kp}\sum\limits_{i}\mathbb{E}\left[\frac{u^\trans \tilde{Q}A^{\frac 12}S_{ij}A^{\frac 12}\tilde{Q}_{-i}A^{\frac 12}z_iz_i^\trans A^{\frac 12}\bar{\tilde{Q}}v}{1+\bm\delta_{ij}}\right]\\
&+\mathbb{E}\left[u^\trans \tilde{Q}A^{\frac 12}S_{ij}A^{\frac 12}\tilde{Q}_{-i}\mathbb{C}_{\bm\delta}\bar{\tilde{Q}}v\right]\\
&-\frac{1}{kp}\mathbb{E}\left[u^\trans \tilde{Q}A^{\frac 12}S_{ij}A^{\frac 12}\tilde{Q}_{-i}A^{\frac 12}z_iz_i^\trans A^{\frac 12} \tilde{Q}\mathbb{C}_{\bm\delta}\bar{\tilde{Q}}v\right]+\mathcal{O}\left(\sqrt{\frac{\log p}{p}}\right)
\end{align*}
where $\mathbb{C}_{\bm\delta}=\sum\limits_{ij}\frac{c_{ij}}{c_0}\frac{\mathbb{C}_{ij}}{1+\bm\delta_{ij}}$.
Using Assumption~\ref{ass:distribution} and following the work of \cite{louart2018concentration}, 
$${\frac{1}{kp}\mathbb{E}\left[u^\trans \tilde{Q}A^{\frac 12}S_{ij}A^{\frac 12}\tilde{Q}_{-i}A^{\frac 12}z_iz_i^\trans  A^{\frac 12}\tilde{Q}\mathbb{C}_{\bm\delta}\bar{\tilde{Q}}v\right]=\mathcal{O}(\frac{1}{p})}.$$ 
Furthermore,
\begin{align*}
\mathbb{E}\left[u^\trans\tilde{Q}A^{\frac 12}S_{ij}A^{\frac 12}(\tilde{Q}-\bar{\tilde{Q}})v\right]&=-\frac{1}{kp}\sum\limits_{i}\mathbb{E}\left[\frac{u^\trans \tilde{Q}_{-i}A^{\frac 12}S_{ij}A^{\frac 12}\tilde{Q}_{-i}A^{\frac 12}z_iz_i^\trans A^{\frac 12}\bar{\tilde{Q}}v}{1+\bm\delta_{ij}}\right]\\
&+\frac{1}{kp}\sum\limits_{i}\mathbb{E}\left[\frac{u^\trans \tilde{Q}_{-i}A^{\frac 12} z_iz_i^\trans A^{\frac 12} \tilde{Q}_{-i}A^{\frac 12}S_{ij}A^{\frac 12}\tilde{Q}_{-i}A^{\frac 12} z_iz_i^\trans A^{\frac 12}\bar{\tilde{Q}}v}{kp(1+\bm\delta_{ij})^2}\right]\\
&+\mathbb{E}\left[u^\trans \tilde{Q}A^{\frac 12}S_{ij}A^{\frac 12}Q_{-i}\mathbb{C}_{\bm\delta}\bar{\tilde{Q}}v\right]+\mathcal{O}\left(\sqrt{\frac{\log p}{p}}\right)\\
&=\frac{1}{kp}\sum\limits_{i}\mathbb{E}\frac{\tr\left(\mathbb{C}_{\mathcal{C}(i)}\tilde{Q}A^{\frac 12}S_{ij}A^{\frac 12}\tilde{Q}\right)}{(1+\bm\delta_{\mathcal{C}(i)})^2}\mathbb{E}\left[u^\trans \bar{\tilde{Q}}\mathbb{C}_{\mathcal{C}(i)}\bar{\tilde{Q}}v\right]+\mathcal{O}\left(\sqrt{\frac{\log p}{p}}\right)
\end{align*}
where $-\frac{1}{kp}\sum_{i}\mathbb{E}[\frac{u^\trans \tilde{Q}_{-i}A^{\frac 12}S_{ij}A^{\frac 12}\tilde{Q}_{-i}z_iz_i^\trans\bar{\tilde{Q}}v}{1+\bm\delta_{ij}}]+\mathbb{E}[u^\trans \tilde{Q}_{-i}A^{\frac 12}S_{ij}A^{\frac 12}Q\mathbb{C}_{\bm\delta}\bar{\tilde{Q}}v]=\mathcal{O}\left(\frac{1}{\sqrt{p}}\right)$, following again \citep{louart2018concentration}.

Let us next denote $d_{ab}=\frac{n_{ab}}{kp(1+\bm\delta_{ab})^2}$. We then have the following identity for $\mathbb{E}[\tilde{Q}A^{\frac 12}S_{ij}A^{\frac 12}\tilde{Q}]$:
\begin{equation}
\label{eq:QAQ}
         \mathbb{E}[\tilde{Q}A^{\frac 12}S_{ij}A^{\frac 12}\tilde{Q}]=\bar{\tilde{Q}}A^{\frac 12}S_{ij}A^{\frac 12}\bar{\tilde{Q}}+\sum\limits_{i'=1}^{k}\sum\limits_{j'=1}^m \frac{d_{i'j'}}{kp} \mathbb{E}\left[\tr\left(\mathbb{C}_{i'j'}\tilde{Q}A^{\frac 12}S_{ij}A^{\frac 12}\tilde{Q}\right)\right]\bar{\tilde{Q}}\mathbb{C}_{i'j'}\bar{\tilde{Q}}+\mathcal{O}_{\|\cdot\|}\left(\sqrt{\frac{\log p}{p}}\right)
\end{equation}
Further introduce the two matrices $\bar{T}$ and $T$ defined as:
$\bar{T}_{ab,ij}=\frac{1}{kp}\tr(\mathbb{C}_{ab}\bar{\tilde{Q}}A^{\frac 12}S_{ij}A^{\frac 12}\bar{\tilde{Q}})$ and ${T_{ij,i'j'}=\frac{1}{kp}\mathbb{E}[\tr\left(\mathbb{C}_{i'j'}\tilde{Q}A^{\frac 12}S_{ij}A^{\frac 12}\tilde{Q}\right)]}$.
These satisfy the following equations (i.e., by right multiplying Equation~\eqref{eq:QAQ} by $\mathbb{C}_{i'j'}$ and taking the trace)
$$T_{i'j'}^{(ij)}=\bar{T}_{ij,i'j'}+\sum\limits_{e=1}^k\sum\limits_{f=1}^m d_{ef}T_{ef,ij}\mathcal{T}_{i'j',ef},$$ so that
$T=\bar{T}(I_{k}-D\mathcal{T})^{-1}$ where $D=\mathcal{D}_{[d_{11},\ldots,d_{km}]^\trans}$ and $\mathcal{T}_{ef,i'j'}=\frac{1}{kp}\tr(\mathbb{C}_{ef}\bar{\tilde{Q}}\mathbb{C}_{i'j'}\bar{\tilde{Q}})$. 

Finally,
\begin{equation}
\label{eq:equation_kappa}
\tilde{Q}A^{\frac 12}S_{ij}A^{\frac 12}\tilde{Q}\leftrightarrow\bar{\tilde{Q}}A^{\frac 12}S_{ij}A^{\frac 12}\bar{\tilde{Q}}+\sum\limits_{i'=1}^k\sum\limits_{j'=1}^m d_{i'j'} T_{i'j'}^{(ij)}\mathbb{E}[\bar{\tilde{Q}}\mathbb{C}_{i'j'}\bar{\tilde{Q}}]
\end{equation}
with $T=\bar{T}(I_{k}-D\mathcal{T})^{-1}$.

\subsection{Proof of Theorem~\ref{th:main}}
\label{app:th}
\paragraph{Proof of the convergence in distribution.}

Under a Gaussian mixture assumption for the input data $X$, the convergence in distribution of the statistics of the classification score $g_{i}({\bf x})$ is identical to the central limit theorem derived in \citep[Appendix~B]{liao2019large} by writing the classification score $g_i({\bf x})$ in polynomial form of a Gaussian vector and by resorting to the Lyapounov central limit theorem \citep{billingsley2008probability}.

Since conditionally on  the training data $X$, the classification score $g(x)$ is expressed as the projection of the deterministic vector $W$ on the concentrated random vector $\bf x$, the CLT for concentrated vector unfolds by proving that projections of deterministic vector on concentrated random vector is asymptotically gaussian. This is ensured by the following result.
\begin{theorem}[CLT for concentrated vector \citep{klartag2007central,fleury2007stability}]
\label{th:CLT_concentrated}
If ${\bf x}$ is a concentrated random vector with $\mathbb{E}[{\bf x}]=0$, $\mathbb{E}[{\bf x}{\bf x}^\trans]=I_p$ with an observable diameter of order $\mathcal{O}(1)$ and $\sigma$ be the uniform measure on the sphere $\mathcal{S}^{p-1}\subset\mathbb{R}^{p}$ of
radius $1$, then for any integer $k$, small compared to $p$, there exist two constants $C,c$ and a set $\Theta\subset (\mathcal{S}^{p-1})^k$ such that $\underbrace{\sigma\otimes\ldots\otimes\sigma}_k(\Theta)\geq 1-\sqrt{p}Ce^{-c\sqrt{p}}$ and $\forall\theta=(\theta_1,\ldots,\theta_k)\in\Theta$,
\begin{equation*}
    \forall a\in\mathbb{R}^k : \sup\limits_{t\in\mathbb{R}}|\mathbb{P}(a^\trans\theta^\trans {\bf x}\geq t)-G(t)|\leq Cp^{-\frac 14}.
\end{equation*}
with $G(t)$ the cumulative distribution function of $\mathcal{N}(0,1)$
\end{theorem}
Then the result unfolds naturally.

\paragraph{Statistical mean of the classification scores.}
Using the definition of the score in \eqref{eq:score_class}, the average output score $g_i({\bf x})$ for ${\bf x}\in\mathcal C_j$ is
\begin{equation*}
\mathbb E[g_i({\bf x})]=\mathbb{E}\left[\frac{1}{kp}\left({e_{i}^{[k]}}\otimes\mu_{ij}\right)^\trans A^{\frac 12}\tilde{Q}A^{\frac 12}Z(Y-Pb)\right]+b_i.
\end{equation*}
Using Lemma~\ref{lem:eq}, this can be further developed as:
\begin{equation}
    \label{eq:mean_score}
    \mathbb E[g_i({\bf x})]=\frac{1}{kp}\left({e_{i}^{[k]}}\otimes\mu_{ij}\right)^\trans A^{\frac 12}\Bar{\Tilde{Q}}A^{\frac 12}M_{\bm\delta}J^{\trans}(Y-P\bar{b})+b_i + o(1).
\end{equation}
Since $\mathbb{C}_{ij}=A^{\frac 12}(e_i^{[k]}{e_i^{[k]}}^\trans\otimes(\Sigma_{ij}+\mu_{ij}\mu_{ij}^\trans) )A^{\frac 12}$ is a finite rank update of $\Sigma_{ij}$, one can further use Woodbury identity  matrix (i.e., $\left(A+UCV\right)^{-1}=A^{-1}+A^{-1}UC(I+VA^{-1}U)VA^{-1}$ for invertible square $A$) to write ${\bar{\tilde{Q}}=\bar{\tilde{Q}}_0-\bar{\tilde{Q}}_0\mathbb{M}(I_{kp}+\mathbb{M}^{\trans}\bar{\tilde{Q}}_0\mathbb{M})^{-1}\mathbb{M}^\trans\bar{\tilde{Q}}_0}$, with
\begin{align*}
    \bar{\tilde{Q}}_0&=\left[\sum\limits_{i=1}^k\sum\limits_{j=1}^m\left(\mathcal{D}_{\gamma}+\lambda\mathbb{1}_{k}\mathbb{1}_{k}\right)^{\frac 12}e_{i}e_{i}^\trans\left(\mathcal{D}_{\gamma}+\lambda\mathbb{1}_{k}\mathbb{1}_{k}\right)^{\frac 12}\otimes {\bm\delta}_{ij}^{[mk]}\Sigma_{ij}+I_{kp}\right]^{-1} \\
    \mathbb{M}&=A^{\frac 12}M\mathcal{D}_{{\bm\delta}^{[mk]}}^{\frac 12}
\end{align*}
with ${\bm\delta}^{[mk]}=[{\bm\delta}_{i1}^{[mk]},\ldots,{\bm\delta}_{mk}^{[mk]}]$ for ${\bm\delta}_{ij}^{[mk]}=\frac{c_{ij}}{c_0(1+\bm\delta_{ij}^{[mk]})}$.
Plugging the expression of $\bar{\tilde{Q}}$ in Equation~\eqref{eq:mean_score}, we obtain
\begin{align*}
    \mathbb E[g_i({\bf x})]&=e_{ij}^{\trans}\mathcal{D}_{{\bm\delta}^{[mk]}}^{-\frac 12}\mathbb{M}^{\trans}\bar{\tilde{Q}}\mathbb{M}\mathcal{D}_{{\bm\delta}^{[mk]}}^{\frac 12}\mathring{\mathcal{Y}}+b_i+o(1)\\
    &=e_{ij}^{\trans}\mathcal{D}_{{\bm\delta}^{[mk]}}^{-\frac 12}\left(I_{mk}-\Gamma\right)\mathcal{D}_{{\bm\delta}^{[mk]}}^{\frac 12}\mathring{\mathcal{Y}}+b_i+o(1)
\end{align*}
with $\Gamma=(I_{mk}+\mathbb{M}^\trans\bar{\tilde{Q}}_0\mathbb{M})^{-1}$ and $e_{ij}^{[mk]}$ is the canonical vector.
Finally, to be exhaustive without going into the technical details,\footnote{Due to Remark~\ref{rem:on_centering}, $b_i$ can take any arbitrary value since only the decision threshold but not the performance is sensitive to a shift of $Y$.} let us conclude by remarking that one can show using the deterministic equivalent for $Q$ provided in \citep{louart2018concentration} that $b_i=\frac{\mathbb{1}_{n_i}^\trans Y_i}{n_i}+\mathcal{O}(p^{-\frac12})=\mathcal{Y}-\mathring{\mathcal{Y}}+\mathcal{O}(p^{-\frac12})$.

Finally, letting $\mathcal{m}_{ij}$ be the above expression of $\mathbb E[g_i({\bf x})]$ without the trailing $o(1)$ and $\mathcal{m}=[\mathcal{m}_{11},\ldots,\mathcal{m}_{km}]^\trans$, one concludes using the notations of Theorem~\ref{th:main} that
\begin{equation*}
    \mathcal{m}=\mathcal{Y}-\mathcal{D}_{{\bm\delta}^{[mk]}}^{-\frac 12}\Gamma\mathcal{D}_{{\bm\delta}^{[mk]}}^{\frac 12}\mathring{\mathcal{Y}}
\end{equation*}
as desired.
\paragraph{Variance of the classification score.}


Using Equation~\eqref{eq:score_class}, for ${\bf x}\in\mathcal C_j$, the covariance of the score $g_i({\bf x})$ is given by
\begin{equation*}
    {\rm Cov}[g_i({\bf x})]=\mathbb{E}\left[\frac{1}{(kp)^2}(Y-Pb)^\trans Z^\trans A^{\frac 12} \tilde{Q}A^{\frac 12}S_{ij}A^{\frac 12}\tilde{Q}A^{\frac 12}Z(Y-Pb)\right]
\end{equation*}
Using the deterministic equivalent of $Z^\trans A^{\frac 12} \tilde{Q}A^{\frac 12}S_{ij}A^{\frac 12}\tilde{Q}A^{\frac 12}Z$  in Lemma 1, the expression further reads 
\begin{align*}
     {\rm Cov}[g_i({\bf x})]=&=\frac{1}{(kp)^2}(Y-P\bar{b})^\trans \left(JM_{\bm\delta}^{\trans}A^{\frac 12}B_{ij}A^{\frac 12}M_{\bm\delta}J+F_{ij}\right)(Y-P\bar{b})\\
     &-\frac{1}{p^2}(Y-P\bar{b})^{\trans}JM_{\bm\delta}^{\trans}A^{\frac 12}\Bar{\Tilde{Q}}A^{\frac 12}M_{\bm\delta}W_{ij}(Y-P\bar{b}).
\end{align*}
Similarly to the calculus performed for $\mathbb E[g_i({\bf x})]$, using again ${\bar{\tilde{Q}}=\bar{\tilde{Q}}_0-\bar{\tilde{Q}}_0\mathbb{M}(I_{kp}+\mathbb{M}^{\trans}\bar{\tilde{Q}}_0\mathbb{M})^{-1}\mathbb{M}^\trans\bar{\tilde{Q}}_0}$, similar algebraic manipulations lead to:
\begin{align*}
    {\rm Cov}[g_i({\bf x})]=&=\mathring{\mathcal{Y}}^\trans\mathcal{D}_{{\bm\delta}^{[mk]}}^{\frac 12}\mathbb{M}^{\trans}B_{ij}\mathbb{M}\mathcal{D}_{{\bm\delta}^{[mk]}}^{\frac 12}\mathring{\mathcal{Y}}+\mathring{\mathcal{Y}}^{\trans}\mathcal{D}_{{\bm\delta}^{[mk]}}^{\frac 12}\mathcal{D}_{\kappa_{ij,.}}\mathcal{D}_{{\bm\delta}^{[mk]}}^{\frac 12}\mathring{\mathcal{Y}}-\mathring{\mathcal{Y}}^{\trans}\mathcal{D}_{{\bm\delta}
   ^{[mk]}}^{\frac 12}\mathbb{M}^{\trans}\bar{\tilde{Q}}\mathbb{M}\mathcal{D}_{\frac{\kappa_{ij,.}}{{\bm\delta}^{[mk]}}}\mathcal{D}_{{\bm\delta}^{[mk]}}^{\frac 12}\mathring{\mathcal{Y}}\\
    &=\mathcal{Y}^\trans\mathcal{D}_{{\bm\delta}^{[mk]}}^{\frac 12}\Gamma\mathbb{M}^{\trans}\bar{\tilde{Q}}_0\mathbb{V}_{ij}\bar{\tilde{Q}}_0\mathbb{M}\Gamma\mathcal{D}_{{\bm\delta}^{[mk]}}^{\frac 12}\mathcal{Y}+\mathcal{Y}^\trans\mathcal{D}_{{\bm\delta}^{[mk]}}^{\frac 12}(I-\Gamma)\mathcal{D}_{\kappa_{ij,.}}(I-\Gamma)+\\
    &\mathring{\mathcal{Y}}^{\trans}\mathcal{D}_{{\bm\delta}^{[mk]}}^{\frac 12}\mathcal{D}_{\kappa_{ij,.}}\mathcal{D}_{{\bm\delta}
   ^{[mk]}}^{\frac 12}\mathring{\mathcal{Y}}-
    2\mathring{\mathcal{Y}}^{\trans}\mathcal{D}_{{\bm\delta}^{[mk]}}^{\frac 12}(I-\Gamma)\mathcal{D}_{\kappa_{ij,.}}\mathcal{D}_{{\bm\delta}^{[mk]}}^{\frac 12}\mathring{\mathcal{Y}}\\
    &=\mathcal{Y}^\trans\mathcal{D}_{{\bm\delta}^{[mk]}}^{\frac 12}\left(\Gamma\mathcal{D}_{\kappa_{ij,.}}\Gamma+\Gamma\mathbb{M}^{\trans}\bar{\tilde{Q}}_0\mathbb{V}_{ij}\bar{\tilde{Q}}_0\mathbb{M}\Gamma\right)\mathcal{D}_{{\bm\delta}^{[mk]}}^{\frac 12}\mathcal{Y}
\end{align*}
with ${\mathbb{V}_{ij}=A^{\frac 12}S_{ij}A^{\frac 12}+\sum\limits_{i'=1}^k\sum\limits_{j'=1}^m {\bm\delta}_{i'j'}^{[mk]}\kappa_{ij,i'j'}A^{\frac 12}S_{i'j'}A^{\frac 12}}$ and $\kappa_{ij,.}=[\kappa_{ij,11},\ldots,\kappa_{ij,k2}]$ with $\kappa_{ij,i'j'}={d_{i'j'}T_{ij,i'j'}}/{{\bm\delta}_{i'j'}^{[mk]}}$.

\subsubsection{Particular Case}
In the case of binary classification ($m=2$) and for $\Sigma_{ij}=I_p$, we have the simplification:
\begin{align*}
    \mathbb{M}&=\sum\limits_{i,j}\left(\mathcal{D}_{\gamma}+\lambda\mathbb{1}_{k}\mathbb{1}_{k}^{\trans}\right)^{\frac 12}e_i^{[k]}{e_i^{[k]}}^{\trans}\otimes \sqrt{\tilde{\bm\delta}_{i}}\mathring{\mu}_{ij}\\
    &=\sum\limits_{i}\left(\mathcal{D}_{\gamma}+\lambda\mathbb{1}_{k}\mathbb{1}_{k}^{\trans}\right)^{\frac 12}e_i^{[k]}{e_i^{[k]}}^{\trans}\otimes \left(\frac{\left[\frac{c_{i2}\sqrt{c_{i1}}}{c_i},-\frac{c_{i1}\sqrt{c_{i2}}}{c_i}\right]}{c_0(1+\bm\delta_i)}\otimes\Delta\mu_{i}\right).
\end{align*}
Moreover, $\bar{\tilde{Q}}_0=[(\mathcal{D}_{\gamma}+\lambda\mathbb{1}_k\mathbb{1}_k^{\trans})^{\frac 12}\mathcal{D}_{\tilde{\bm\delta}}\left(\mathcal{D}_{\gamma}+\lambda\mathbb{1}_k\mathbb{1}_k^{\trans}\right)^{\frac 12}+I_k]^{-1}\otimes I_p$, so that
\begin{align*}
    \mathbb{M}^\trans\bar{\tilde{Q}}_0\mathbb{M}&=\sum\limits_{i,j}e_i^{[k]}{e_i^{[k]}}^{\trans}\left[I_{k}+\mathcal{D}_{\tilde{\bm\delta}}^{-\frac 12}\left(\mathcal{D}_{\gamma}+\lambda\mathbb{1}_k\mathbb{1}_k^{\trans}\right)^{-1}\mathcal{D}_{\tilde{\bm\delta}}^{-\frac 12}\right]^{-1}e_j^{[k]}{e_j^{[k]}}^{\trans} \Delta\mu_i^\trans\Delta\mu_j\otimes \mathbb{c}_i\mathbb{c}_j^\trans\\
    &=\sum\limits_{i,j}\mathcal{A}_{ij} \Delta\mu_i^\trans\Delta\mu_j e_i^{[k]}{e_j^{[k]}}^{\trans} \otimes \mathbb{c}_i\mathbb{c}_j^\trans\\
    &=\left(\mathcal{A}\otimes\mathbb{1}_k\mathbb{1}_k^\trans\right)\odot \mathcal M
\end{align*}
with 
\begin{align*}
    \mathcal M &\equiv \sum\limits_{i,j}\Delta\mu_i^\trans\Delta\mu_j \left(E_{ij}^{[k]} \otimes \mathbb{c}_i\mathbb{c}_j^\trans\right) \\ \mathbb{c}_i &\equiv\begin{bmatrix}\frac{c_{i2}}{c_i}\sqrt{\frac{c_{i1}}{c_i}}\\-\frac{c_{i1}}{c_i}\sqrt{\frac{c_{i2}}{c_i}}\end{bmatrix} \\
    \mathcal{A} &\equiv\left[I_k+\mathcal{D}_{\tilde{\bm\delta}}^{-\frac 12}\left(\mathcal{D}_{\gamma}+\lambda\mathbb{1}_k\mathbb{1}_k^{\trans}\right)^{-1}\mathcal{D}_{\tilde{\bm\delta}}^{-\frac 12}\right]^{-1}.
\end{align*}
As for the covariance terms,
\begin{align*}
    \mathbb{M}^\trans\bar{\tilde{Q}}_0V_{ij}\bar{\tilde{Q}}_0\mathbb{M}&=\sum\limits_{i,j}e_i^{[k]}{e_i^{[k]}}^{\trans}\left[I_k+\mathcal{D}_{\tilde{\bm\delta}}^{-\frac 12}\left(\mathcal{D}_{\gamma}+\lambda\mathbb{1}_k\mathbb{1}_k^{\trans}\right)^{-1}\mathcal{D}_{\tilde{\bm\delta}}^{-\frac 12}\right]^{-1}\mathcal{D}_{\tilde{\bm\delta}}^{-\frac 12}\left(e_i^{[k]}{e_i^{[k]}}^{\trans}+\mathcal{D}_{\kappa_i\odot\tilde{\bm\delta}}\right)\times\\
    &\mathcal{D}_{\tilde{\bm\delta}}^{-\frac 12}\left[I_k+\mathcal{D}_{\tilde{\bm\delta}}^{-\frac 12}\left(I_k+\lambda\mathbb{1}_k\mathbb{1}_k^{\trans}\right)^{-1}\mathcal{D}_{\tilde{\bm\delta}}^{-\frac 12}\right]^{-1}e_j^{[k]}{e_j^{[k]}}^{\trans} \Delta\mu_i^\trans\Delta\mu_j\otimes \mathbb{c}_i\mathbb{c}_j^\trans\\
    &=\sum\limits_{i,j}\left[\mathcal{A}\mathcal{D}_{\tilde{\bm\delta}}^{-\frac 12}\left(e_i^{[k]}{e_i^{[k]}}^{\trans}+\mathcal{D}_{\kappa_i\odot\tilde{\bm\delta}}\right)\mathcal{D}_{\tilde{\bm\delta}}^{-\frac 12}\mathcal{A}\right]_{ij} \Delta\mu_i^\trans\Delta\mu_j e_i^{[k]}{e_j^{[k]}}^{\trans} \otimes \mathbb{c}_i\mathbb{c}_j^\trans\\
    &=\left(\mathcal{A}\mathcal{D}_{\tilde{\bm\delta}}^{-\frac 12}\left(e_i^{[k]}{e_i^{[k]}}^{\trans}+\mathcal{D}_{\kappa_i\odot\tilde{\bm\delta}}\right)\mathcal{D}_{\tilde{\bm\delta}}^{-\frac 12}\mathcal{A}\otimes\mathbb{1}_k\mathbb{1}_k^\trans\right)\odot \mathcal M\\
    &=\frac{1}{\bm{\delta}_i^{[k]}}\left(\mathcal{A}\mathcal{D}_{\mathcal{K}_i+e_{i}^{[k]}}\mathcal{A}\otimes\mathbb{1}_k\mathbb{1}_k^\trans\right)\odot \mathcal M
\end{align*}
with $\mathcal{K}_{ia}=\tilde{\bm\delta}_i\kappa_{ia}$.
Using Equation~\eqref{eq:equation_kappa}, after algebraic manipulations, we finally obtain the compact form
\begin{equation}
    \mathcal K =  \frac{c_0}{k}[\mathcal A\odot \mathcal A]\left(\mathcal{D}_c -\frac{c_0}{k}[\mathcal A\odot \mathcal A]\right)^{-1}.
\end{equation}
\subsection{Proof of Propositions~\ref{prop:classification_accuracy_one_vs_all}--\ref{prop:classification_accuracy_one_hot}}
\label{sec:proposition_proof}
\subsubsection{One-versus-all}
The probability of correct classification for Task~$i$ and for a test data ${\bf x}\in\mathcal{C}_{j}$ reads
\begin{align*}
    \mathbb{P}\left(g_i^{\rm bin}({\bf x};j)>\max_{j'\neq j} \{g_i^{\rm bin}({\bf x};j')\}\right)=\mathbb{P}\left(g_i^{\rm bin}({\bf x};j)-\max_{j'\neq j} \{g_i^{\rm bin}({\bf x};j')\}>0\right).
\end{align*}

Since by definition (Equation~\eqref{eq:score_class})
\begin{align}
    g_i^{\rm bin}({\bf x};j) &= \frac{1}{kp} \mathring{\mathcal{y}}(j)^\trans J^\trans QZ^\trans A \left(e_{i}^{[k]}\otimes \mathring{\bf x}\right)+b_i,
\end{align}
we have that $g_i^{\rm bin}({\bf x},j) {\bf 1}_{m-1}-\left\{g_i^{\rm bin}({\bf x};j')\right\}_{j'\neq j}= \frac{1}{kp} \mathcal{Y}_{-j} J^\trans QZ^\trans A \left(e_{k}^{[k]}\otimes \mathring{\bf x}\right)$, where $\mathcal{Y}_{-j}=(\mathring{\mathcal{y}}(j)^\trans-\left[\mathring{\mathcal{y}}(j')^\trans\right]_{j'\neq j})\in\mathbb{R}^{(m-1)\times km}$.
Using Theorem~\ref{th:main} with $\mathcal Y$ replaced by $\mathcal{Y}_{-j}$, $g_i^{\rm bin}({\bf x},j) {\bf 1}_{m-1}-g_i^{\rm bin}({\bf x};j')_{j'\neq j}\in\mathbb{R}^{m-1}$ is asymptotically a multivariate Gaussian random vector with statistics detailed in the theorem statement. Proposition~\ref{prop:classification_accuracy_one_vs_all} then unfolds trivially by remarking that $g_i^{\rm bin}({\bf x};j)>\max_{j'\neq j} g_i^{\rm bin}({\bf x};j')\Leftrightarrow \forall j'\neq j,~g_i^{\rm bin}({\bf x};j)-g_i^{\rm bin}({\bf x};j')\geq 0$.

\subsubsection{One Hot encoding}
The proof is similar to the one-versus-all case.

The probability of correct classification for a test data ${\bf x}\in\mathcal{C}_{j}$ is
\begin{align*}
    \mathbb{P}\left(g_i^{\rm bin}({\bf x};j)>\max_{j'\neq j} \{g_i^{\rm bin}({\bf x};j')\}\right)=\mathbb{P}\left(g_i^{\rm bin}({\bf x};j)-\max_{j'\neq j} \{g_i^{\rm bin}({\bf x};j')\}>0\right)
\end{align*}
where
\begin{align}
    g_i^{\rm bin}({\bf x};j) &= \frac{1}{kp} {e_{j}^{[k]}}^\trans\mathring{\mathcal Y}^\trans J^\trans QZ^\trans A \left(e_{i}^{[k]}\otimes \mathring{\bf x}\right)+b_i.
\end{align}
Therefore $g_i^{\rm bin}({\bf x};j){\bf 1}_{m-1}-\left\{(g_i({\bf x};j'))\right\}_{j'\neq j}= \frac{1}{kp} \mathcal{E}_j\mathring{\mathcal Y}^\trans J^\trans QZ^\trans A (e_{k}^{[k]}\otimes \mathring{\bf x})$,
with
$\mathcal{E}_j=\{(e_j^{(m)}-e_{j'}^{(m)})^\trans\}_{j\neq j'}\in\mathbb{R}^{(m-1)\times m}$.
By Theorem~\ref{th:main} with $\mathring{\mathcal Y}$ replaced by $\mathcal{E}_j^\trans\mathring{\mathcal Y}^\trans$, this vector is asymptotically normally distributed and Proposition~\ref{prop:classification_accuracy_one_hot} unfolds immediately using again the fact that $g_i^{\rm bin}({\bf x};j)>\max_{j'\neq j} g_i^{\rm bin}({\bf x};j')\Leftrightarrow\forall j'\neq j,~ g_i^{\rm bin}({\bf x};j)-g_i^{\rm bin}({\bf x};j'j)\geq 0$.
\vskip 0.2in
\bibliography{sample.bib}

\end{document}